\newif\ifarXiv         
\newif\ifPNAS        
\PNASfalse            
\arXivtrue              
\ifPNAS
\documentclass[9pt,twocolumn,twoside,lineno]{pnas-new}
\templatetype{pnasresearcharticle} 
\fi
\ifarXiv
\documentclass[11pt]{article}
\usepackage{authblk}
\providecommand{\keywords}[1]
{
\small	
\textbf{\textit{Keywords: }} #1
}
\usepackage
[
        a4paper,
        left=2cm,
        right=2cm,
        top=4cm,
        bottom=4cm,
]
{geometry}
\fi
\usepackage{bm}
\usepackage{amsmath}
\usepackage{amsfonts, verbatim}
\usepackage{amsthm}            
\usepackage{amssymb}          
\usepackage{bbm}                 
\usepackage{booktabs}          
\usepackage{graphicx}
\usepackage{subcaption}
\usepackage{wrapfig}
\usepackage{float}
\usepackage{algorithm}
\usepackage[noend]{algpseudocode}
\usepackage[normalem]{ulem}
\DeclareMathAlphabet{\mathpzc}{OT1}{pzc}{m}{it}
\usepackage{xcolor}
\usepackage[normalem]{ulem} 
\newcommand{\stkout}[1]{\ifmmode\text{\sout{\ensuremath{#1}}}\else\sout{#1}\fi}
\usepackage{mathtools}
\usepackage{color}
\usepackage{graphicx}
\usepackage{setspace}

\usepackage{tabularx,ragged2e,booktabs,caption}
\newcolumntype{C}[1]{>{\Centering}m{#1}}

\ifPNAS
\setboolean{displaywatermark}{false}
\fi

\ifarXiv
\usepackage[colorlinks,citecolor=red,urlcolor=blue,bookmarks=false,hypertexnames=true]{hyperref} 
\fi

\ifPNAS
\newcommand{\mycaption}[1]{\caption{{#1}}}
\fi
\ifarXiv
\newcommand{\mycaption}[1]{\caption{\small{#1}}}
\fi

\newtheorem{theorem}{Theorem}[section]
\newtheorem{lemma}[theorem]{Lemma}
\newtheorem{proposition}[theorem]{Proposition}

\newtheorem{remark}[theorem]{Remark}
\newtheorem{definition}{Definition}[section]

\newcommand{\rhoT}{\rho_T}
\newcommand{\rhoL}{\rho_T^L}
\newcommand{\rhoLM}{\rho_T^{L,M}}

\newcommand{\mbf}[1]{\boldsymbol{#1}}

\newcommand{\dbinnerp}[1]{\langle\hspace{-0.5mm}\langle{#1}\rangle\hspace{-0.5mm}\rangle}
\newcommand{\abs}[1]{\big| #1 \big|}

\newcommand{\bX}{\mbf{X}}
\newcommand{\bxm}{\bx^{m}}
\newcommand{\dotbxm}{\dot\bx^m}
\newcommand{\bXm}{\bX^{m}}
\newcommand{\dotbXm}{\dot\bX^{m}}
\newcommand{\Xtrain}{\bX_{\mathrm{tr}}}
\newcommand{\brm}{\br^{m}}
\newcommand{\nbrm}{r^{m}}

\newcommand{\norm}[1]{\left\| #1 \right\|}

\newcommand{\tnorm}[1]{\norm{#1}_{\mathbb{S}}}

\newcommand{\real}{\mathbb{R}}

\newcommand{\br}{\mbf{r}}

\newcommand{\bv}{\mbf{v}}

\newcommand{\bx}{\mbf{x}}

\newcommand{\mE}{\mathcal{E}}
\newcommand{\mF}{\mathcal{F}}

\newcommand{\bigO}{\mathcal{O}}
\newcommand{\R}{\real}
\newcommand{\dimamb}{D}
\renewcommand{\dim}{d}
\newcommand{\numcl}{K}
\newcommand{\idxcl}{k}
\newcommand{\spaceM}{S}
\newcommand{\cl}{C}
\newcommand{\clof}{\mathpzc{k}}

\newcommand{\Tmixing}{T_{\mathrm{mix}}}

\newcommand{\force}{F}
\newcommand{\forcev}{\force^{\bv}}
\newcommand{\forcexi}{\force^{\xi}}
\newcommand{\intkernel}{\phi}
\newcommand{\lintkernel}{\hat\intkernel}
\newcommand{\bintkernel}{{\bm{\phi}}}
\newcommand{\blintkernel}{{\widehat{\bm{\phi}}}}
\newcommand{\intkernelvar}{\varphi}
\newcommand{\bintkernelvar}{{\bm{\varphi}}}
\newcommand{\intkernele}{\intkernel^E}
\newcommand{\intkernela}{\intkernel^A}
\newcommand{\lintkernele}{\lintkernel^E}
\newcommand{\lintkernela}{\lintkernel^A}
\newcommand{\intkernelxi}{\intkernel^{\xi}}

\newcommand{\rhsfo}{\mathbf{f}}
\newcommand{\hypspace}{\mathcal{H}}
\newcommand{\E}{\mathbb{E}}

\newcommand{\probIC}{\mu_0}
\newcommand{\intkerneltrue}{\phi}

\newcommand{\ptrans}[1]{(#1)^{\top}}

\newcommand{\argmin}[1]{\underset{#1}{\operatorname{arg}\operatorname{min}}\;}

\newcommand{\supp}[1]{\text{supp}(#1)}

\newcommand{\revision}[1]{\textcolor{black}{{#1}}}
\newcommand{\newrefs}[1]{{#1}}
\newcommand{\theTitle}{Nonparametric inference of interaction laws in systems of agents from trajectory data}
\newcommand{\theAbstract}{Inferring the laws of interaction in agent-based systems from observational data is a fundamental challenge in a wide variety of disciplines.
We propose a non-parametric statistical learning approach for distance-based interactions, with no reference or assumption on their analytical form, given data consisting of sampled trajectories of interacting agents. We demonstrate the effectiveness of our estimators both by providing theoretical guarantees that avoid the curse of dimensionality, and by testing them on a variety of prototypical systems used in various disciplines.  These systems include homogeneous and heterogeneous agents systems, ranging from particle systems in fundamental physics to agent-based systems that model opinion dynamics under the social influence, prey-predator dynamics, flocking and swarming, and phototaxis in cell dynamics.}
\newcommand{\theKeywords}{Data-driven modeling $|$ Dynamical systems $|$ Agent-based systems}
\newcommand{\theAcknowledge}{We are grateful for comments by the reviewers, which lead to significant improvements to the paper, and for discussions with Prof. Massimo Fornasier, Prof. Pierre-Emmanuel Jabin, Prof. Yannis Kevrekidis, Prof. Nathan Kutz, Prof. Yaozhong Hu and Dr. Cheng Zhang. We acknowledge support from the National Science Foundation under grants DMS-1708602, ATD-1737984, IIS-1546392, DMS-1821211, IIS-1837991, the Air Force Office of Scientific Research (AFOSR) under grant AFOSR-FA9550-17-1-0280,  and American Mathematical Society-Simons Travel grant. We acknowledge Duke University for computing equipment and the Maryland Advanced Research Computing Center (MARCC). }

\begin{document}
\ifarXiv
\author[a, b]{Fei Lu}
\author[a, b]{Mauro Maggioni}
\author[a]{Sui Tang}
\author[b]{Ming Zhong}
\affil[a]{Department of Mathematics}
\affil[b]{Department of Applied Mathematics \& Statistics}
\affil[{ }]{Johns Hopkins University, Baltimore, MD $21218$, USA}
\title{\theTitle}
\date{\today}
\maketitle
\begin{abstract}
\theAbstract
\end{abstract}
\fi
\section{Introduction}
\ifarXiv
Systems
\fi
 of interacting agents arise in a wide variety of disciplines, including Physics, Biology, Ecology, Neurobiology, Social Sciences, and Economics (see e.g. \cite{carrillo2017review, KSUB2011, vicsek2012collective,Shoham} and references therein).
Agents may represent particles, atoms, cells, animals, neurons, people, rational agents, opinions, etc...
The understanding of agent interactions at the appropriate scale in these systems is as fundamental a problem as the understanding of interaction laws of particles in Physics.

How can laws of interaction between agents be discovered? In Physics vast knowledge and intuition exist to formulate hypotheses about the form of interactions, inspiring careful experiments and accurate measurements, that together lead to the inference of interaction laws. 
This is a classical area of research, dating back to at least Gauss, Lagrange, and Laplace \cite{StiglerHistoryStats}, that plays a fundamental role in many disciplines. 
In the context of interacting agents at the scale of complex organisms, there are fewer controlled experiments possible, and few ``canonical'' choices for modeling the interactions. 
Different types and models of interactions have been proposed in different scientific fields, and fit to experimental data, which in turn may suggest new modeling approaches, in a model-data validation loop.
Often the form of governing interaction laws is chosen a priori, within perhaps a small parametric family, and the aim is often to reproduce only qualitatively, and not quantitatively, some of the macroscopic features of the observed dynamics, such as the formation of certain patterns. 

Our work fits at the boundary between statistical/machine learning and dynamical systems, where equations are estimated from observed trajectory data, and inference takes into account assumptions about the form of the equations governing the dynamics. 
Since the past decade, the rapidly increasing acquisition of data, due to decreasing costs of sensors and measurements, has made the learning of large and complex systems possible, and there has been an increasing interest in inference techniques that are model-agnostic and scalable to high-dimensional systems and large data sets. 

We establish statistically sound, dynamically accurate, computationally efficient techniques\footnote{The software package implementing the proposed algorithms can be found on \url{https://github.com/MingZhongCodes/LearningDynamics}.} for inferring these interaction laws from trajectory data.
We propose a {\em{non-parametric}} approach for learning interaction laws in particle and agent systems, based on observations of trajectories of the states (e.g. position, opinion, etc...) of the systems, on the assumption that the interaction kernel depends on pairwise distances only, unlike recent efforts either require feature libraries or parametric forms for such interactions \cite{Schaeffer6634,BPK2016,TranWardExactRecovery,BCGMSVW2012}, or aim at identifying only the type of interaction from a small set of possible types \cite{BCCCCGLOPPVZ2008, LLEK2010, KTIHC2011}.
We consider a Least Squares (LS) estimator, classical in the area of inverse problems (dating back to Legendre and Gauss), suitably regularized and tuned to the learning of the interaction kernel in agent-based systems. 

The unknown is the interaction kernel, a function of pairwise-distances between agents of the systems. 
While the values of this function are not observed, in contrast to the standard regression problems, yet we are able to show that our estimator converges at an optimal rate as if we were in the 1-dimensional regression setting. 
In particular, \revision{the learning rate has no dependency on the dimension} of the state space of the system, therefore avoiding any curse of dimensionality, and making these estimators well-suited for the modern high-dimensional data regime. 
Our estimator is constructed with algorithms that are computationally efficient and may be implemented in a streaming fashion: it is, therefore, well-suited for large data sets.
It may be easily extended to a variety of complex systems; we consider here first order and second order models, with single and multiple types of agents, and with interactions with simple environments. 
We also show that the theoretical guarantees on the performance of the estimator make it suitable for hypothesis testing when the true model is unknown, assisting the investigator in choosing among different possible (nonparametric) models.
%
\section{Learning interaction kernels}\label{sec:main2}
We start with a model that is used in a wide variety of interacting agent systems (e.g. physical particles, influence propagation in a population \cite{Krause2000, CKFL2005}): consider $N>1$ agents $\{\bx_i\}_{i=1}^N$ in $\R^\dim$, evolving according to the system of ODE's
\begin{equation}
\label{e:firstordersystemsimple}
  \dot{\bx}_i(t) = \frac1N\sum_{i' = 1}^{N}\intkernel(||\bx_{i'}(t) - \bx_i(t)||)(\bx_{i'}(t) - \bx_i(t))\,,
\end{equation}
where $\dot{\bx}_i(t)=\frac{d}{dt}\bx_i(t)$; $\norm{\cdot}{}$ is the Euclidean norm, and $\intkernel:\R_+\rightarrow\R$ is the {\em{interaction kernel}}. \revision{In other words, every agent's velocity is obtained by superimposing the interactions with all the other agents, each weighted in way dependent on the distance to the interacting agent. In a prototypical example, e.g. arising in particle systems (see \ifarXiv Sec.~\ref{s:LJexample}\fi) and flocking systems, the interaction kernel may be negative for small distances, inducing repulsion, and attractive for large distances.}
Let $\bX:=(\bx_i)_{i=1}^N\in\R^{\dim N}$ be the state vector for all the agents, $\br_{ii'}(t):=\bx_{i'}(t) - \bx_i(t)$ and $r_{ii'}(t):=||\br_{ii'}(t)||$.
The evolution \eqref{e:firstordersystemsimple} is the gradient flow for the potential energy
$\mathcal{U}(\bX(t)):=\frac{1}{2N}\sum_{i\neq i'}\Phi(r_{ii'}(t))$,
with $\intkernel(\cdot)=\Phi'(\cdot)/\cdot$. 
The function $\intkernel(\cdot)\cdot$ reappears naturally below, the fundamental reason being its relationship with $\mathcal{U}$ and $\Phi$. Our observations are positions along trajectories: $\Xtrain:=\{\bXm(t_l)\}_{l=1,m=1}^{L,M}$, \revision{with $0=t_1<\dots<t_L=T$ being the times at which observations occur, and $m$ indexing $M$ different trajectories}. Velocities $\dotbXm(t_l)$ are approximated by finite differences.
The $M$ initial conditions (I.C.'s) $\bX_0^{m}:=\bXm(0)$ are drawn independently at random from a probability measure $\probIC$ on $\R^{dN}$.

Our goal is to infer, in a nonparametric fashion, the interaction kernel $\intkernel$, by constructing an estimator $\lintkernel$ from training data.
A fundamental statistical problem that involves estimating a function is regression: given samples $(z_i,g(z_i))_{i=1}^n$, with the $z_i$'s i.i.d. samples from an (unknown) measure $\rho_Z$ in $\R^\dimamb$, and $g$ a suitably regular (say H\"older $s$) unknown function $\R^\dimamb\rightarrow\R$, one constructs an estimator $\hat g$ such that $||\hat g-g||_{L^2(\rho_Z)}\lesssim n^{-\frac{s}{2s+D}}$, with high probability (over the $z_i$'s). 
This rate is optimal in a minimax sense, \cite{Gyorfi06}, and its dramatic degradation with $\dimamb$ is a manifestation of the curse of dimensionality. 
Upon re-writing \eqref{e:firstordersystemsimple} as $\dot{\bX}=\rhsfo_\intkernel(\bX)$,
our observations (with either approximated or directly observed velocities) resemble those needed for regression if we thought of $Z=\bX$ as a random variable, and $g=\rhsfo_\intkernel$. 
However, our observations are not i.i.d. samples of $\bX$ with respect to any probability measure, the lack of independence being the most glaring aspect. 
If we nevertheless pursued this line of thought, we would be hit with the curse of dimensionality in trying to learn the target function $g=\rhsfo_\intkernel$ on the state space $\R^{dN}$, leading to a rate $n^{-O(1/dN)}$ for regression. This renders this approach useless in practice as soon as, say, $dN\ge20$.
A direct application of existing approaches (e.g. \cite{Schaeffer6634,BPK2016,TranWardExactRecovery}), developed for low-dimensional systems, go in this direction, These works would try to ameliorate this curse of dimensionality by requiring $\rhsfo_\intkernel$ to be well-approximated by a linear combination of a small number of functions in a known large dictionary. While such dictionaries may be known for specific problems, they are usually not given in the case of complex, agent-based systems. Finally, such dictionaries typically grow dramatically in size with the dimension (here, $dN$), and existing guarantees that avoid the curse dimensionality require further, strong assumptions on the measurements or the dynamics.

We proceed in a different direction, aiming for the flexibility of a non-parametric model while exploiting the structure of the system in \eqref{e:firstordersystemsimple}. 
The target function $\intkernel$ depends on just one variable (pairwise distance), but it is observed through a collection of non-independent linear measurements (the l.h.s. of \eqref{e:firstordersystemsimple}), at locations $\nbrm_{ii'}(t_l)=||\bxm_{i'}(t_l)-\bxm_i(t_l)||$, with coefficients $\brm_{ii'}(t_l)=\bxm_{i'}(t_l)-\bxm_i(t_l)$, as in the r.h.s. of \eqref{e:firstordersystemsimple}.
When the $t_l$'s are equidistant in time, we consider an estimator minimizing the empirical error functional
\begin{align}
\mE_{L,M}(\intkernelvar) &:= \frac{1}{LMN}\sum_{l,m,i= 1}^{L,M,N}\big\|\dotbxm_i(t_l)-\rhsfo_\intkernelvar( \bxm(t_l))_i\big\|^2, \label{e:firstordersystem_eef} \\
\widehat\intkernel&=\widehat\intkernel_{L,M,\hypspace} := \argmin{\intkernelvar\in\hypspace} \mE_{L,M}(\intkernelvar)\, , \label{e:estimator}
\end{align}
where $\hypspace$ is a hypothesis space of functions $\R_+\rightarrow\R$, of dimension $n$ (we will choose $n$ dependent on $M$). 
We introduce a natural probability measure $\rhoT$ on $\R_+$  adapted to the dynamics: it can be thought of as an ``occupancy'' measure, in the sense that for any interval $I$, $\rhoT(I)$ is the probability (over the random initial conditions distributed according to $\probIC$) of seeing a pair of agents with a distance between them being a value in $I$, averaged over the time interval $[0,T]$; see \eqref{e:rhoT} for a formal definition.

We measure the performance of $\hat\phi$ in terms of the error $||\hat\intkernel(\cdot)\cdot-\intkernel(\cdot)\cdot||_{L^2(\rhoT)}$. Thm.~\ref{t:mainsimple}, our main result, will bound this error by $\smash{\tilde O(M^{-s/(2s+1)})}$ if $\intkernel$ is H\"older $s$: this is the optimal exponent for learning $\intkernel$ if we were in the (more favorable) $1$-dimensional regression setting! We therefore completely avoid the curse of dimensionality. \revision{In fact, we show under some rather general assumptions that not only the rate, but even the constants in the bound are independent of $N$, making the bounds essentially dimension-free.}
It is crucial that $\rhoT$ has wide support in order for the error to be informative. When the system is ergodic, we expect $\rhoT$ to have a large support for large $T$, as the system explores its ergodic distribution.  However many deterministic systems of interest may reach a stationary state (as in the cases of the Lennard-Jones or opinion dynamics, to be considered momentarily), in which case $\rhoT$ becomes highly concentrated on a finite set for large $T$: in these cases it may be more relevant to consider $T$ small compared to the relaxation time.

We are also interested in whether trajectories $\bX(t)$ of the true system are well-approximated by trajectories $\smash{\widehat{\bX}(t)}$ of the system governed by the interaction kernel $\smash{\hat\intkernel}$, on both the ``training'' time interval $[0,T]$ and after time $T$. Prop.~\ref{stateestimation} below bounds $\smash{\sup_{t\in[0,T']} \|\widehat{\bX}(t)- \bX(t)\|}$ in terms of $\smash{||\hat\intkernel(\cdot)\cdot-\intkernel(\cdot)\cdot||_{L^2(\rhoT)}}$, at least for $T'$ not too large; this further validates the use of $L^2(\rhoT)$.
We will report on this distance for both $T'=T$ and $T'>T$ (``prediction'' regime).

\begin{figure*}[!h]
\begin{subfigure}[t]{\ifarXiv 0.49\fi\textwidth}
\centering
      \includegraphics[width=\linewidth]{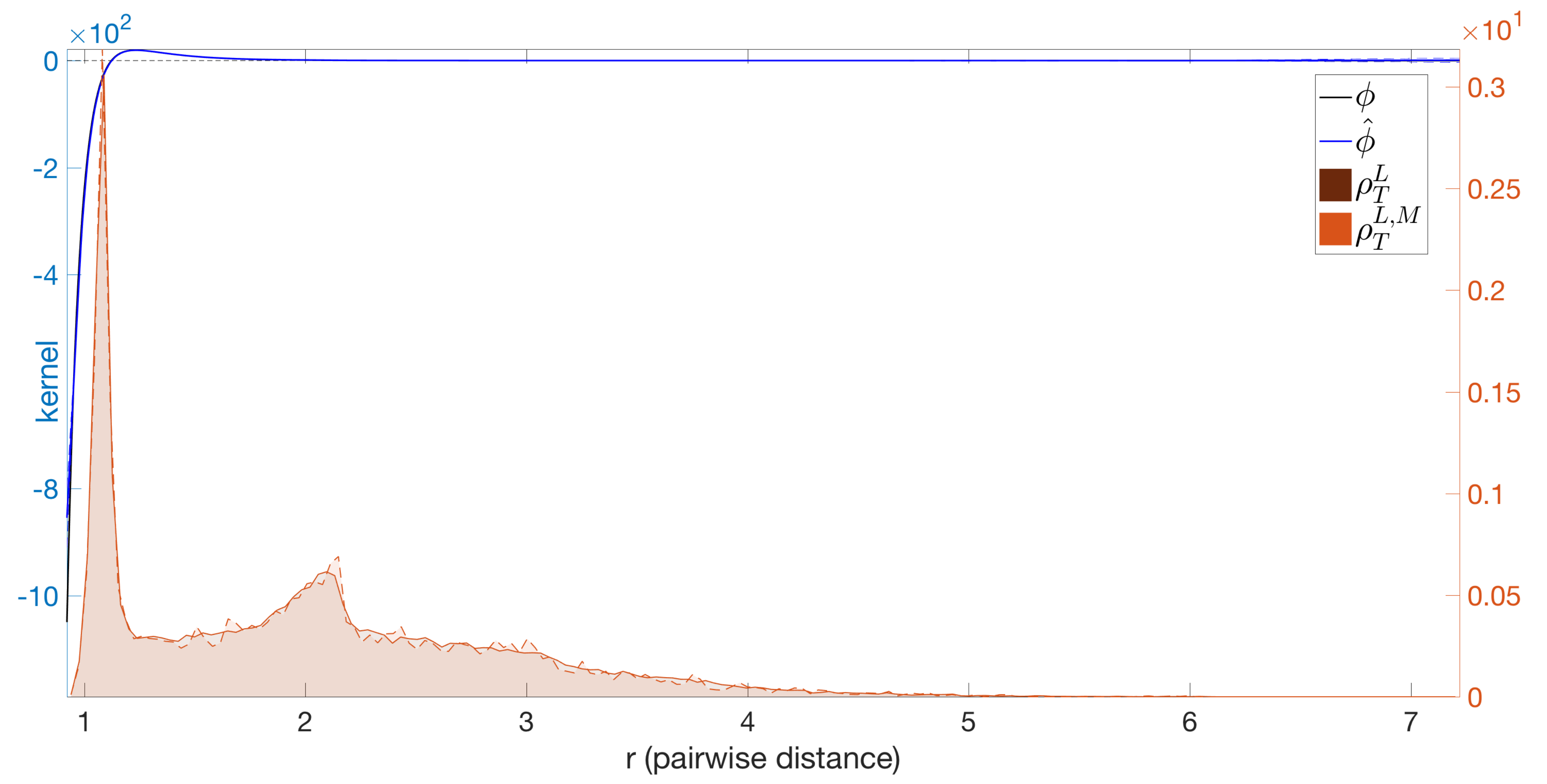}
      \subcaption{Interaction kernel learned from many shot-time trajectories} \end{subfigure}
\begin{subfigure}[t]{\ifarXiv 0.49\fi\textwidth}
\centering
      \includegraphics[width= \linewidth]{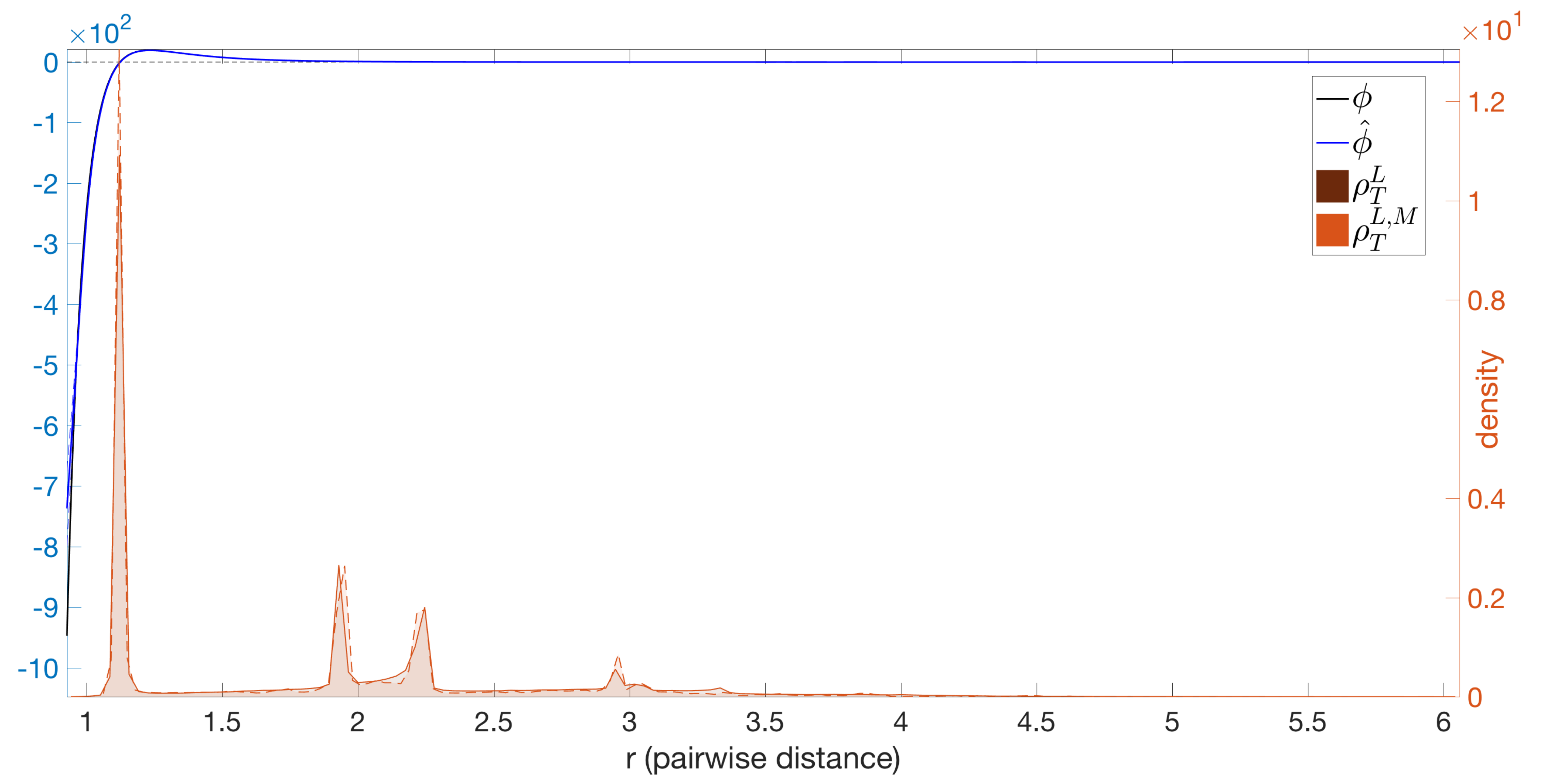}
      \subcaption{Interaction kernel learned from a few  long trajectories}
\end{subfigure}
\begin{subfigure}[t]{\ifarXiv 0.49\fi\textwidth}
\centering
     \includegraphics[width=\linewidth]{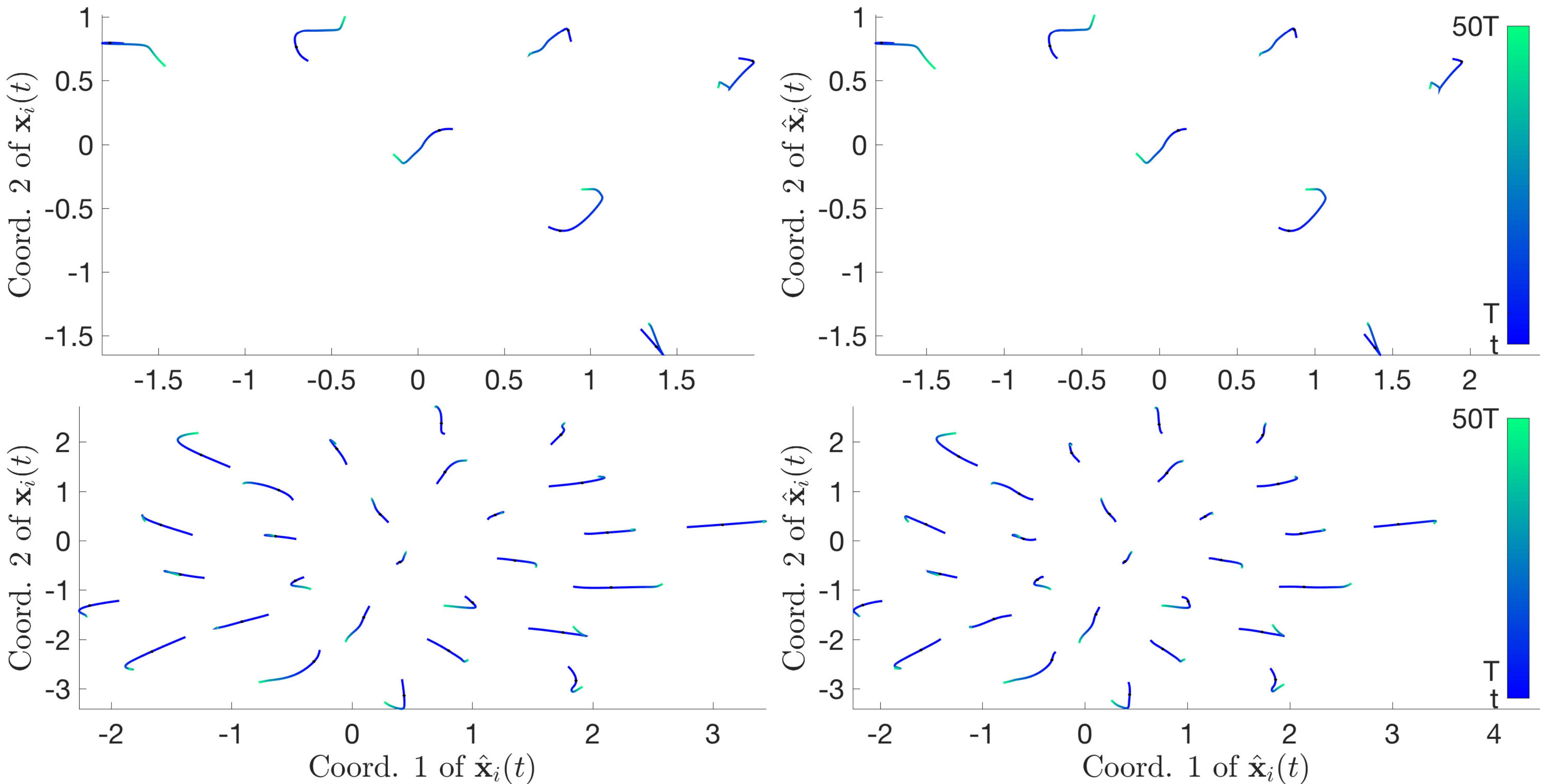}
      \subcaption{True and predicted trajectories for systems with interaction kernel learned in (a)} 
\end{subfigure}
\begin{subfigure}[t]{\ifarXiv 0.49\fi\textwidth}
\centering
      \includegraphics[width=\linewidth]{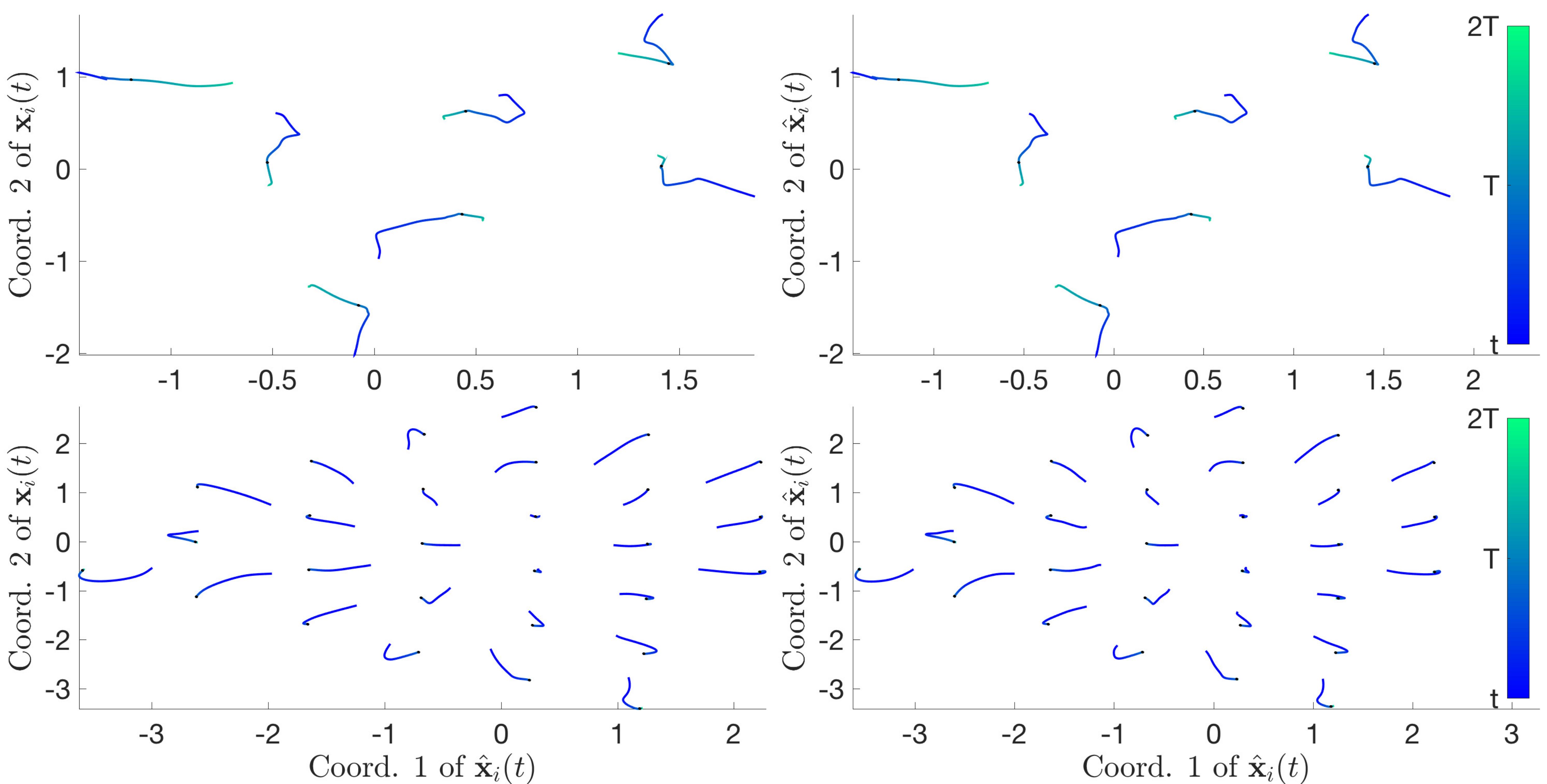}
      \subcaption{True and predicted trajectories for system with interaction  kernel learned in (b) } 
 \end{subfigure}
\mycaption{\textbf{Interaction kernel estimation and trajectory prediction for Lennard-Jones system}.  {\textit{Top row}}: estimators $\hat\intkernel$ (in blue) of the true interaction kernel $\intkernel$ (in black) in two sampling regimes: many short-time trajectories (left), and a few large-time trajectories (right). The proposed nonparametric estimators perform extremely well - the means and standard deviations of the relative $L^2(\rhoL)$ errors are $6.6\cdot 10^{-2}\pm 5.0\cdot 10^{-3}$ and $7.2\cdot 10^{-2} \pm 1.0\cdot 10^{-2}$ respectively, over 10 independent learning runs. The standard deviation (dashed) lines on the estimated kernel are so small to be barely visible. 
In both cases we superimpose histograms of $\rhoL$ (estimated from a large number of trajectories, outside of training data) and $\smash{\rho_T^{L,M}}$ (estimated from the $M$ training data trajectories, see \ifarXiv Eq.\eqref{e:rhoLM}\fi). 
The estimators belong to a hypothesis space $\hypspace_n$ of piecewise linear functions with equidistant knots, and yield accurate estimators in $L^2(\rhoL)$.
Note that we observe the dynamics starting from a suitable $t_0>0$, due to the singularity of Lennard-Jones kernel at $r=0$. See \ifarXiv Sec.~\ref{LJdescriptions} \fi for details about the setup and results. {\textit{Bottom row}}: the true and predicted trajectories for the N-particle system (top row) and a 4N-particle system (bottom row) with interaction kernels learned on the N-particle system, for randomly sampled initial conditions. The blue-to-green color gradient indicates the movement of particles in time (see color scales on the side). We achieve small errors in predicting the trajectories in all cases, even when we transfer the interaction kernel learned on an $N$ particle system to predict trajectories of a system with $4N$ particles. 
}
\label{f:LJ_main}
\end{figure*}

Finally, while the error $\smash{||\hat\intkernel(\cdot)\cdot-\intkernel(\cdot)\cdot ||_{L^2(\rhoT)}}$ is unknown in practice (since $\intkernel$ is unknown), our results give guarantees on its size, which in turn imply guarantees on accuracy of trajectory predictions. Proxies for the error on trajectories, for example by holding out portions of trajectories during the training phase, may be derived from data. 
These measures of error may be used to test and validate different models of the dynamics: too large an error with one model may invalidate it and suggest that a different one (e.g. $2^{\text{nd}}$ vs. $1^{\text{st}}$ order, or 
multiple vs. single agent types) should be used (see Sec.~\ref{s:examples}).

\subsection{Different sampling regimes, and randomness}
The total number of observations is (\# of initial conditions)$\times$(\# of temporal observations in $[0,T]$)$=M\times L$, each in $\mathbb{R}^{dN}$.
We will consider several regimes:
\begin{itemize} \setlength\itemsep{0mm} 
\item{\em{Many Short Time Trajectories}}: $T$ is small, $L$ is small (e.g. $L=1$), and $M$ is large (many I.C.'s sampled from $\probIC$);
\item{\em{Single Large Time Trajectory}}: $T$ large (even comparable to the relaxation time of the system if applicable), $L$ is large, and $M=1$ (or very small);
\item {\em{Intermediate Time Scale}}: $T$, $L$ and $M$ are all not small, but none is very large, corresponding to multiple ``medium''-length trajectories, with several different initial conditions.
\end{itemize}

Randomness is injected via the initial conditions, and in our main results in Sec.~\ref{s:MainResults} the sample size will be $M$.
If the system is ergodic, the regimes above are partially related to each other, at least when the initial conditions are sampled from the ergodic distribution $\mu_{\mathrm{erg}}$. Indeed, at times much larger than the mixing time $\Tmixing$, the state of the system becomes indistinguishable from a random sample of $\mu_{\mathrm{erg}}$, and we may interpret the subsequent part of the trajectory as a new trajectory with that initial condition. The $M$ observed trajectories of length $T\gg\Tmixing$ are then equivalent to $M\times T/\Tmixing$ trajectories of length $\Tmixing$, to which our results apply.
In regimes when $M$ is very small or $\probIC$ is very concentrated, there is little randomness: the problem is close to a fixed-design inverse problem which is solvable if the dynamics produces different-enough pairwise distances.

\begin{figure}[!ht] \center
\includegraphics[width=0.5 \textwidth]{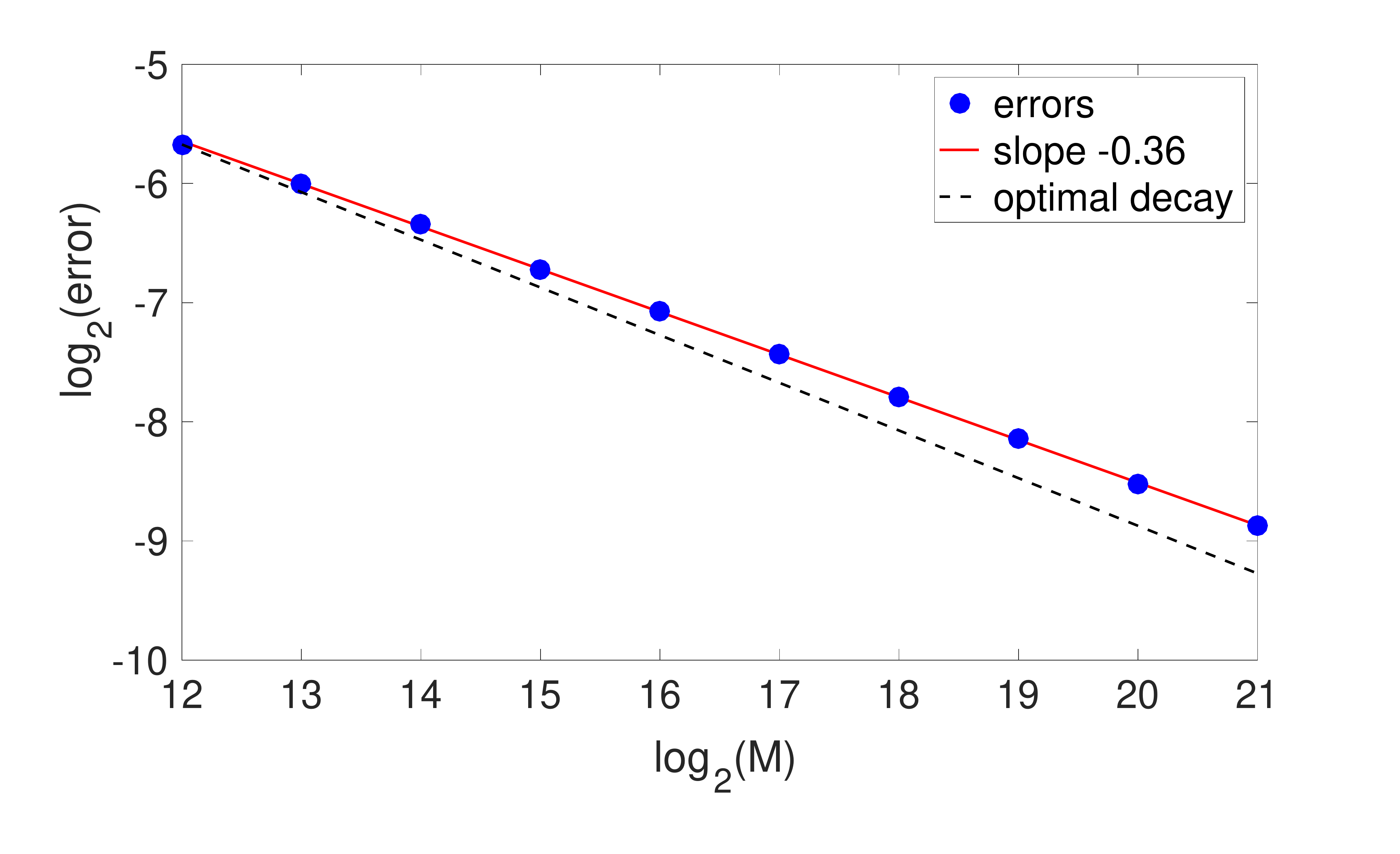}
\mycaption{\textbf{Learning rate in $M$ for the LJ system.}: The estimation error in ${L^2(\rhoL)}$ decays at rate $0.36$, close to the optimal rate $0.4$ (black dotted line) as in Thm.~\ref{t:mainsimple}.}
\label{f:LJ_rate}
\vskip-0.4cm
\end{figure}

\begin{figure}[h] 
\centering
\includegraphics[width=0.5\textwidth]{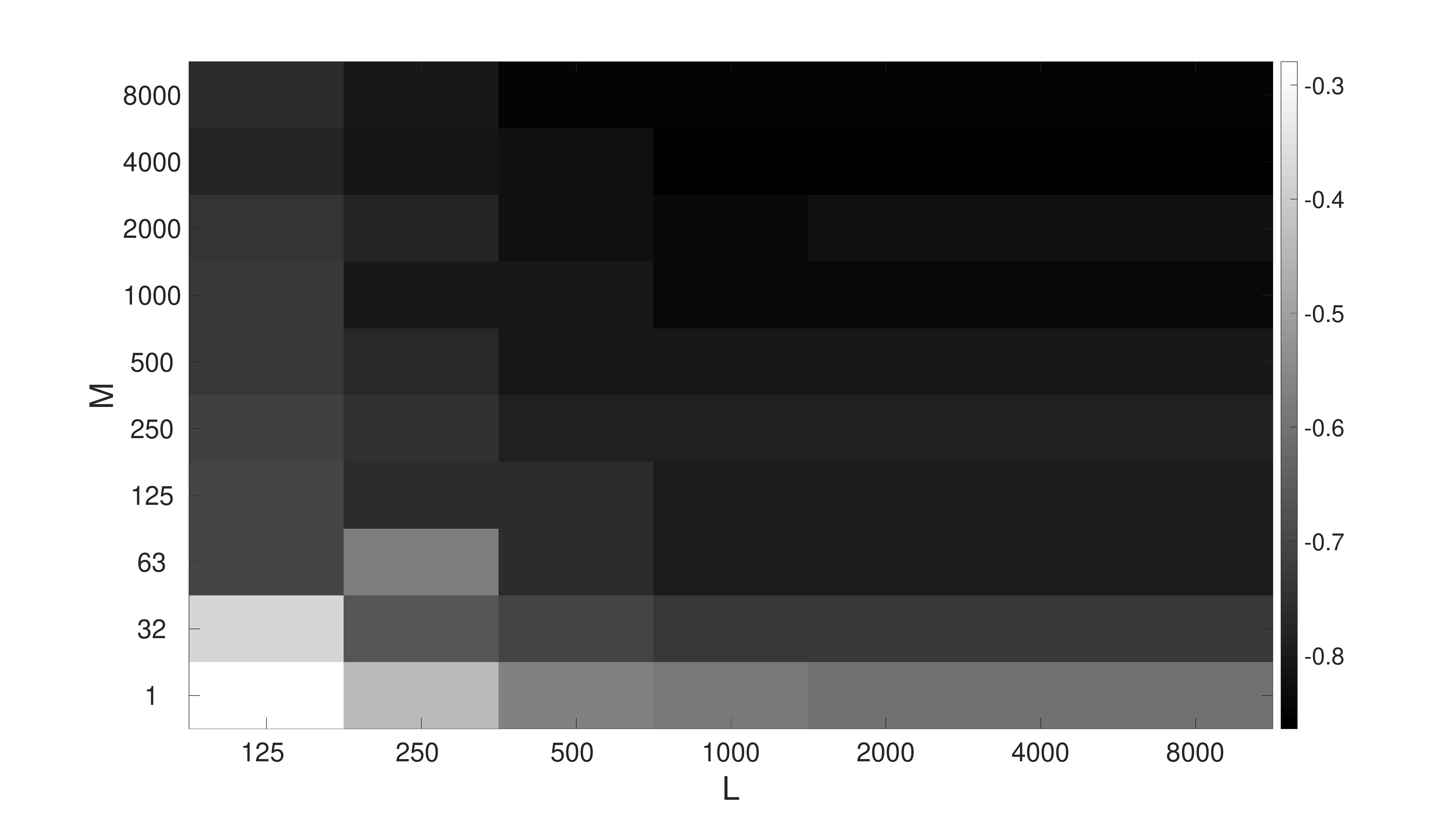}
\mycaption{ \textbf{The relative error of the estimated kernel as a function of $M, L$ for the LJ system. } The Relative error, in $\log_{10}$ scale, of $\hat\intkernel$ decreases both in $L$ and $M$, in fact roughly in the product $ML$, at least when $M$ and $L$ are not too small. $M=1$ does not seem to suffice, no matter how large $L$ is, due to the limited amount of ``information'' contained in a single trajectory.}
\label{fig:JLMLtest}
\end{figure} 

\subsection{Example: interacting particles with the Lennard-Jones potential}\label{s:LJexample}
We illustrate the learning procedure on a particle system with $N=7$ particles in $\R^2$, interacting according to \eqref{e:firstordersystemsimple} with $\phi(r)=\Phi_{LJ}'(r)/r$, where $\Phi_{LJ}(r):=4\epsilon\left( (\sigma/r)^{12}-(\sigma/r)^6\right)$ is the Lennard-Jones potential, consisting of a strong near-field repulsion and a long-range attraction. 
The system converges quickly to equilibrium configurations, 
which often consist of ordered, crystal-like structures. 
 This example is challenging for various reasons:
the interaction kernel is unbounded, has unbounded support, and equilibrium is reached quickly, reducing the amount of information in trajectories. \ifarXiv Sec.~\ref{LJdescriptions} \fi contains a detailed description of  the experiments. Fig.~\ref{f:LJ_main} demonstrates that the estimators approximate the true kernel well in different sampling regimes, and that the trajectories of the true system are well-approximated by those of the learned system both in the ``training'' interval ($[t_0,T]$) and in the ``prediction'' interval ($[T, 50T]$ and $[T, 2T]$ respectively for the two regimes). \revision{We also show, as a simple example of transfer learning, that we can use the interaction kernel learned on the system with $N$ particles to accurately predict trajectories of a system with $4N$  particles.}

The rate of decay of the estimation error is about 0.36 (see Fig.~\ref{f:LJ_rate}), close to the optimal rate 0.4 in Thm.~\ref{t:mainsimple}; 
this is a consequence of two factors: the use of an empirical approximation to $\rhoL$ and the blowup at $0$ of $\Phi_{LJ}$, which is not an admissible kernel as in Thm.~\ref{t:mainsimple} (see \ifarXiv Sec.~\ref{LJdescriptions} \fi for a detailed discussion).  

Fig.~\ref{fig:JLMLtest} shows the behavior of the error of the estimators as both $L$ and $M$ are increased. \revision{It indicates that a single long trajectory may not contain enough ``information'' to learn the kernel, at least for deterministic systems approaching a steady state. It also shows the behavior predicted by Thm.~\ref{t:mainsimple}, namely for each fixed $L$ the error decreases as $M$ increases.}

\section{Learning Theory}\label{s:MainResults}
We introduce an error functional based on the structure of the dynamical system $\dot{\bX}=\rhsfo_\intkernel(\bX)$, whose minimizer will be our estimator of the interaction kernel $\intkernel$.
We consider kernels in the \emph{admissible set}
$\mathcal{K}_{R,\spaceM}: = \{\intkernel \in C^1(\R_+):  \ \supp \intkernel \subset  [0,R], \sup_{r\in[0,R]} |\intkernel(r)| + | \intkernel'(r) | \leq \spaceM \}$,
for some $R,\spaceM >0$. 
The boundedness of $\intkernel$ and $\intkernel'$ ensures the global well-posedness of the system in \eqref{e:firstordersystemsimple}. 
The restriction $\supp \intkernel \subset  [0,R]$ models the finite range of interaction between agents, and it may be relaxed to $\intkernel\in W^{1,\infty}(\R_+)$ with a suitable decay. 

\subsection{Probability measures adapted to the dynamics}
In order to measure the quality of the estimator of the interaction kernel $\intkernel$, we introduce two probability measures on $\R_+$, the space of pairwise distances $\nbrm_{ii'}(t_l)=||\bxm_{i'}(t_l) - \bxm_i(t_l)||$. 
We consider the expectation of the empirical measure of pairwise distances, for continuous and discrete time observations respectively:
\begin{align}
\rhoT(r)  &:= \frac{1}{\binom N2T}\int_{t = 0}^T\E_{\bX_0\sim\probIC}\bigg[\sum_{i,i'=1, i< i' }^N\delta_{r_{ii'}(t)}(r) \, dt\bigg]\,, \label{e:rhoT}\\
\rhoL(r)  &:= \frac{1}{\binom N2 L}\sum_{l= 1}^{L}\E_{\bX_0\sim\probIC}\bigg[\sum_{i,i'=1, i< i' }^N\delta_{r_{ii'}(t_l)}(r)\bigg] \,. \label{e:rhoL} 
\end{align}
The expectations are over the initial conditions, with distribution $\probIC$.
The measure $\rhoT$ is intrinsic to the dynamical system, dependent on $\probIC$ and the time scale $T$, and independent of the observation data. $\rhoL$ depends also on the sampling scheme $\{t_l\}_{l=1}^L$ in time. 
Both are Borel probability measures on $\R_+$ (\ifarXiv Lemma \ref{averagemeasure}\fi) measuring how much regions of $\R_+$ on average (over the observed times and ICs) are explored by the system. Highly explored regions are where the learning process ought to be more accurate, as they are populated by more ``samples'' of pairwise distances. 
We will measure the estimation error of our estimators in  $L^2(\rhoT)$ or $L^2(\rhoL)$.

We report here on the analysis in the discrete-time observation case, most relevant in practice, with $\rhoL$; the arguments however also apply to continuous-time observations, with $\rhoT$.

\subsection{Learnability: the coercivity condition} 

A fundamental question is the learnability of the kernel, i.e., the convergence of the estimator $\lintkernel_{L,M,\hypspace}$ defined in \eqref{e:estimator} to the true kernel $\intkerneltrue$ as the sample size increases (i.e. $M\to \infty$) and $\hypspace$ increases in a suitable way. The following condition, similar to the one introduced in \cite{BFHM17} for studying the mean field limit ($N\rightarrow\infty$), ensures learnability and well-posedness of the estimation.
 
\begin{definition}[Coercivity condition] \label{def_coercivity}
The dynamical system in \eqref{e:firstordersystemsimple}, with initial condition sampled from $\probIC$ on $\R^{dN}$, satisfies the {\bf{coercivity condition}} \revision{on a set $\hypspace$} if there exists a constant $c_L>0$ such that for all $\intkernelvar\in \hypspace$ with $\intkernelvar(\cdot)\cdot  \in  L^2(\rhoL)$,
\begin{align}\label{coercivity}
  c_L\|\intkernelvar(\cdot)\cdot\|_{L^2(\rhoL)}^2\!\!  \leq \!\frac{1}{NL}\sum_{l,i=1}^{L,N}\E \big\| \frac{1}{N}\sum_{i'= 1}^{N}  \intkernelvar(r_{ii'}(t_l))\br_{ii'}(t_l) \big\|^2
 \end{align} 
\end{definition}

The coercivity condition ensures learnability, by implying the uniqueness of minimizer of 
$\mE_{L,\infty}(\intkernelvar): = \E[\mE_{L,M}(\intkernelvar)]$
and, eventually, the convergence of estimators through a control of the error of the estimator in $L^2(\rhoL)$ (see \ifarXiv see Thm.~\ref{thm_Learnability} and Prop.~\ref{convexity} \fi).  Thm.~\ref{t:coercivity} proves that the coercivity condition holds under suitable hypotheses, \revision{even independently of $N$}; 
 numerical tests suggest that it holds generically over larger classes of interaction kernels and distributions of initial conditions, for large $L$, and as long as $\rhoL$ is not degenerate, see \ifarXiv Fig.~\ref{fig:LJ_coercivity_1}.\fi 
Finally, $c_L$ also controls the condition number of the matrix in the Least Squares problem yielding the estimator (see \ifarXiv Sec. \ref{sec_condN}\fi).

The next theorem proves the coercivity condition when $\probIC$ is exchangeable (i.e. the distribution is invariant under permutation of components), Gaussian, and $L=1$. Numerical tests show that the coercivity condition holds true for a larger class of interaction kernels, for various initial distributions including Gaussian and uniform distributions, and for large $L$ as long as $\rhoL$ is not degenerate. We conjecture that the coercivity condition holds true in much greater generality (but not always!), leaving a detailed investigation to future work.  
\begin{theorem}
\label{t:coercivity}
Suppose $L=1, N>1$ and assume that the distribution of $\bX(t_1)=(\bx_1(t_1),\cdots, \bx_N(t_1))$ is exchangeable Gaussian with $\mathrm{cov}(\bX_i)-\mathrm{cov}(\bX_i,\bx_{i'})=\lambda I_d$ for a constant $\lambda>0$. 
Then the coercivity condition holds true with constant $c_L=\frac{N-1}{N^2}$ on $L^{2}(\rhoL)$\revision{, and with a constant $c_{\mathcal{H}}>0$, independent of $N$, for any compact hypothesis space $\hypspace \subset L^{2}(\rhoL)$}.
 \end{theorem}

\revision{The constant $c_{\hypspace}$ is independent of $N$ fundamentally because of the exchangeability of the distribution;} the following lemma is key in the proof of the theorem: 
 \begin{lemma}\label{coerlemma}
Let $X, Y, Z$ be exchangeable Gaussian random vectors in $\R^d$ with $\mathrm{cov}(X)-\mathrm{cov}(X,Y)=\lambda I_d$ for a constant $\lambda>0$, and let $g:\R_+\to \R$ be a function such that $g(\cdot)\cdot\in L^2(\R_+, \rho_1))$  with the probability distribution $\rho_1(r)\propto r^{d-1}e^{-r^2/3}$. 
Then
 \begin{equation*} \label{ineq_gauss}
\E\left[g(|X-Y|)g(|X-Z|)  \langle X-Y, X-Z \rangle \right] \geq 0.
 \end{equation*} 
 \revision{Moreover, for any compact hypothesis space $\hypspace \subset L^{2}(\rhoL)$, 
  \begin{equation*} \label{ineq_gauss2}
\E\left[g(|X-Y|)g(|X-Z|)  \langle X-Y, X-Z \rangle \right] \geq c_{\hypspace} \|g(\cdot)\cdot\|^2_{ L^2(\rho_1)},
 \end{equation*} 
 for some constant $c_{\hypspace}>0$.
 }
 \end{lemma}
\noindent\revision{The lemma is proved by writing the above expectation as an integral 
\[ \int\int g(r)g(s) \mathcal{K}(r,s)drds,\]
and by showing that the function $\mathcal{K}(r,s)$ is a positive definite integral kernel. }

\subsection{Optimal rates of convergence}

The classical bias-variance trade-off in statistical estimation guides the selection of the hypothesis space $\hypspace$, whose dimension will depend on $M$, the number of observed trajectories. On the one hand, $\hypspace$ should be large so that the bias (distance between the true kernel $\intkerneltrue$ and $\hypspace$) is small; on the other hand, $\hypspace$ should be small so that variance of the estimator is small. 
In the extreme case where $\hypspace=\mathcal{K}_{R, \spaceM}$, the bias is $0$, the variance of the estimator dominates, and we obtain the dimension-independent bound
$\mathbb{E}[\| \widehat\intkernel_{L,M,\hypspace}(\cdot)\cdot-\intkerneltrue(\cdot)\cdot\|_{L^2(\rhoL)} ]\leq C M^{-{1}/{4}}$  (see \ifarXiv Prop. \ref{prop_optRate}\fi).  In fact, significantly better rates may be achieved for regular $\intkerneltrue$'s:
\begin{theorem} 
\label{t:mainsimple}
Assume that $\intkerneltrue\in\mathcal{K}_{R,\spaceM}$.  
Let $\{\hypspace_n\}_n$ be a sequence of subspaces of $L^\infty([0,R])$, with  $\mathrm{dim}(\hypspace_n) \leq c_0 n$ and 
$\inf_{\intkernelvar \in \hypspace_n}\|\intkernelvar-\intkerneltrue\|_{L^\infty([0,R])}  \leq c_1 n^{-s}$, for some constants $c_0, c_1, s >0$. \revision{Assume that the coercivity condition holds on $\cup_{n=1}^\infty \hypspace_n$.}
Such a sequence exists, for example, if $\intkernel$ is $s$-H\"older regular.
Choose  $n_*=({M}/{\log M})^{{1}/{(2s+1)}}$. Then there exists a constant $C=C(c_0,c_1, R,S)$ such that
\begin{align}
 \mathbb{E}[\| \widehat\intkernel_{L,M,\hypspace_{n_*}}(\cdot)\cdot-\intkerneltrue(\cdot)\cdot\|_{L^2(\rhoL)} ] \leq \revision{\frac{C}{c_L}} \left(\frac{\log M}{M}\right)^{\frac{s}{2s+1}}\,. 
 \label{e:expectationMainBound}
 \end{align}
\end{theorem}

The rate we achieve is {\em{optimal}}: it coincides with the minimax rate in the classical regression setting where one can observe directly the values of an $s$-H\"older regression function at the sample points. 
Obtaining this optimal rate in our context is perhaps surprising, because we do not observe the values $\{\intkerneltrue(r_{ii'}^{m}(t_l))\}_{l,i,i',m}$, but a ``mixture'' of them in the observed trajectory data. 
\revision{Many choices of $\{\hypspace_n\}$ are consistent with the requirements in the theorem, e.g. splines on increasingly finer grids, or band-limited functions with increasing frequency limits. These choices affect the constants in \eqref{e:expectationMainBound}, the computational complexity of computing $\smash{\widehat\intkernel_{L,M,\hypspace_{n_*}}}$, but not the rate as a function of $M$.}
\revision{While the rate is independent of the dimension $dN$ of the state space, the constant may depend on $d$ and $N$ through the coercivity constant $c_L$. However, we do expect that under rather general conditions (e.g. as shown in Thm.~\ref{t:coercivity}),  $c_L$ is, in fact, independent of $N$ -- i.e. it is a fundamental property of the mean field limit ($N\rightarrow\infty$) of the system.
}

One shortcoming of our result is that \revision{the rate is not a function of $LN^2M$} (we have $LN^2/2$ pairwise distances for each of the $M$ trajectories), but only of $M$, the number of {\em{random}} samples. Numerical experiments  (see Fig.~\ref{fig:JLMLtest} and similar experiments for the other systems ) do suggest that the estimator does improve as $L$ increases, at least to a point, limited by the information contained in a single trajectory.
Comparing to \cite{BFHM17}, where the mean field limit $N\rightarrow\infty$, $M=1$, is studied, we see the rates in \cite{BFHM17} are no better than $N^{-1/d}$, i.e. they are cursed by dimension. 
So are sparsity-based inference techniques such as those in \cite{Schaeffer6634,STW:extractingsaprsedynamics,TranWardExactRecovery,BPK2016,BCCCCGLOPPVZ2008}, 
 which also require a good dictionary of template functions, are not non-parametric (at least in the form therein presented), and lack performance guarantees except in some cases under stringent assumptions.

 \revision{Our work here may be compared with the classical parameter estimation problem for the ODE models \cite{brunel2008parameter,liang2008parameter, cao2011robust,ramsay2007parameter}, where one is interested in estimating the vector parameter $\mbf{\theta}$ in the ODE model 
$\dot{\bX}=\mbf{f}(\bX(t),t,\mbf{\theta})$ from the observation of a single noisy trajectory. Our error functional, in spirit, is the same with the gradient matching method (also called the two-stage method) used in the parameter estimation problems (see \cite{bellman1971use,varah1982spline,ramsay1996principal,pascual2000linking,timmer2000parametric}).  A challenging problem is the identifiability of $\mbf{\theta}$. We refer the reader \cite{miao2011identifiability} for the statistical analysis and \cite{ramsay2005functional} (and references therein) for a comprehensive survey of this topic.
However, the problem and approach we considered here are different from the parameter estimation problem in several aspects. First of all our state variable $\bX$ enters into the domain of the $\intkernel$ (via its ``projection'' onto pairwise distance), while the parameter vector $\mbf{\theta}$ is decoupled from the state variable $\bX$. Moreover, our estimator is nonparametric, i.e, the goal is to estimate a function $\intkernel$ (a vector infinite dimensions) instead of a finite-dimensional vector $\mbf{\theta}$ of parameters.   Finally, we establish identifiability conditions for $\intkernel$ from the perspective that the observations are i.i.d trajectories with random initial conditions, in contrast with identifiability of $\mbf{\theta}$ from observations along a fixed single trajectory with i.i.d noise. }
\revision{We would like to mention the different but related problem of inferring potentials from ground states and unstable modes, see for example \cite{BJ2012}, as well as recent results on existence and properties of ground states for systems with non-local interactions \cite{SST2015}.}

\vskip-4mm
\subsection{Trajectory-based Performance Measures}
It is important not only that $\widehat\intkernel$ is close to $\intkerneltrue$, but also that the dynamics of the system governed by $\widehat\intkernel$ approximate well the original dynamics. 
The error in prediction may be bounded trajectory-wise by a continuous-time version of the error functional, and bounded in average by the $L^2(\rhoT)$ error of the estimated kernel (further evidence of the usefulness of $\rhoT$):
\begin{proposition}\label{stateestimation}
Assume $\widehat\intkernel(\|\cdot\|) \cdot \in\mathrm{Lip}(\R ^d)$, with Lipschitz constant $C_{\rm{Lip}}$. Let $\widehat{\bX}(t)$ and $\bX(t)$ be the solutions of systems with kernels $\widehat\intkernel$ and $\intkerneltrue$ respectively, started from the same initial condition. Then for each trajectory
\begin{align*}
\smash{
\sup_{t\in[0,T]}\!\! \|\widehat{\bX}(t)- \bX(t)\|^2
\leq 2Te^{8T^2C^2_{\rm{Lip}}} \!\!\!\int_0^T\!\!\!\left\|\dot\bX(t)-\rhsfo_{\hat{\intkernelvar}}(\bX(t))\right\|^2\!\!\!dt\,,}
\end{align*}
and on average w.r.t. the distribution $\probIC$ of initial conditions:
\[
\E_{\probIC}[\sup_{t\in[0,T]} \|\widehat{\bX}(t)- \bX(t)\|] \leq  C\sqrt{N} \|\lintkernel(\cdot)\cdot-\intkerneltrue(\cdot)\cdot\|_{L^2(\rhoT)}\,,
\] where the measure $\rho_T$ is defined in \eqref{e:rhoT} and  $C=C(T, C_{\rm{Lip}})$.
\end{proposition}

\section{Extensions: Heterogeneous agent systems, first and second order}\label{s:extensions}
The method proposed extends naturally to a large variety of interacting agent systems arising in a multitude of applications \cite{Shoham}, including systems with multiple types of agents, driven by second order equations, and including interactions with an environment. 
For detailed discussions of related topics on self-organized dynamics, we refer the readers to \cite{CD2011, CM2008, GC2004, KMAW2002, vicsek2012collective} and the recent surveys \cite{CPT2014,  CFTV2010}.

\subsection{First Order Heterogeneous Agents Systems}\label{s:SI_first_order}

Let the agents be divided into $\numcl$ disjoint sets $\{\cl_\idxcl\}_{\idxcl=1}^\numcl$ (``types''), with different interaction kernels for each ordered pair of types:
\begin{equation}
\label{e:firstordersystem}
  \dot{\bx}_i(t) = \sum_{i' = 1}^{N}\frac{1}{N_{\clof_{i'}}}\intkernel_{\clof_i\clof_{i'}}(r_{ii'}(t))\br_{ii'}(t)\,,
\end{equation}
where $\clof_i$ is the index of the type of agent $i$, i.e. $i\in C_{\clof_i}$; $N_{\clof_{i'}}$ is the number of agents in type $C_{\clof_{i'}}$; $\br_{ii'} = \bx_{i'} - \bx_i$ and $r_{ii'} = \norm{\br_{ii'}}$; $\intkernel_{\idxcl\idxcl'}:\R_+\rightarrow\R$ is the interaction kernel governing how agents in type $\cl_{k'}$ influence agents in type $\cl_k$.  As usual we let $\bX:=(\bx_i)_{i=1}^N\in\R^{\dim N}$ be the vector describing the state of the system.  We assume that the interaction kernels $\smash{\intkernel_{\clof_i\clof_{i'}}}$'s are the only unknown factors in the model; in particular we know the sets $\cl_k$'s (i.e. the type of each agent is known). The goal is to infer the interaction kernels $\intkernel_{\idxcl\idxcl'}$ from observations $\{\bX^{m}(t_l)\}_{l, m = 1}^{L, M}$ with $0=t_1<\dots<t_l=T$ and with the initial conditions $\bX^{m}(0)=\bX_0^{m}$ randomly sampled from $\probIC$.

Let $\rhsfo_{\bintkernel}(\bX^{m})\in\R^{dN}$ to be the vectorization of the right hand sides of \eqref{e:firstordersystem}, and $\bintkernel=(\intkernel_{\idxcl\idxcl'})_{\idxcl,\idxcl'=1}^{\numcl}$. Dropping from the notation of quantities that are assumed known, we  rewrite the equations for the dynamics in \eqref{e:firstordersystem} as
$
  \dot{\bX}^{m} = \rhsfo_{\bintkernel}(\bX^{m})
$.
We use an error functional similar to \eqref{e:firstordersystem_eef}, with a weighted norm, to define the estimators:
\begin{equation}\label{e:discretecomputedfirstordererror_edmc}
\blintkernel := \argmin{\bintkernelvar\in\hypspace}
\frac{1}{ML}\sum_{m = 1, l = 1}^{M, L}\norm{\dot{\bX}^{m}(t_l) - \rhsfo_{\bintkernelvar}(\bxm(t_l))}^2_{\mathcal{S}},
\end{equation}
where $\bintkernelvar =(\intkernelvar_{\idxcl\idxcl'})_{\idxcl,\idxcl'=1}^{\numcl}$, $\blintkernel = (\lintkernel_{\idxcl\idxcl'})_{\idxcl,\idxcl'=1}^{\numcl}$ and $\norm{\bX}^2_{\mathcal{S}} := \sum_{i = 1}^N\frac1{{N_{\clof_i}}}\norm{\bx_i}^2$.
The weighted norm $\norm{\cdot}^2_{\mathcal{S}}$ is introduced so that, when different types of agents have significantly different cardinalities (e.g. a large number of preys vs. a single predator), the error functional will take into suitable consideration the least numerous type. Otherwise only the interaction kernel of the most numerous type of agents would be accurately learned. Other more general weighting strategies may be considered, with minimal changes to the algorithm.

The generalization of $\rhoL$ in \eqref{e:rhoL} (similarly for $\rhoT$) to the heterogeneous-agent case is the family, indexed by ordered pairs $\{(k,k')\}_{k,k'\in \{1,\dots,K\}}$, of probability measures on $\R_+$
\begin{equation}\label{e:rhot_norm_mc} 
\rhoT^{L,\idxcl\idxcl'}(r)  = \displaystyle \frac{1}{LN_{\idxcl\idxcl'}}\sum_{l= 1}^{L}\E_{\mathbf{X}_0\sim\probIC}\!\!\!\!\!\!\!\!\!\sum_{\substack{i \in C_{\idxcl}, i' \in C_{\idxcl'}, i\neq i'}}\!\!\delta_{r_{ii'}(t_l)}(r),
\end{equation}
where $N_{\idxcl\idxcl'} = N_{\idxcl}N_{\idxcl'}$ when $\idxcl \neq \idxcl'$ and $N_{\idxcl\idxcl'} = {{N_{\idxcl}}\choose 2}$ when $\idxcl = \idxcl'$ (for $N_k>1$, otherwise there is no interaction kernel to learn).  
The error of an estimator, $\lintkernel_{\idxcl\idxcl'}$, will be measured by $\norm{\lintkernel_{\idxcl\idxcl'}(\cdot)\cdot - \intkernel_{\idxcl\idxcl'}(\cdot)\cdot}_{L^2(\rhoT^{L,\idxcl\idxcl'})}$.

While this case requires learning multiple interaction kernels, it turns out that the learning theory developed for the single-type agent systems can be generalized, and the estimator in \eqref{e:discretecomputedfirstordererror_edmc} still achieves optimal rates of convergence, and a similar control on the error of predicted trajectories can be obtained.

\subsection{Second order heterogeneous agent systems}
\label{s:secondordersystem}
Here we focus on a broad family of second order multi-type agent systems (not included, even when rewritten as first order systems, in the family discussed above).
We consider systems with $K$ types of agents:
\begin{align}
\!\!\!\left\{
\begin{aligned}
\!  m_i\ddot{\bx}_i &= \forcev_i(\dot\bx_i,\xi_i)  + \sum_{i' = 1}^N\frac{1}{N_{\clof_{i'}}}\big(\intkernele_{\clof_i\clof_{i'}}(r_{ii'})\br_{ii'} + \intkernela_{\clof_i\clof_{i'}}(r_{ii'})\dot\br_{ii'}\big)\\
\!\!  \dot{\xi}_i    &= \forcexi_i(\xi_i) + \sum_{i' = 1}^N\frac{1}{N_{\clof_{i'}}}\intkernelxi_{\clof_i\clof_{i'}}(r_{ii'})\xi_{ii'}\,, 
\end{aligned}
\raisetag{2.5\baselineskip}
\label{e:secondorder}
\right.
\end{align}
for $i=1,\ldots,N$.  
Here $\clof_{i}\in\{1,\dots,\numcl\}$ is the type of agent $i$, $\xi_i\in\R$ is a variable modeling the agent's response to the environment (e.g. food/light source), $\xi_{ii'} = \xi_{i'} - \xi_i$, and:
\begin{center} \vspace{-1mm}
\footnotesize{\begin{tabular}{ ll }
$m_i$, $N_k$  & mass of agent $i$ and number of agents of type $k$\\
$\forcev_i$, $\forcexi_i$  & non-collective influences on $\dot\bx_i$ and $\xi_i$ \\ 
$\smash{\intkernele_{\idxcl\idxcl'}}$, $\smash{\intkernela_{\idxcl\idxcl'}}$ & energy- and alignment- type interaction kernels \\ 
\end{tabular}}  
\end{center}
Note that here each agent is influenced by a weighted sum of different influences over agents of different types, leading to a rich family of models (including but not limited to prey-predator, leader-follower, cars-pedestrian models).
 Using vector notation, let $\rhsfo_{\bintkernel^E}(\bX^m)$ and $\rhsfo_{\bintkernel^A}(\bX^m, \dot{\bX}^{m}) \in \R^{dN}$ be the collection of the energy and alignment induced interaction terms respectively, and $\mF^{\bv}(\dot{\bX}^{m}, \Xi^{m})_i = \forcev_i(\dot\bx_i,\xi_i)$ (similar setup for $\mF^{\xi}(\Xi^{m})$ and $\rhsfo_{\bintkernel^{\xi}}(\bX^m, \Xi^{m})$) we can rewrite the equations as:
\begin{equation}
\left\{
\begin{aligned}
  \ddot{\bX}^{m} &= \mF^{\bv}(\dot{\bX}^{m}, \Xi^{m}) + \rhsfo_{\bintkernel^E}(\bX^m) + \rhsfo_{\bintkernel^A}(\bX^m, \dot{\bX}^{m})\\
  \dot{\Xi}^{m} &= \mF^{\xi}(\Xi^{m}) + \rhsfo_{\bintkernel^{\xi}}(\bX^m, \Xi^{m})\,,
\end{aligned}
\right.
\label{e:secondordercompact}
\end{equation}
where $\bintkernel^E =\{\intkernel_{\idxcl\idxcl'}^E\}$, $\bintkernel^A =\{\intkernel_{\idxcl\idxcl'}^A\}$ and $\bintkernel^\xi = \{\intkernel^\xi_{\idxcl\idxcl'}\}$, with $\idxcl,\idxcl'=1,\ldots,\numcl$.
We assume that the interaction kernels are the only unknowns in the model, to be estimated from the observations $\{\bXm(t_l), \dot{\bX}^{m}(t_l),  \Xi^{m}(t_l)\}_{l, m = 1}^{L, M}$, with $M$ initial conditions $\bX_0^{m}:=\bX^{m}(0)$, $\dot\bX^m_0:=\dot\bX^{m}(0)$, and $\Xi_0^{m}:=\Xi^{m}(0)$ sampled independently from $\probIC^{\bX}$, $\probIC^{\dot\bX}$, and $\probIC^{\Xi}$ respectively. 
With $\ddot{\bX}^{m}(t_l)$ approximated by finite difference, we construct estimators similar to those in \eqref{e:firstordersystem_eef} 
\begin{equation}\label{e:discretecomputedsecondordererror_v_edmc} 
\begin{aligned}
(\blintkernel^E,\blintkernel^A)&:=\!
\argmin{\bintkernelvar^E, \bintkernelvar^A \in \hypspace^{\bv}} \!\!\frac{1}{ML}\!\!\sum_{m,l=1}^{M, L }||\ddot{\bX}^m(t_l) - \mF^{\bv}(\dot{\bX}^m(t_l), \Xi^{m}(t_l))  \\
&\quad - \rhsfo_{\bintkernelvar^E}(\bX^m(t_l)) - \rhsfo_{\bintkernelvar^A}(\bX^m(t_l), \dot{\bX}^{m}(t_l))||_{\mathbb{S}}^2\,,
\end{aligned}
\end{equation}
and the interactions acting on the auxiliary variable $\xi_i$ can be solved for separately as
\[
\blintkernel^{\xi} := \argmin{\bintkernel^\xi \in \hypspace^\xi}\frac{1}{ML}\!\!\!\sum_{m = 1, l = 2}^{M, L}\!\!\!||{\dot{\Xi}^{m}_l - \mF^{\xi}(\Xi^{m}_l) - \rhsfo_{\bintkernel^{\xi}}(\bX^m_l, \Xi^{m}_l)}||_{\mathbb{S}},
\]
where $\dot{\Xi}^{m}_l = \dot\bX^m(t_l)$, $\bX^m_l = \bX^m(t_l)$, $\Xi^{m}_l = \Xi^m(t_l)$, $\blintkernel^{\xi} = \{\lintkernel_{\idxcl\idxcl'}^\xi\}_{\idxcl,\idxcl'=1}^\numcl$, and the state space norm $||\cdot||_{\mathbb{S}}$ is defined similarly to the first order case.  
Here we are using a vectorized notation for $\bintkernelvar^E, \bintkernelvar^A$, $\hypspace^{\bv}$ (a suitable product hypothesis space).
In order to measure performance, for each pair $(k,k')$, we define a probability measure on $\R_+\times\R_+$
\begin{equation*}
\rho_{T}^{\idxcl\idxcl'}\!\!(r,\dot r)  = \displaystyle \frac{1}{TN_{\idxcl\idxcl'}}\int_{t = 0}^T\E\!\!\!\!\sum_{\substack{i \in C_{\idxcl}, i' \in C_{\idxcl'} i\neq i'}}\!\!\!\!\!\!\!\!\!\!\delta_{r_{ii'}(t),\dot r_{ii'}(t)}(r,\dot r) dt\,,
\end{equation*}
and another probability measure on $\R_+ \times \R_+$,
\begin{equation*}
\rho_{T, r, \xi}^{L,\idxcl\idxcl'}(r, \xi)  = \displaystyle \frac{1}{LN_{\idxcl\idxcl'}}\sum_{l= 1}^{L} \E\!\!\!\!\!\!\sum_{\substack{i \in C_{\idxcl}, i' \in C_{\idxcl'}, i\neq i'}}\!\!\!\!\!\!\delta_{r_{ii'}(t_l), \xi_{ii'}(t)}(r, \xi)\,,
\end{equation*}
where the expectation is with respect to initial conditions distributed according to $\probIC^{\bX}\times\probIC^{\dot\bX}\times\probIC^{\Xi}$, and 
we let $\dot r=\norm{\dot\br}$ (with abuse of notation), $\xi_{ii'}(t) = \abs{\xi_{i'}(t) - \xi_i(t)}$, $N_{\idxcl\idxcl'} = N_{\idxcl}N_{\idxcl'}$ if $\idxcl\neq\idxcl'$ and $N_{\idxcl\idxcl'}= {N_{\idxcl}\choose{2}}$ if $\idxcl=\idxcl'$ (and $N_{\idxcl}>1$, as there is no kernel to learn if $N_{\idxcl}=1$).  Let $\rho_{T,r}^{\idxcl\idxcl'}$ be the marginal of $\rho_{T}^{\idxcl\idxcl'}$ with respect to $r$.  We will measure the errors for $\lintkernel^E_{\idxcl\idxcl'}(r)r$, $\lintkernel^A_{\idxcl\idxcl'}(r)\dot r$ and $\blintkernel^{\xi}_{\idxcl\idxcl'}(r)\xi$ in $L^2(\rho_{T,r}^{\idxcl\idxcl'})$, $L^2(\rho_{T}^{\idxcl\idxcl'})$ and $L^2(\rho_{T, r, \xi}^{\idxcl\idxcl'})$ respectively.

The algorithm to construct the estimator in \eqref{e:discretecomputedsecondordererror_v_edmc} generalizes that for the first order single-type agent systems, and involves a least squares problem with a structured matrix with $\numcl^2$ vertical blocks indexed by $(\idxcl,\idxcl')$, accommodating the estimators for the interaction kernels. Note that such LS problem takes into account, as it should, the dependencies in learning the various interaction kernels, all at once.

\begin{figure}[!h]
\begin{subfigure}[b]{0.49\textwidth}    \centering
  \includegraphics[width=\textwidth]{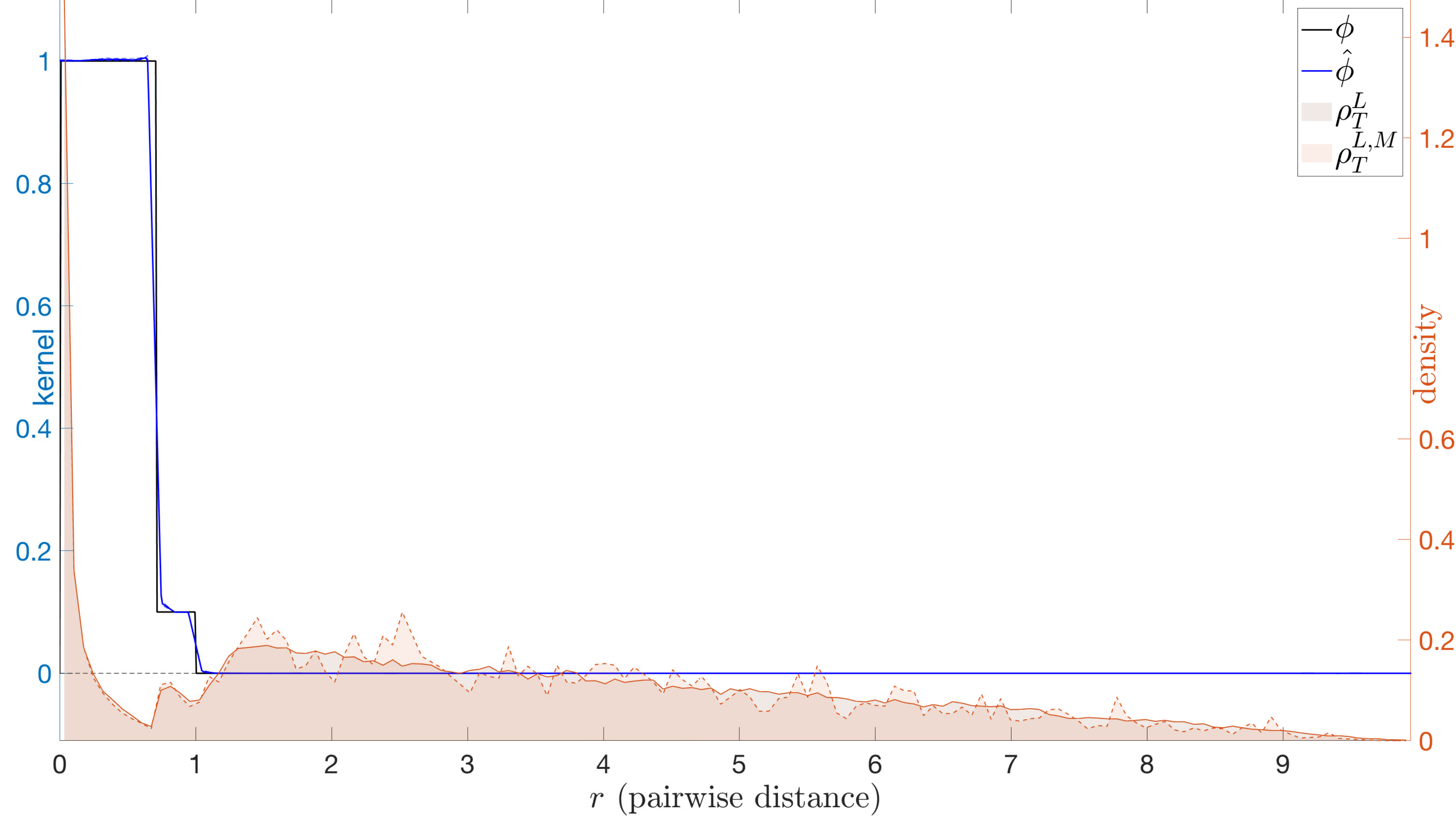}
\end{subfigure}
\begin{subfigure}[b]{0.49\textwidth}    \centering
  \includegraphics[width=\textwidth]{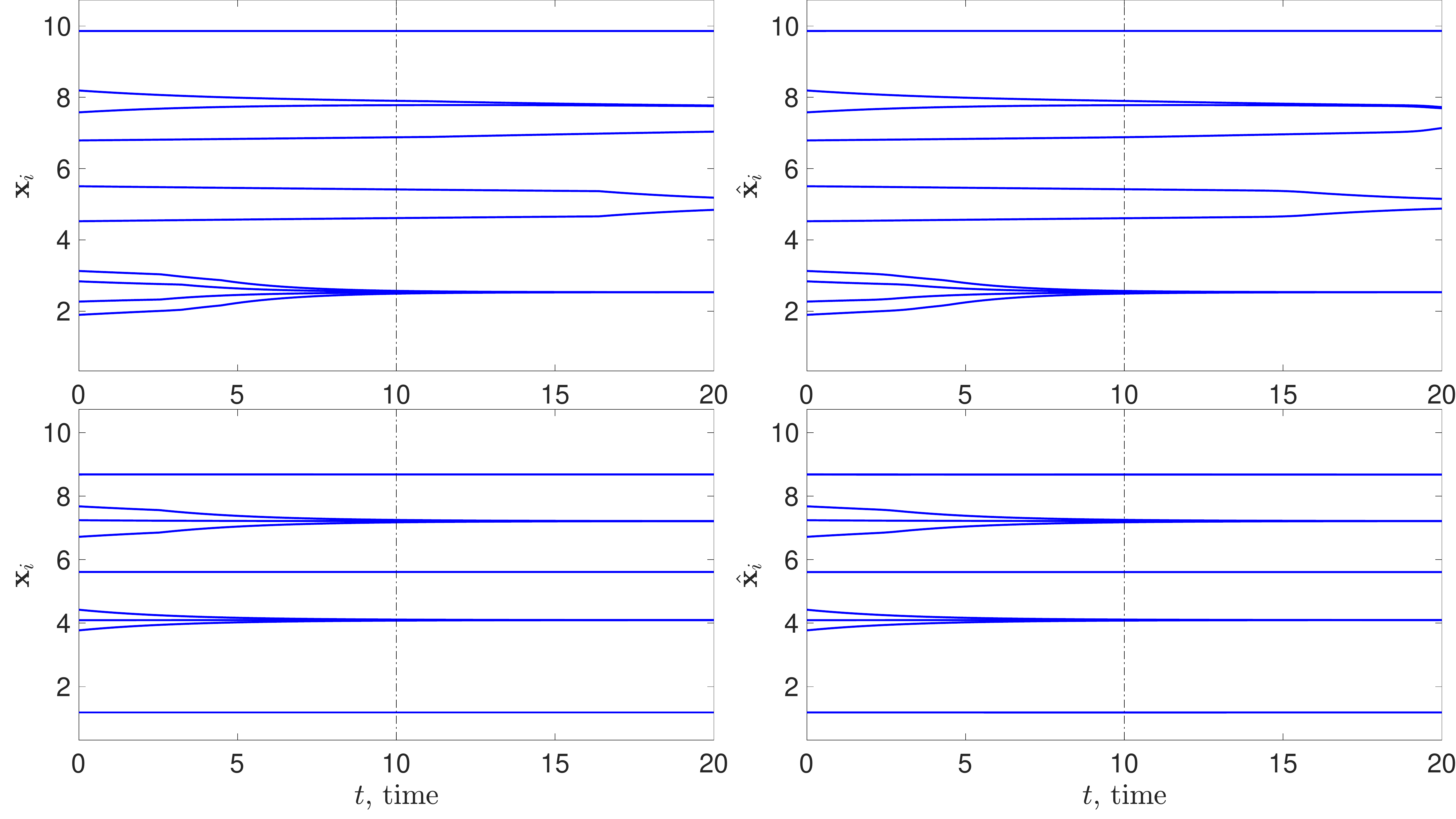}
\end{subfigure} 
\mycaption{
\textbf{Opinion dynamics.} {\textit{Top:}} Comparison between true and estimated interaction kernel, together with histograms for $\rhoL$ and $\smash{\rho_T^{L, M}}$. The mean and standard deviation of the relative error for the interaction kernel are $1.6\cdot10^{-1} \pm 2.3 \cdot10^{-3}$ over 10 independent learning runs. The standard deviation lines (in dash lines) on the estimated kernel are so small to be barely visible.
{\textit{Bottom:}} Trajectories $\bX(t)$ and $\smash{\widehat{\bX}(t)}$ obtained with $\intkernel$ and $\smash{\lintkernel}$ respectively, for an initial condition in the training data (top) and an initial condition randomly chosen (bottom). The black dashed vertical line at $t = T$ divides the ``training'' interval $[0, T]$ from the ``prediction'' interval $[T, T_f]$ (which in this case, $T_f = 2T$).  We achieve small errors in all cases, in particular predicting number and location of clusters for large time.}
\label{fig:example_main_OD}
\end{figure}

We note that while of course the second order system may be written as a first order system in the variables $\bx_i$ and $\bv_i = \dot{\bx}_i$; even when $\forcev_i\equiv0$ and $\intkernela_{\clof_i,\clof_{i'}} \equiv 0$, the resulting equations for $(\bx_i, \bv_i)$ are different from those governing the first order systems considered above in \eqref{e:firstordersystem}.  

\begin{figure}
\begin{subfigure}[b]{0.49\textwidth} \centering
  \includegraphics[width=\textwidth]{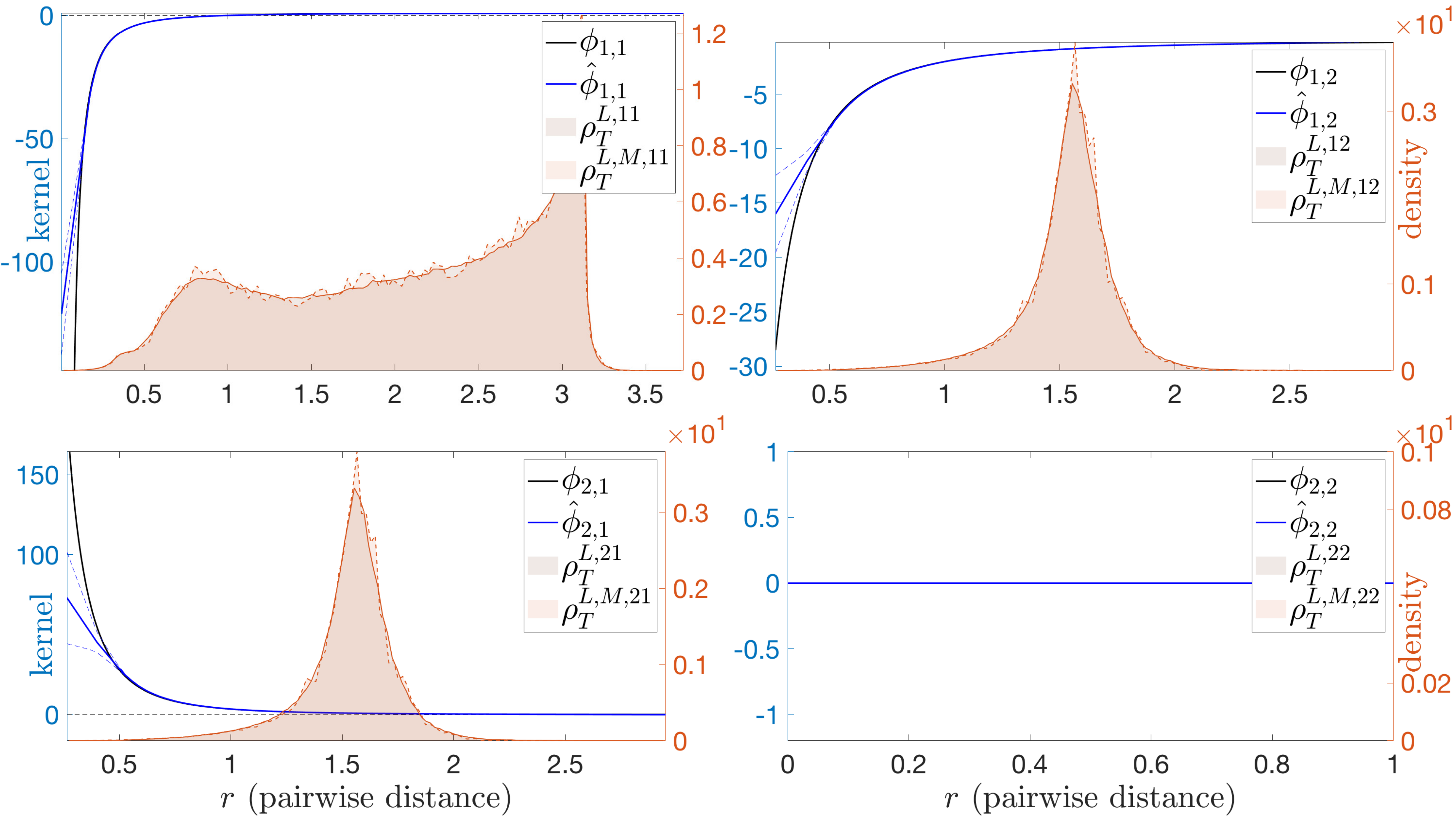}
\end{subfigure}
\begin{subfigure}[b]{0.49\textwidth}   \centering
  \includegraphics[width=\textwidth]{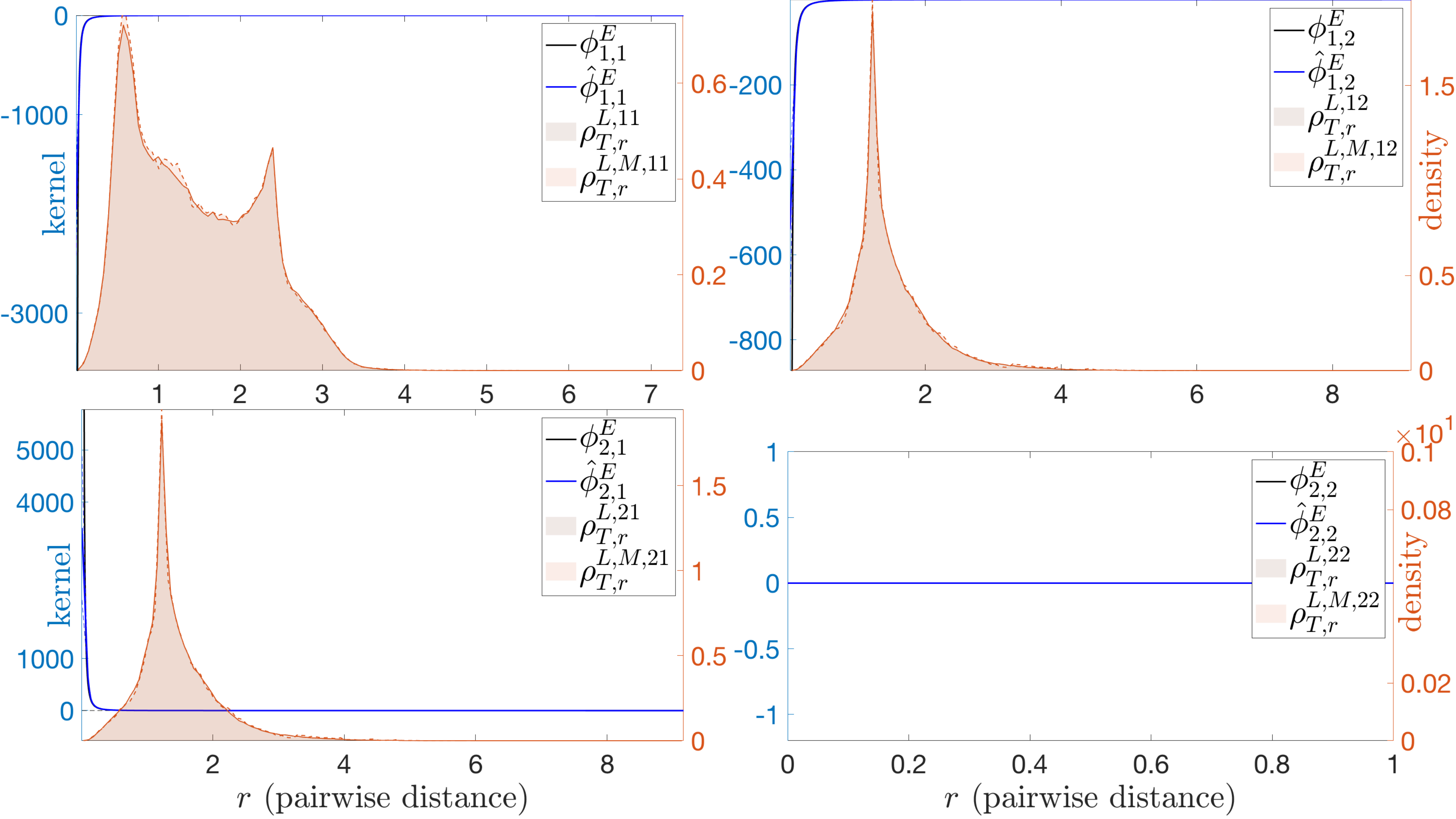}
\end{subfigure}
\begin{subfigure}[b]{0.49\textwidth}    \centering
  \includegraphics[width=\textwidth]{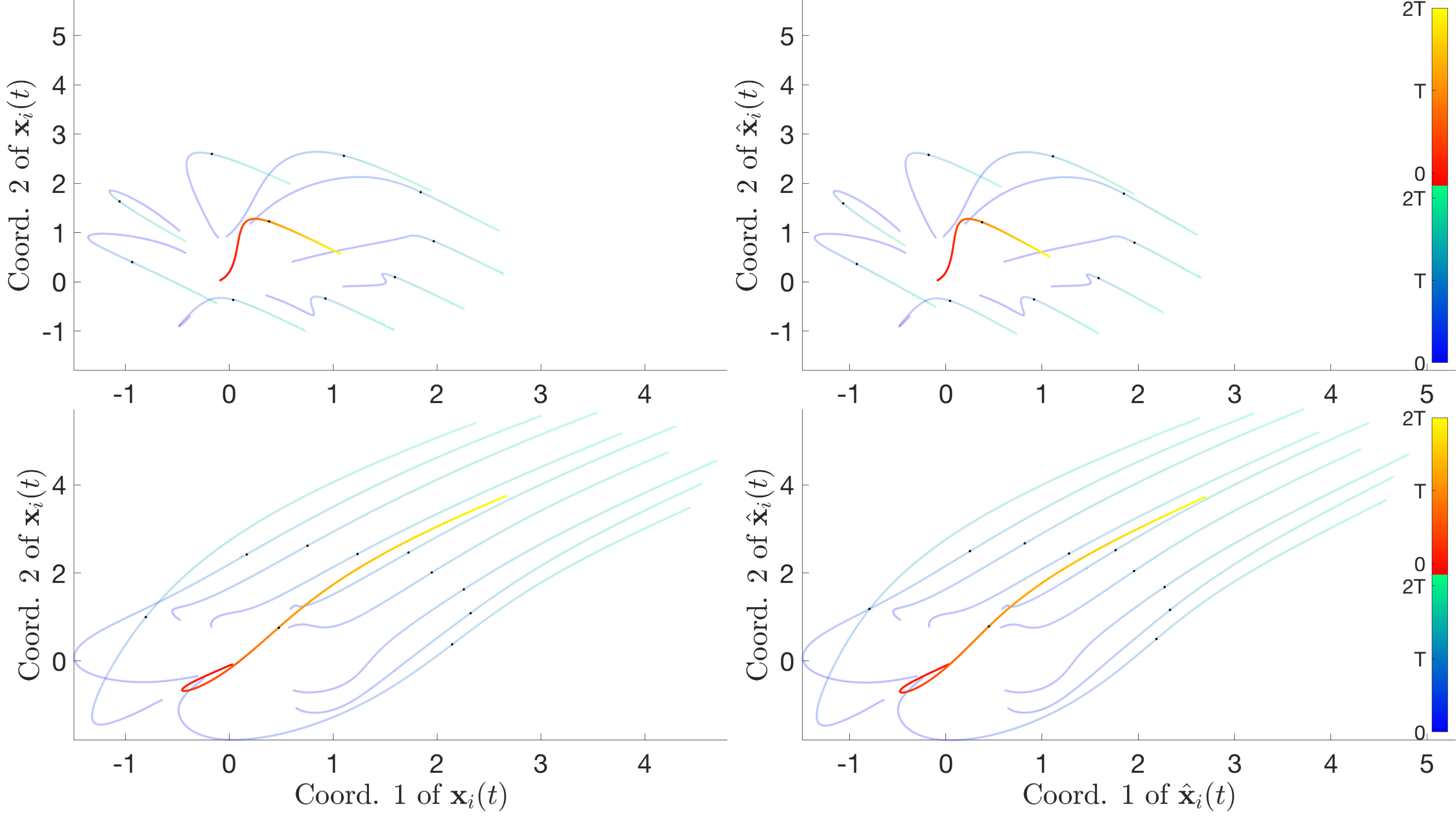}
    \mycaption{ {\em{Predator-Prey, $1^{st}$ order}}.   }
\end{subfigure}
\begin{subfigure}[b]{0.49\textwidth} \centering
  \includegraphics[width=\textwidth]{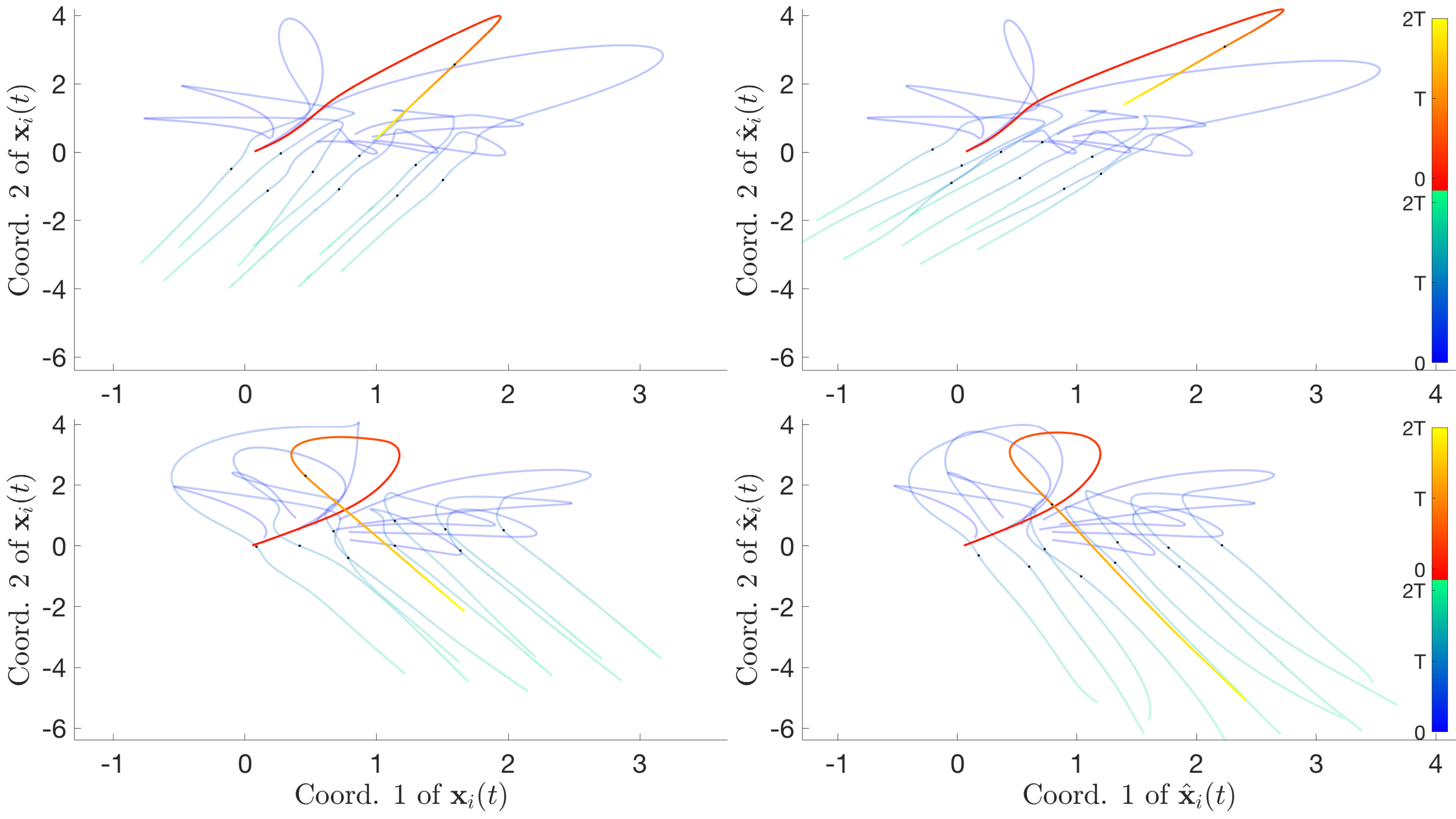}
    \mycaption{ {\em{Predator-Prey, $2^{nd}$ order}}.      }
\end{subfigure}
\mycaption{\textbf{Estimation of interaction kernels and trajectory prediction for Predator-Swarm $1^{st}$ and $2^{nd}$ order Systems.} Results for the $1^{st}$ and $2^{nd}$ order Predator-Swarm systems, as described in Section \ref{s:extensions}-\ref{s:examples}. For each system (corresponding to each column), the {\textit{top row}} represent $\intkerneltrue_{k,k'}$ and $\smash{\widehat\intkernel_{k,k'}}$, superimposed with the histograms of $\rhoL$ (estimated from a large number of trajectories, outside of training data) and $\smash{\rho_T^{L,M}}$ (estimated from the $M$ training data trajectories, see \ifarXiv \eqref{e:rhoLM} \fi). The {\textit{bottom row}} show trajectories $\bX(t)$ and $\smash{\widehat\bX(t)}$ of the corresponding (original and estimated) systems, evolved from the same ICs as the training data (3rd row) and newly sampled ICs (4th row), over both the training time interval $[0,T]$ and in the future $[T,2T]$ (see color bars; the black dots in the trajectories correspond to $t = T$). \revision{For trajectories generated by the Predator-Swarm system, red-to-yellow lines indicate the movement of predators, whereas the blue-to-green lines indicate the movement of preys. The color gradients indicate time, see the colors scales on the side of the plots.} The estimators $\smash{\hat\intkernel_{k,k'}}$ perform extremely well: with negligible differences in the regions with large $\rhoL$ and with possibly larger errors in regions with small $\rhoL$ (where the standard deviations over $10$ independent learning runs become visible). The $L^2(\rhoL)$ errors of the estimators are reported numerically in \ifarXiv Sec.~\ref{s:SIExamples} \fi. Note that they are truncated to a constant while preserving continuity, when there are no samples (e.g.  $r$ near $0$ or $r$ very large). The measure $\rhoL$ is quite smooth but can have interesting features; $\rho_T^{L,M}$ is typically a noisy version of $\rhoL$. The trajectories of the estimated system are typically good approximations to those of the original system, on both ICs in the training data and newly sampled ICs. The error of the estimated trajectories increases with time, as expected, albeit it still typically excellent also in the ``prediction'' time interval $[T,2T]$, showing that the bounds in Prop.~\ref{stateestimation}, while sharp in general, may be overly pessimistic in some practical cases. Some slightly larger errors are present in some trajectories, e.g. when preys and predators get much closer to each other than they did in the training data. 
}
\label{fig:example_main}
\end{figure}

\begin{figure}
\begin{subfigure}[b]{0.48\textwidth}    \centering
  \includegraphics[width=\textwidth]{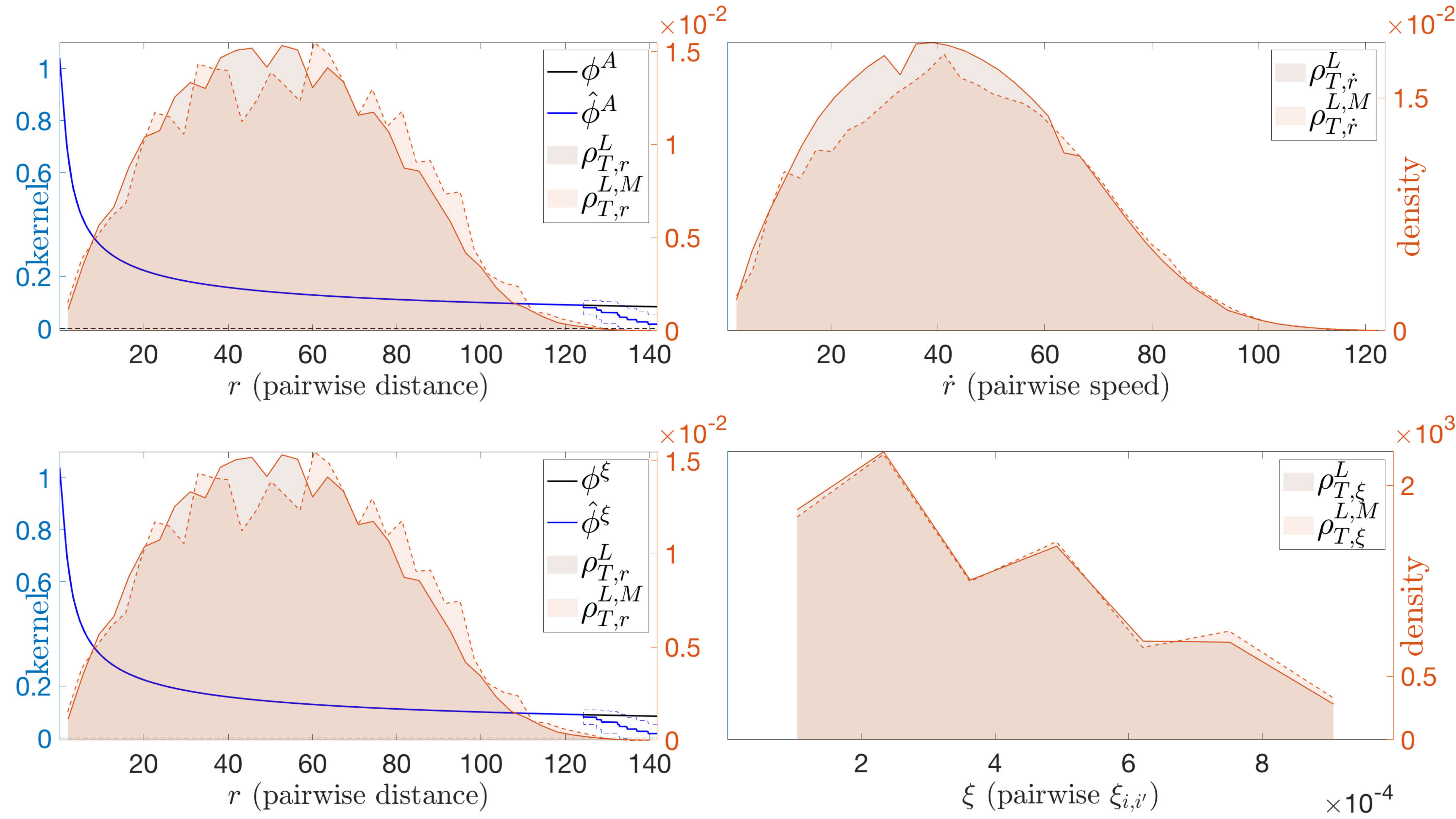}
  \mycaption{{\em{Phototaxis}}.  Interaction Kernels}
\end{subfigure}
\begin{subfigure}[b]{0.48\textwidth}    \centering
  \includegraphics[width=\textwidth]{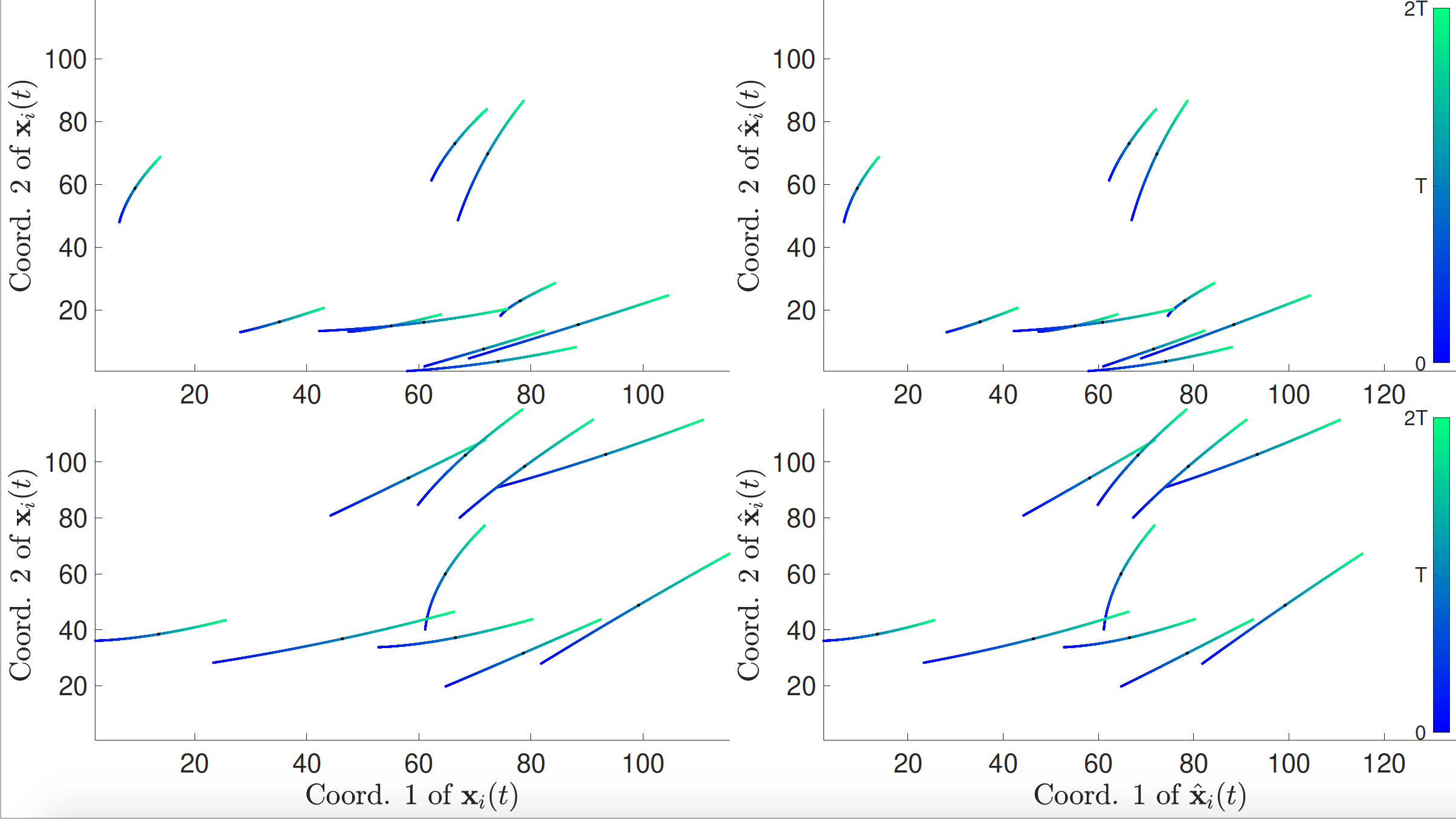}
  \mycaption{{\em{Phototaxis}}.  Trajectories}
\end{subfigure} 
\mycaption{ \textbf{Estimation of interaction kernels and trajectory prediction for the Phototaxis system.} Results for the Phototaxis systems, as described in Sec.~\ref{s:extensions}-\ref{s:examples}.  {\em{Top:}} the first column represents $\intkernel^A$ versus $\hat\intkernel^A$ (top), and $\intkernel^\xi$ versus $\hat\intkernel^\xi$ (bottom), superimposed with the histograms of $\rho^L_{T, r}$ and, respectively, $\rho^{L, M}_{T, r}$. The second column shows the comparison of the marginal distributions, $\rho^L_{T, \dot r}$ versus $\rho^{L, M}_{T, \dot r}$ and $\rho^L_{T, \xi}$ versus $\rho^{L, M}_{T, \xi}$. {\em{Bottom:}} The left column represents the trajectories generated from true interaction kernels, whereas the second column shows the trajectories generated by the estimated kernels, generated from training IC data (top row) and from a new random IC (bottom row).  In this system the interaction kernels $\intkernel^A$ and $\intkernel^\xi$ are the same; the corresponding estimators $\hat\intkernel^A$ and $\hat\intkernel^\xi$ are both learned accurately, but note that they are being learned from two different sets of data, $(r, \dot r)$ and $(r, \xi)$ respectively. In both cases, data is scarce or missing for large values $r$, leading to estimators tapering to $0$ faster than the true interaction kernels. However, despite the undesired tail end behavior of our estimators, the estimators perform extremely well in re-generating the trajectories.  See \ifarXiv Sec.~\ref{s:SIExamples} \fi for more details.
}
\label{fig:example_main_PT}
\end{figure}

%
\section{Examples}
\label{s:examples}
We consider the learning of interaction kernels and the prediction of trajectories for three canonical categories of examples of self-organized dynamics (see \ifarXiv Sec.~\ref{s:SIExamples}\fi for details). 

\noindent{\bf{Opinion dynamics.}} These are first-order ODE systems with a single type of agent, with bounded, discontinuous, compactly supported and attraction-only interaction kernels. \revision{They model how the opinions of people influence each other and how consensus is formed based on different kinds of influence functions (see \cite{Krause2000, MT2014, CKFL2005} and references therein).}
 
\noindent{\bf{Predator-Swarm System}}.
We consider a first-order system with a single predator and a swarm of preys, with the interaction kernels (prey-prey, predator-prey, prey-predator) similar to Lennard-Jones kernels (with appropriate signs to model attractions and repulsions).  \revision{Different chasing patterns arise depending on the relative interaction strength of predator-prey vs. prey-predator interactions}.  We also consider a second order Predator-Swarm system, with the collective interaction acting on accelerations, leading to even richer dynamics and chasing patterns (see e.g.~\cite{CK2013, JT2007, ZKHFK2005}).

\noindent{\bf{Phototaxis}}. This is a second order ODE system with a single type of agents interacting in an environment, \revision{modeling phototactic bacteria moving towards a far away fixed light source.  The response of the bacteria to the light source is represented in the auxiliary variable $\xi_i$ as the excitation level for each bacteria $i$ (see e.g.~\cite{HL2009, SB2001, BTG2001}).  Another example which we do not pursue here is the Vicsek model \cite{Vicsek_model}, which fits perfectly in our model upon choosing $\xi_i = \theta_i$ ($\theta_i$: moving direction of agent $i$).}

In our experiments we report the measure $\rhoLM$ estimated from the training data, our estimator, and similarly in the case of noisy observations; we measure performance in terms of (relative) $L^2(\rhoL)$ error of the kernel estimators and of distance between true trajectories $\bX(t)$ and estimated trajectories $\smash{\widehat{\bX}(t)}$, on both the ``training'' interval $[0,T]$ (where observations were given) and in the future $[T,2T]$ (predictions). See Prop.~\ref{stateestimation}, where the bounds may be overly pessimistic, especially for systems tending to stable configurations.
\revision{Our estimator performs extremely well in all these examples: the interaction kernels are accurately estimated and the trajectories are accurately predicted. We refer the reader to Fig.~\ref{fig:example_main_OD} for the results of the opinion dynamics, Fig.~\ref{fig:example_main} for the results of the predator-swarm dynamics and Fig.~\ref{fig:example_main_PT} for the results of the phototaxis, and to \ifarXiv Sec.~\ref{s:SIExamples} \fi for further details on the setup for the experiments and a comprehensive report of all the results.} 

\begin{figure}[!h]
\begin{subfigure} [b]{0.49\textwidth}    \centering
  \includegraphics[width=\ifarXiv 0.99 \fi\textwidth]{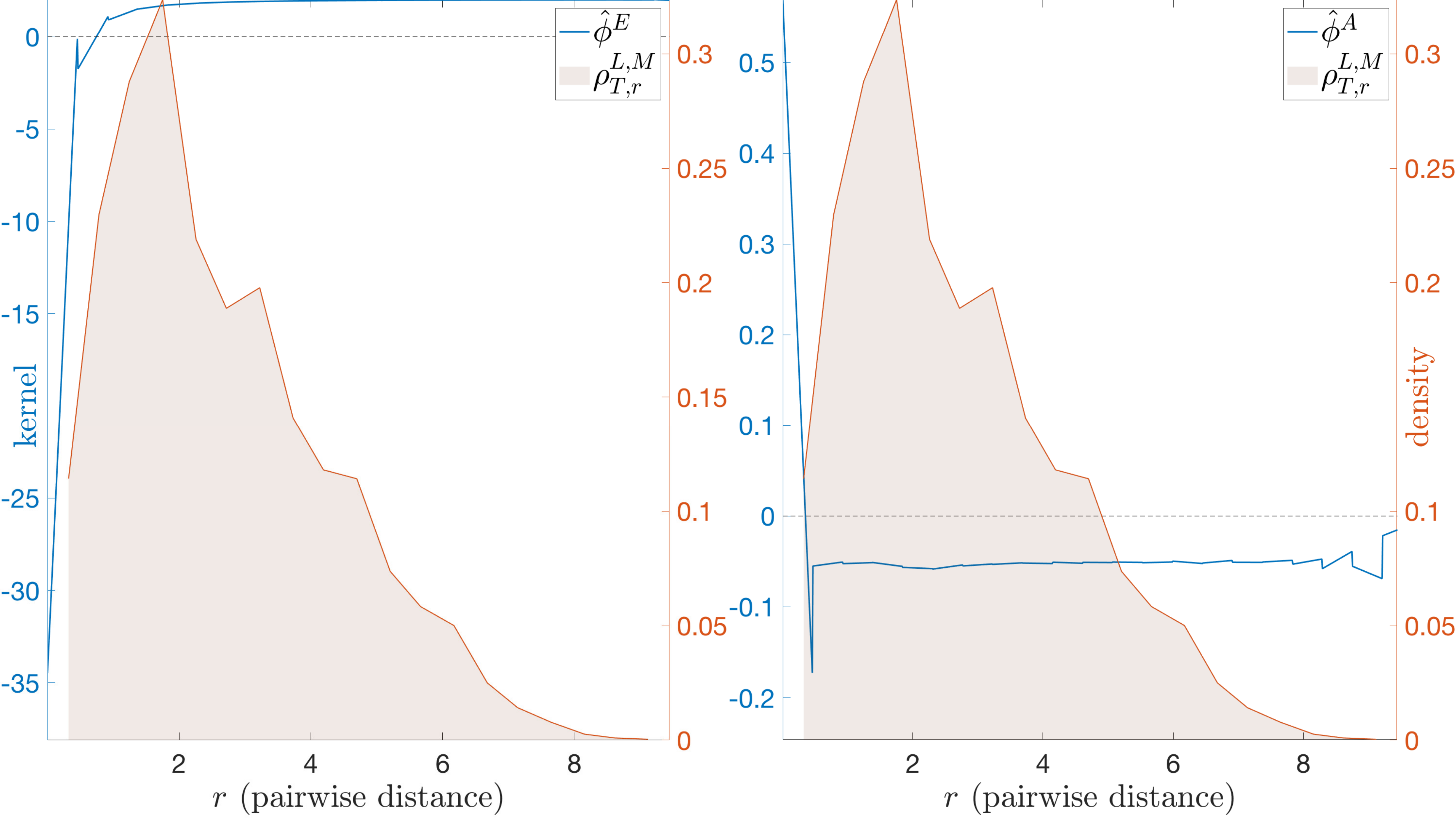} 
  \subcaption{Energy-based model}
  \end{subfigure}  
\begin{subfigure}[b]{0.49\textwidth}    \centering
  \includegraphics[width=\ifarXiv 0.99 \fi\textwidth]{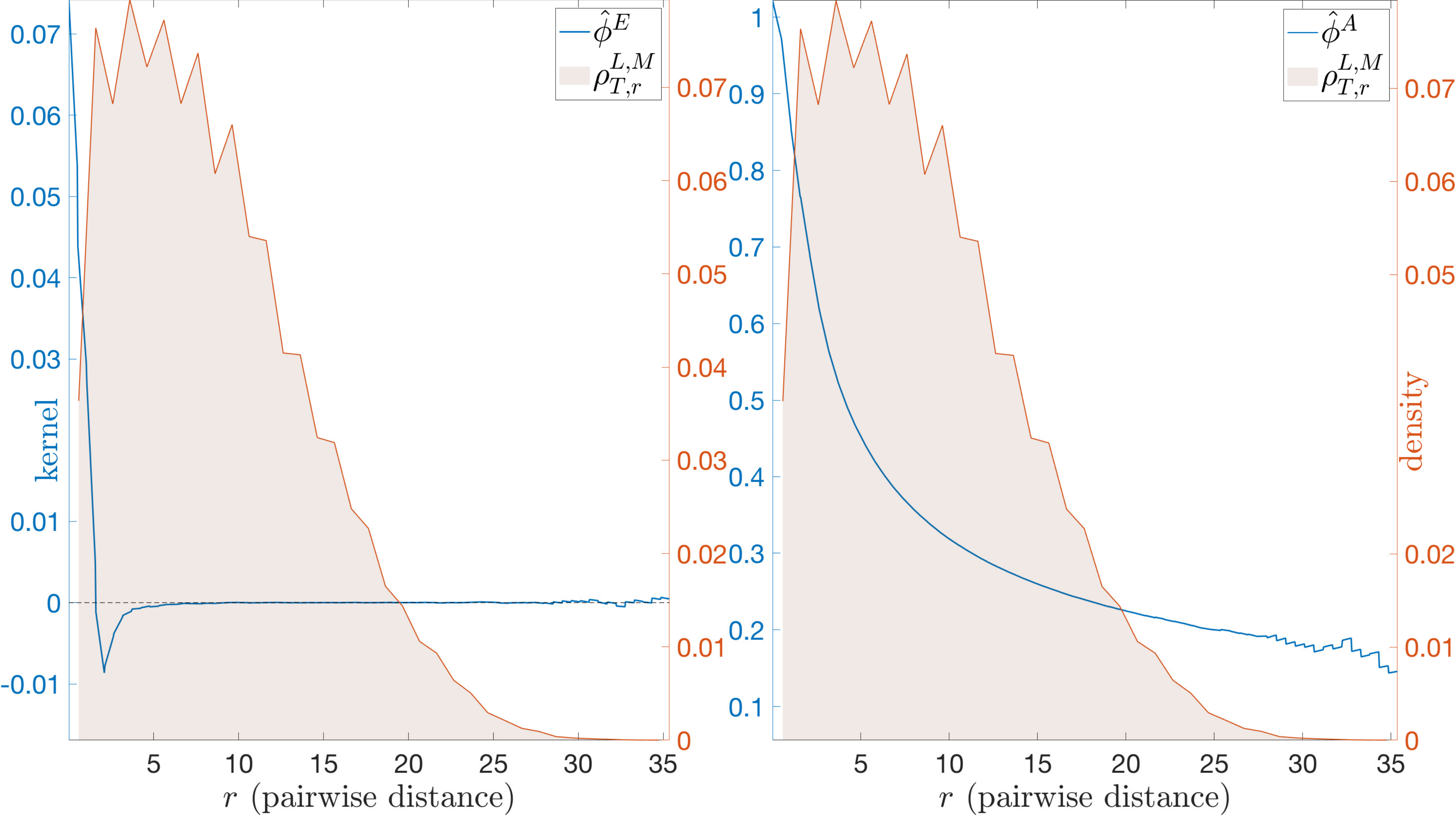}
   \subcaption{Alignment-based model}
    \end{subfigure}
\mycaption{{\textbf{Model Selection: energy-based vs. alignment-based}}. The estimated interaction kernels for an energy-based model (\textit{top row}) and an alignment-based model (\textit{bottom row}). \revision{For each model, we compute two estimators: an energy-based interaction kernel $\lintkernele$ and an alignment-based interaction kernel $\lintkernela$.} Our estimators correctly identify the type of model in each case: 
the $L^2(\rho^L_{T, r})$ norm of $\lintkernele$ is significantly larger than that of $\lintkernela$ (means and std.'s:  $\smash{\mathbf{18.8 \pm 0.4}}$ vs. $\smash{6.5 \pm 0.3}$) for the energy-based model, and the $\smash{L^2(\rho^L_{T, r})}$ norm of $\smash{\lintkernela}$ is larger than that of $\smash{\lintkernele}$ (means and std.'s:   $\mathbf{27.6 \pm 0.7}$ vs.~$2.4 \cdot 10^{-2}\pm 0.1$  ) for the alignment-based model.  Note the $y$-axes are on very different scales.}
\label{fig:ms_cases1}
\end{figure}

 \paragraph{Model Selection \revision{and Transfer Learning}.}
We also consider the use of our method for model selection, where the theoretical guarantees on learning the interaction kernels and on predicting trajectories are used to decide between different models for the dynamics.
We consider two examples of model selection, to test whether: (i) a second order system is driven by energy-based or alignment-based interactions; (ii) a heterogeneous agent system is driven by first order or second order ODE's.  For each of them, we construct two estimators assuming either case, then select models according to the performance of the estimators in predicting trajectories. See Table \ref{tab:ms_cases2} and Fig.~\ref{fig:ms_cases1} for results and discussions, and \ifarXiv Sec.~\ref{s:SI_MS}  \fi for details.

\revision{As a simple example of transfer learning, we use the interaction kernel learned on a system with $N$ agents to accurately predict trajectories of the same type of system but with more agents ($4N$ in our simulations); the interaction kernel acts as a sort of ``latent variable'' that seamlessly enables transfer across such related systems. In \ifarXiv Sec.~\ref{s:SIExamples} \fi we report the corresponding results, for all the systems considered (see however Fig.~\ref{f:LJ_main} for the Lennard-Jones system).}

\begin{table}[ht!]
\mycaption{ \textmd{\footnotesize{\textbf{Model Selection: first order vs. second order.} 
The table shows the mean and standard deviation of the errors of estimated trajectories, over $M =250$ train-test runs, with random initial conditions in each case. Small errors, consistent with our theory that the errors are on a scale of $M^{-2/5}$, indicate a correct model. The order is correctly identified in each case (highlighted in bold). 
}}} \vspace{-2mm}
  \centering
\footnotesize{\begin{tabular}{ c | c | c}
\hline
                  & Learned as $1^{st}$ order             & Learned as $2^{nd}$ order \\
\hline
$1^{st}$ order system   &  $\mathbf{0.01 \, \pm 0.002}$    &  $1.6\, \pm 1.1$ \\
$2^{nd}$ order system  &  $1.7 \, \pm 0.3$                        & $\mathbf{0.2 \, \pm 0.06}$ \\
\end{tabular}}  
\label{tab:ms_cases2} 
\end{table}
\revision{
\paragraph{Noisy observations.} 
 Our estimators appear robust under observation noise, namely if the observed positions and derivatives are corrupted by noise. Fig.~\ref{fig:PS1noisetraj} demonstrates the kernel estimation and trajectory prediction for the first-order Predator-Swarm system when only noisy observations are available. Similar results (reported in  \ifarXiv Sec.~\ref{s:SIExamples}\fi) are obtained in all the other systems considered. }
 
\begin{figure}[!t]
\centering
\includegraphics[width=\ifarXiv0.49\fi\textwidth]{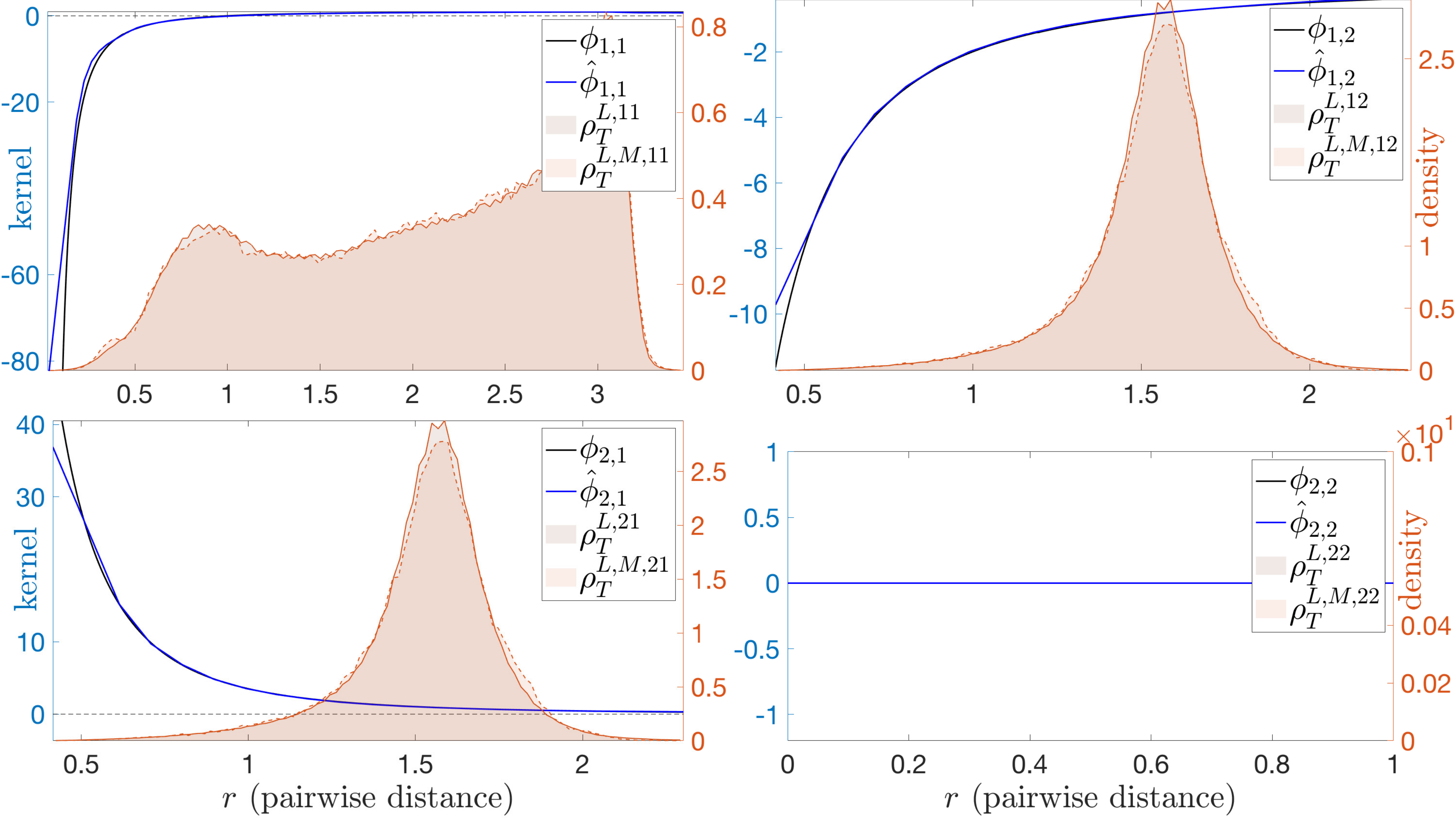}
\includegraphics[width=\ifarXiv0.49\fi\textwidth]{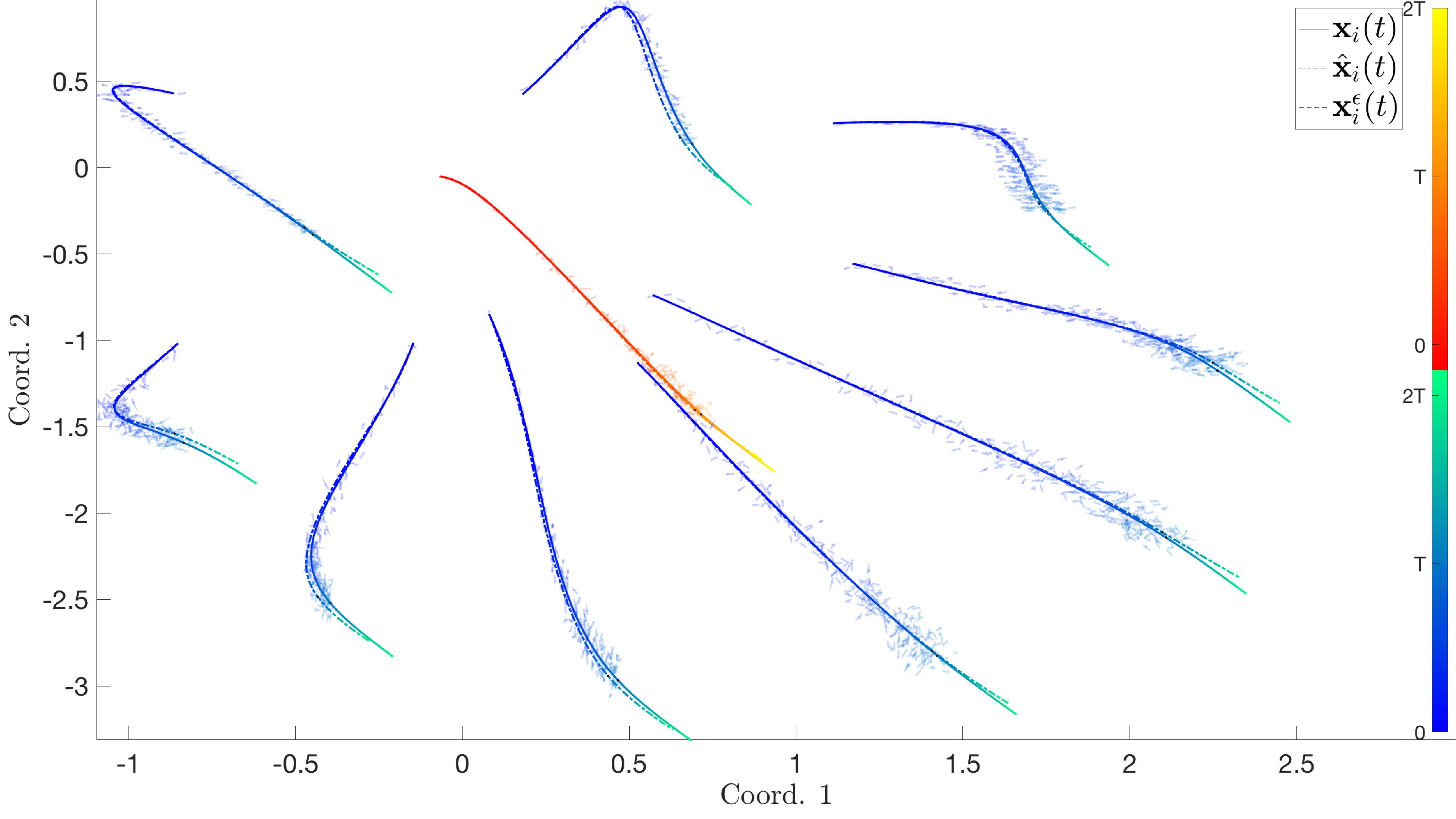}
\mycaption{\textbf{Kernel estimation for PS$1^{st}$ from noisy observations.} \textit{Top}: interaction kernels learned with Unif.$([-\sigma,\sigma])$ multiplicative noise, for $\sigma=0.1$ in the observed positions and velocities, with parameters as in Table \ifarXiv \ref{t:PSparams_1}\fi. The estimated kernels are minimally affected, and only in regions with small $\rhoL$. \textit{Bottom}: one of the observed trajectories before and after being perturbed by noise.  The solid lines represent the true trajectory, the dashed semi-transparent lines represent the noisy trajectory used as training data (together with noisy observations of the derivative, not shown), and the dash-dotted lines are the predicted trajectory learned from the noisy trajectory.}
\label{fig:PS1noisetraj}
\end{figure}
\revision{
\paragraph{Choice of the basis of the hypothesis space.} Our learning approach is robust to the choice of hypothesis space $\hypspace$, as long as the coercivity condition is satisfied by $\hypspace$ (or the sequence $\hypspace_n$). Additionally, different well-conditioned bases may be used in $\hypspace$ to compute the projection onto $\hypspace$, implying, together with the coercivity condition, a control of the condition number of the least squares problem (see \ifarXiv Prop.~\ref{p:sigmaminAL} \fi).  To demonstrate this numerically,  we compare the B-splines linear basis with the piecewise polynomial basis on the $1^{st}$-order Predator-Swarm system, with results shown in \ifarXiv Fig.~\ref{fig:PS1splines} in SI.\fi}

\section{Discussion and Conclusion}
We proposed a non-parametric estimator for learning interaction kernels from observations of agent systems, implemented by computationally efficient algorithms. We applied the estimator to several classes of systems, including first- or second-order, with single- or multiple-type agents, and with simple environments. We have also considered observation data from different sampling regimes: many short-time trajectories, a single large-time trajectory, and intermediate time scales. 

Our inference approach is non-parametric, does not rely on a dictionary of hypotheses (such as in \cite{Schaeffer6634,BPK2016,TranWardExactRecovery}), exploits the structure of dynamics, and enjoys  optimal rates of convergence (which we proved here for first-order systems), independent of the dimension of the state space of the system. 
Having techniques with solid statistical guarantees is fundamental in establishing trust in data-driven models for these systems, and in using them as an aide to the researcher in formulating and testing conjectures about models underlying observed systems. In this vein, we presented two examples of model selection, showing that our estimators can reliably identify the order of a system, and identify whether a system is driven by energy- or alignment-type interactions.

We expect further generalizations to the case of stochastic dynamical systems and to the cases of more general interaction kernels that depend on more general types of interaction between agents, beyond pairwise, distance-based interactions.
Other future directions include (but are not limited to) a better understanding of learnability, model selection based on the theory, learning from partial observations, and learning reduced models for large systems. 


%

\ifarXiv
\section*{Acknowledgement}
\theAcknowledge
\fi
\ifarXiv
\pagebreak
\begin{center}
\textbf{\large Supplemental Information (SI)}
\end{center}
\fi
\section{Learning Theory}
Consider the problem of estimating the interaction kernel $\intkernel: \R_+\to \R$ of the dynamical system as follows
\begin{equation}
\label{e:1stordersystem}
  \dot{\bx}_i(t) = \frac1N\sum_{i' = 1}^{N}\intkernel(\norm{\bx_{i'}(t) - \bx_i(t)}{})(\bx_{i'}(t) - \bx_i(t))\,,
\end{equation}
from observations of discrete-time trajectories and derivatives,  $\{\bXm(t_l)\}$ and $\{\dot{\bX}^{m}(t_l)\}$ with $0=t_1<\dots<t_L=T$ and $m=1,\dots,M$. We let $\bX:=(\bx_i)_{i=1}^N\in\R^{\dim N}$ be the state space variable. The initial conditions  $\bX_0^{m}:=\bXm(0)$ are sampled independently from a probability measure $\probIC$ on $\R^{dN}$. 

Such a system can also be described as the gradient flow 
$\dot{\bX}=\rhsfo_\intkernel(\bX) = \nabla \mathcal{U}(\bX)$ of the potential energy $\mathcal{U}(\bX) =\frac{1}{2N}\sum_{i,i'} \Phi(\|\bx_i-\bx_{i'}\|)$, with the function $\Phi:\R_+\to \R$ satisfying $\Phi'(r) = \phi(r)r$.  Therefore, the estimation of $\intkernel$ is equivalent to the estimation of $\Phi'$. As we will see later, the function $\intkernel(\cdot)\cdot$ appears naturally in assessing the quality of approximation of estimators of $\intkernel$, the fundamental reason being the relationship with the potential involving $\Phi$.

We restrict our attention to kernels in the \emph{admissible set}
\begin{align} \label{eq_admSet}
\mathcal{K}_{R,\spaceM }: = \{\intkernel \in & W^{1,\infty}:  \supp{\intkernel}\in [0,R], \sup_{r\in[0,R]} \left[|\intkernel(r)| + | \intkernel'(r) | \right]\leq \spaceM \}
\end{align}
for some $R,\spaceM >0$.  The boundedness of $\intkernel$ and its derivative ensures the existence and uniqueness of a global solution to initial value problems of the first order system \eqref{e:1stordersystem}, and the continuous dependence of the solution on the initial condition. The restriction $\supp \intkernel \subset  [0,R]$ represents the finite range of interaction between particles, and this restriction may be replaced by functions with unbounded support but with a suitable decay on $\R_+$. 

We shall construct an error functional based on the special structure of the dynamical system $\dot{\bX}=\rhsfo_\intkernel(\bX)$, taking advance of the form of the dependency of the right-hand side $\rhsfo_\intkernel$ on the interaction kernel $\intkernel$. 
This learning procedure deviates from standard regression in two aspects: 
(i) the values of the interaction kernel are not observed, and cannot be explicitly estimated from the observations of the state variables; (ii) the observations of the independent variable of the interaction kernel, given by the pairwise distance between the agents, though abundant, are not independent and may be redundant. 

We would also like to stress the importance of using a carefully chosen measure on the pairwise distance space, so as to account for both the randomness from the initial conditions and the evolution of the dynamical system, and to reflect the (relative) abundances of pairwise distances. Our analysis shows that the expectation of the empirical measure of the pairwise distances is a natural choice, and it is closely related to the coercivity condition, the other fundamental ingredient which ensures learnability and convergence of the estimators.  

\subsection{The Error functional and estimators} 
Given the structure of the first order system \eqref{e:1stordersystem}, we consider the error functional
\begin{equation}
\mE_{L,M}(\intkernelvar) := \frac{1}{MN}\sum_{l,m,i= 1}^{L,M,N} w_l \big\|\dotbxm_i(t_l)-\rhsfo_\intkernelvar( \bxm(t_l))_i\big\|^2\, ,
 \label{e:errorFnal}
\end{equation}
where $\{w_l\}_{l=1}^L$ is a normalized set of weights ($w_l>0$ and $\smash{\sum_{l=1}^L w_l=1}$), and define an estimator
\begin{equation}\label{e:LSE}
\widehat\intkernel_{L,M,\hypspace} := \argmin{\intkernelvar\in\hypspace} \mE_{L,M}(\intkernelvar),
\end{equation}
where $\hypspace$ is a suitable class of functions that will be referred as hypothesis space. 
Natural choices of weights $\{w_l\}$ may be chosen to be all equal to $1/L$, as in the case of equispaced $t_l$'s, which is what we considered throughout the paper, and is consistent with the definition of $\rhoL$ and its use in measuring the performance of the estimator in $L^2(\rhoL)$. However, if one wished to measure the performance in a different $L^2$ space, one could choose the weights differently. A distinguished choice would be $L^2(\rho_{\mathrm{Lebesgue}})$, in which case one may choose $w_l=1/(t_{l+1}-t_l)$, for $l=1,\dots,L-1$ (and change all the summations involving $l$ to stop at $L-1$ instead of $L$). Other choices of weights corresponding to other quadrature rules are also possible.

Note that the error functional is quadratic in $\intkernelvar$ and bounded below by $0$, therefore the minimizer exists for any finite-dimensional convex hypothesis spaces $\hypspace$. We can always truncate this minimizer so that it is bounded above by $S$, the upper bound of the functions in the admissible set $\mathcal{K}_{R,\spaceM }$, and this truncated estimator behaves similarly to the estimator obtained by assuming that the functions in $\hypspace$ are uniformly bounded. In fact, such truncation can only reduce the error. Hence, without loss of generality, we assume $\hypspace$ to be a compact set in the $L^{\infty}$ norm.

Our objectives are measuring the quality of approximation of the estimator and finding the hypothesis spaces for which the optimal rate of convergence of $\smash{\hat\intkernel}$ to the true interaction kernel $\intkerneltrue$ is achieved.

\subsection{Measures on the pairwise distance space}
We introduce a probability measure on $\R_+$, to define a suitable function space that contains all the estimators and the true interaction kernel, and to provide a norm to assess the accuracy of the estimators.  We let
\[
\mathbf{r}_{ii'}(t) = \bx_{i'}(t) - \bx_i(t), \text{ and }  r_{ii'}(t) =  \|\bx_{i'}(t) - \bx_i(t)\|.
\]
Note that the independent variable of the interaction kernel is the pairwise distances $\nbrm_{ii'}(t)$, which can be computed from the observed trajectories. It is natural to start from the empirical measure of pairwise distances 
\begin{equation}
\rhoLM(r)  =\frac{1}{\binom N2 L M}\sum_{l,m= 1}^{L,M}\sum_{i,i'=1,\\ i< i' }^N \delta_{\nbrm_{ii'}(t_l)}(r)\,,
\label{e:rhoLM}
\end{equation} 
which tends, as $M\to \infty$, using the law of large numbers, to $\rhoL$ defined in (5) in the main text.
When trajectories are observed continuously in time, the counterpart of $\rhoL$ is the measure defined in (5).
We now establish basic properties of these measures:
\begin{lemma}\label{averagemeasure}
For each $\intkernel \in \mathcal{K}_{R,\spaceM }$ defined in \eqref{eq_admSet}, the measures $\rhoL$ and $\rhoT$ defined in (5) and (4) in the main text are Borel probability measures on $\R_+$. They are absolutely continuous with respect to the Lebesgue measure provided that $\probIC$ is absolutely continuous  with respect to the Lebesgue measure on $\R^{dN}$. 
\end{lemma}

\subsection{Learnability: the coercivity condition} 
\label{s:SIcoercivity}
A fundamental question is the learnability of the true interaction kernel, i.e. the well-posedness of the inverse problem of kernel learning.  Since the estimators $\smash{\lintkernel_{L,M,\hypspace}}$ always exists for suitably chosen hypothesis spaces $\hypspace$ (e.g. compact sets), learnability is equivalent to the convergence of the estimator $\smash{\lintkernel_{L,M,\hypspace}}$ to the true kernel $\intkerneltrue$ as the sample size increases (i.e. $M\to \infty$) and as the hypothesis space grows. 
To ensure such a convergence, one would naturally wish:  (i) that  the true kernel $\intkerneltrue$ is the unique minimizer of the expectation of the error functional (by the law of large numbers)
\begin{align}
\mE_{L,\infty}(\intkernelvar) &\coloneqq \lim_{M\to\infty}\mE_{L,M}(\intkernelvar) =\frac{1}{L N}\sum_{l,i= 1}^{L,N} \E \left[ \big\| \frac{1}{N}\sum_{i'= 1}^{N}  \left(\intkernelvar -\intkerneltrue\right)(r_{ii'}(t_l))\br_{ii'}(t_l)\big\|^2\right]; \label{e_errorFtnl_infty2} 
\end{align}
(ii) that the error of the estimator, in terms of a metric based on the $L^2(\rhoL)$ norm, can be controlled by the discrepancy between the empirical error functional and its limit.

Note that $\mE_{L,\infty}(\intkernelvar) \geq 0$ for any $\intkernelvar$ and that $\mE_{L,\infty}(\intkernel) =0$. 
Furthermore, \eqref{e_errorFtnl_infty2} reveals that $\mE_{L,\infty}(\intkernelvar)$ is a quadratic functional of $\intkernelvar -\intkerneltrue$, and we have, by Jensen's inequality, 
\begin{align*}
 \mE_{L,\infty}(\intkernelvar) \leq \frac{(N-1)^2}{N^2}\norm{\intkernelvar(\cdot)\cdot- \intkerneltrue(\cdot)\cdot}^2_{L^2(\rhoL)}\,. 
\end{align*}
This inequality suggests the above weighted $L^2(\rhoL)$ norm as a metric on the error of the estimator that we wish to be controlled.  Therefore, as long we as can bound the limit error functional from below by $\norm{\intkernelvar(\cdot)\cdot - \intkerneltrue(\cdot)\cdot}^2_{L^2(\rhoL)}$, we can conclude that $\intkerneltrue$ is the unique minimizer of $\mE_{L,\infty}(\cdot)$ and that the estimators converge to $\intkerneltrue$.  
This suggests the following coercivity condition:
 
\begin{definition}[Coercivity condition] \label{def_coercivity_SI}
We say that the dynamical system defined in \eqref{e:1stordersystem} together with the probability measure $\probIC$ on $\R^{dN}$, satisfies the coercivity condition on $\hypspace$ with a constant $c_L>0$, if 
\begin{align}\label{gencoer}
  c_L\|\intkernelvar(\cdot)\cdot\|_{L^2(\rhoL)}^2  \!\!\leq   \!\!\frac{1}{NL}\sum_{i,l=1}^{L,N}\E \bigg[ \big\| \frac{1}{N}\sum_{i'= 1}^{N}  \intkernelvar(r_{ii'}(t_l))\br_{ii'}(t_l) \big\|^2\bigg]
 \end{align} 
 for all \revision{$\intkernelvar\in \hypspace$ such that $\intkernelvar(\cdot)\cdot\in L^2(\rhoL)$}, with the measure $\rhoL$ defined in (4) in the main text, and the expectation being with respect to initial conditions distributed according to $\probIC$.
\end{definition}

The above inequality is called a coercivity condition because that it implies coercivity of the bilinear functional $\dbinnerp {\cdot, \cdot}$ on $L^2(\real_+,\rhoL)$, 
\begin{align}\label{eq:bilinearFn}
& \dbinnerp {\intkernelvar_1, \intkernelvar_2}:=\frac{1}{LN}\sum_{l,i=1}^{L,N}\mathbb{E}  \bigg[\bigg\langle \frac{1}{N}\sum_{j=1}^{N}\intkernelvar_1(r_{ji}(t_l))\br_{ij}(t_l), \frac{1}{N}\sum_{j=1}^{N}\intkernelvar_2(r_{ji}(t_l))\br_{ij}(t_l) \bigg\rangle \bigg],
\end{align}
as \eqref{gencoer} may be rewritten as 
\[
c_L \norm{\intkernelvar(\cdot)\cdot }^2_{L^2(\R_{+},\rhoL)}\leq  \dbinnerp {\varphi, \varphi}.
\]

The coercivity condition plays a key role in the learning of the interaction kernel. It ensures learnability by ensuring the uniqueness of minimizer of the expectation of the error functional, and by guaranteeing convergence of estimators through control of the error of the estimator on every compact convex hypothesis space $\hypspace$ in  $L^2(\rhoL)$. 
To see this, apply the coercivity inequality to $\intkernelvar - \intkerneltrue$, to obtain
\begin{align}
c_L \norm{\intkernelvar(\cdot)\cdot  - \intkerneltrue(\cdot)\cdot }^2_{L^2(\R_{+},\rhoL)} \leq  \mE_{L,\infty}(\intkernelvar).
\end{align}
From the facts that $\mE_{L,\infty}(\intkernelvar) \geq 0$ for any $\intkernelvar$ and that $\mE_{L,\infty}(\intkernel) =0$, we can conclude that the true kernel $\intkernel$ is the unique minimizer of the  $\mE_{L,\infty}(\intkernelvar)$. Furthermore, the coercivity condition enables us to control the error of the estimator, on every compact convex hypothesis space in $L^2(\rhoL)$, by the discrepancy of the error functional (see Proposition \ref{convexity}), therefore guaranteeing convergence of the estimator.

\begin{theorem}  \label{thm_Learnability}
Let $\hypspace_n $ be a sequence of compact convex subsets of $L^{\infty}([0,R])$ such that 
$$\quad\inf_{\intkernelvar \in \hypspace_n} \|\intkernelvar (\cdot)\cdot -\intkernel(\cdot)\cdot \|_{L^2(\rhoL)} \to 0$$ as $n\to \infty$.  \revision{Assume that the coercivity condition holds on $\cup_{n=1}^\infty \hypspace_n$.} Then the estimator $\widehat\intkernel_{L,M,\hypspace_n}$ defined in \eqref{e:LSE} converges to the true kernel in $L^2(\rhoL)$ almost surely as $n, M$ approaches infinity, i.e.
 \[ \lim_{n\to\infty} \lim_{M\to \infty} \|\widehat\intkernel_{L,M,\hypspace_n}(\cdot)\cdot  -\intkernel (\cdot)\cdot\|_{L^2(\rhoL)} = 0, \text{ almost surely}.   \] 
\end{theorem}

The above theorem follows from the next proposition. 
\begin{proposition}\label{convexity}
 Let $\mathcal{H}$ be a compact convex subset of $L^{2}(\rhoL)$ and assume the coercivity condition holds true on $\hypspace$.  Then the functional $\mathcal{E}_{L,\infty}$ defined in \eqref{e_errorFtnl_infty2} admits a unique minimizer 
  \begin{equation}\label{e:minimizer}
 \widehat{\phi}_{L,\infty, \mathcal{H}}=\argmin{\intkernelvar \in \mathcal{H}} \mathcal{E}_{L,\infty}(\intkernelvar),
 \end{equation}
in $L^2(\rhoL)$. Furthermore, for all $\intkernelvar \in \mathcal{H}$
 \begin{equation}\label{eq_minH}
 \mathcal{E}_{L,\infty}(\intkernelvar)- \mathcal{E}_{L,\infty}(\widehat{\phi}_{L, \infty, \mathcal{H}}) \geq c_L \|\intkernelvar(\cdot)\cdot-\widehat{\phi}_{L,\infty, \mathcal{H}}(\cdot)\cdot\|_{L^2(\rhoL) }^2.
 \end{equation}
\end{proposition}


\subsection{Optimal rate of convergence of the estimator}
We now turn to the rate of convergence of the estimator.

\begin{theorem}\label{t:Bernstein}
Let the true kernel $\intkerneltrue \in \mathcal{K}_{R,\spaceM}$, and let $\hypspace \subset L^\infty([0,R])$ be compact convex and bounded above by $\spaceM_0\geq \spaceM$.  Assume that the coercivity condition in \eqref{gencoer} holds. Then for any $\epsilon >0$, we have
\begin{align}\label{mainestimate}
& c_L\| \lintkernel_{L,M,\hypspace}(\cdot)\cdot- \intkerneltrue(\cdot)\cdot\|^2_{L^2(\rhoL)}
\leq   2\inf_{\intkernelvar \in \hypspace}\|\intkernelvar(\cdot)\cdot-\intkerneltrue(\cdot)\cdot\|^2_{L^\infty([0,R])} +2\epsilon
\end{align}
with probability at least $1-\delta$, provided that 
$$M \geq \frac{1152\spaceM_0^2R^2}{c_T\epsilon}\big(\log (\mathcal{N}(\hypspace,\frac{\epsilon}{48\spaceM_0R^2} ))+\log(\frac{1}{\delta})  \big)\,,$$
where $\mathcal{N}(\hypspace,\eta)$ is the $\eta$-covering number of $\hypspace$ under the $\infty$-norm.  
\end{theorem} 
We discuss first the implications of this theorem on the choice of hypothesis space in view of obtaining optimal rates of convergence of our estimator. The proof of the theorem will be presented at the end of this section. 
 In practice, given a set of $M$ trajectories, we would like to chose the best finite-dimensional hypothesis space  $\hypspace$ to minimize the error of the estimator. There are two competing issues. On one hand, we would like the hypothesis space $\hypspace$ to be large so that the bias $\inf_{\intkernelvar \in \hypspace}\|\intkernelvar-\intkerneltrue\|^2_{L^\infty([0,R])}$ is small. On the other hand,  we would like to keep $\hypspace$  to be small so that the covering number $\mathcal{N}(\hypspace,{\epsilon}/{48\spaceM_0R^2})$, and therefore the variance of the estimator is small. This is the classical bias-variance trade-off in statistical estimation. Inspired from approximation methods in regression \cite{cucker2002mathematical,binev2005universal,devore2006approximation} ,  the following proposition quantifies the effect of hypothesis spaces on the rate of convergence of the estimator.   

\begin{proposition} \label{prop_optRate}
Assume that  the coercivity condition holds with a constant $c_L$, and recall $\smash{\widehat\intkernel_{L,M,\hypspace}}$ defined in \eqref{e:LSE} is a minimizer of the empirical error functional over a hypothesis space $\hypspace$. \\
(a) For $\hypspace=\mathcal{K}_{R, \spaceM}$, there exists a constant $C=C(\spaceM, R)$ such that 
 \[ \mathbb{E}[\| \widehat\intkernel_{L,M,\hypspace} (\cdot)\cdot-\intkerneltrue(\cdot)\cdot\|_{L^2(\rhoL)} ]\leq \frac{C}{c_L} M^{-\frac{1}{4}}.\]
(b) Assume that $\hypspace_n $ is a sequence of finite dimensional  spaces  of $L^\infty([0,R])$ such that  $dim(\hypspace_n) \leq c_0 n$ and 
 \begin{align}\label{assumptions}
\inf_{\intkernelvar \in \hypspace_n}\|\intkernelvar(\cdot)-\intkerneltrue(\cdot)\|^2_{L^\infty([0,R])}  \leq c_1 n^{-s}
\end{align}
for all $n$ for some constants $c_0, c_1, s>0$, then by choosing  $\smash{n=n_*:=({M}/{\log M})^{\frac{1}{2s+1}}}$,  we have 
\begin{align*}
 \mathbb{E}[\| \widehat\intkernel_{L,M,\hypspace_{n_*}}(\cdot)\cdot-\intkerneltrue(\cdot)\cdot\|_{L^2(\rhoL)} ] \leq \frac{C}{c_L} \left(\frac{\log M}{M}\right)^{\frac{s}{2s+1}}, 
 \end{align*} 
 where  $C=C(c_0,c_1, R,S)$.
\end{proposition}


 It is interesting to compare this rate with those in the mean field regime, where the regime $N\rightarrow\infty$ (with $M=1$, $L\rightarrow\infty$) was studied: the rates in the previous work are not very precise, but at any rate they are no better than $N^{-1/d}$, i.e. they are cursed by the dimension, even if the problem is fundamentally that of estimating a $1$-dimensional function. It would be interesting to understand whether that rate is optimal for this problem in the mean-field regime ($N\rightarrow\infty$), or if in fact, the results in the present work lead to sharper, dimension-independent bounds in the mean-field limit as well.

\bigskip
The proof of Thm.~\ref{t:Bernstein} is based on this technical Proposition:
\begin{proposition}\label{sampleerror}
Assume the coercivity condition holds true and let $\mathcal{H} \subset L^\infty([0,R])$ be compact convex, bounded above by 
$S_0$. Let 
\begin{align*}
\mathcal{D}_{L,\infty, \hypspace}(\intkernelvar):=\mathcal{E}_{L,\infty}(\intkernelvar)-\mathcal{E}_{L,\infty}(\widehat\intkernel_{L,\infty, \hypspace}) \quad,\quad
\mathcal{D}_{L,M, \hypspace}(\intkernelvar):=\mathcal{E}_{L,M}(\intkernelvar)-\mathcal{E}_{L,M}(\widehat\intkernel_{L,\infty,\hypspace}),
\end{align*}where $\widehat\intkernel_{L,\infty, \hypspace}$ is the minimizer of $ \mathcal{E}_{L,\infty}(\cdot)$ over $\hypspace$. Then for all $\epsilon>0$ and $0<\alpha<1$, we have 
\begin{align*} 
& P\bigg\{\sup_{ \intkernelvar \in \mathcal{H}} \frac{\mathcal{D}_{L,\infty, \hypspace}(\intkernelvar)-\mathcal{D}_{L,M, \hypspace}(\intkernelvar)}{\mathcal{D}_{L,\infty, \hypspace}(\intkernelvar)+\epsilon}  \geq 3\alpha \bigg\} \leq \mathcal{N}\left(\hypspace, C_1\alpha \epsilon\right) e^{-C_2\alpha^2M\epsilon}
\end{align*}
where $C_1 =\frac{1}{8 \spaceM_0 R^2}$ and $C_2=\frac{-c_L }{32\spaceM_0^2 R^2}$. 
\end{proposition}

\begin{proof} [\textbf{Proof of the Theorem \rm{\ref{t:Bernstein}} }]
Put $\alpha=\frac{1}{6}$ in Proposition \ref{sampleerror}. 
We know that,  with probability at least
 $$1-\mathcal{N}\left(\hypspace, \frac{\epsilon}{48\spaceM_0R^2}\right) e^{-\frac{c_LM\epsilon}{1152\spaceM_0^2R^2} }, $$ we have 
  $$\sup_{ \intkernelvar \in \hypspace}  \frac{\mathcal{D}_{L,\infty, \hypspace}(\intkernelvar)-\mathcal{D}_{L,M, \hypspace}(\intkernelvar)}{\mathcal{D}_{L,\infty, \hypspace}(\intkernelvar)+\epsilon}  <\frac{1}{2},$$
  and therefore, for all $\intkernelvar \in \hypspace$, $$\frac{1}{2}\mathcal{D}_{L,\infty, \hypspace}(\intkernelvar)<\mathcal{D}_{L,M, \hypspace}(\intkernelvar)+\frac{1}{2}\epsilon. $$
 Taking $\intkernelvar=\widehat \intkernel_{L,M,\hypspace}$, we have
 $$  \mathcal{D}_{L,\infty, \hypspace}(\widehat \intkernel_{L,M,\hypspace} )< 2\mathcal{D}_{L,M, \hypspace}(\widehat \intkernel_{L,M,\hypspace})+\epsilon\,.$$
But $\mathcal{D}_{L,M, \hypspace}(\widehat \intkernel_{L,M,\hypspace})=\mathcal{E}_{L,M}(\widehat \intkernel_{L,M,\hypspace})-\mathcal{E}_{L,M}(\widehat\intkernel_{L,\infty, \hypspace}) \leq 0$ and hence by Proposition \ref{convexity} we have 
 $$c_L \|\widehat \intkernel_{L,M,\hypspace}(\cdot)\cdot-\widehat\intkernel_{L,\infty, \hypspace}(\cdot)\cdot\|_{L^2(\rhoL)}^2 \leq \mathcal{D}_{L,\infty, \hypspace}(\widehat \intkernel_{L,M,\hypspace} )<\epsilon.$$ Therefore, 
\begin{align*}
 \|\widehat \intkernel_{L,M,\hypspace}(\cdot)\cdot-\intkernel(\cdot)\cdot\|_{L^2(\rhoL)}^2 &\leq 2 \|\widehat \intkernel_{L,M,\hypspace}(\cdot)\cdot-\widehat\intkernel_{L,\infty, \hypspace}(\cdot)\cdot\|_{L^2(\rhoL)}^2
+2\|\widehat\intkernel_{L,\infty, \hypspace}(\cdot)\cdot-\intkernel(\cdot)\cdot\|_{L^2(\rhoL)}^2 \\
&\leq  \frac{2}{c_L}( \epsilon +\inf_{\intkernelvar \in \hypspace} \| \intkernelvar(\cdot) \cdot-\intkernel(\cdot)\cdot\|_{\infty}^2),
\end{align*}  
where the last inequality follows from the coercivity condition and by the definition of $\widehat\intkernel_{L,\infty,\hypspace}$(see \eqref{e:minimizer}). 
Given $0<\delta<1$, we see we need $M$ large enough so that $$1-\mathcal{N}(\hypspace, \frac{\epsilon}{48\spaceM_0 R^2}) e^{-\frac{c_LM\epsilon}{1152\spaceM_0^2R^2} } \geq 1-\delta\,.$$ The conclusion follows.
\end{proof}

\subsection{Trajectory-based Performance Measures}
After having established results on the convergence rate of our estimator, we turn to control the accuracy of trajectories predicted when using the estimated interaction kernel, evolved from initial conditions both in and outside of the training data. 
Trajectory-based measurements of accuracy are interesting because (a) they provide a quantitative assessment on the quality of the approximated dynamics, (b) while the true interactions kernels are typically not known, and so the accuracy of the estimated interaction kernel may not be evaluated, trajectories are known, and may be used to perform model validation and cross-validation for parameter selection (if needed).

The next Proposition shows that the error in prediction is (i) bounded trajectory-wise by a continuous time version of the error functional, and (ii) bounded in the mean squared sense by the mean squared error of the estimated interaction kernel. 
\begin{proposition}\label{Trajdiff}
 Let $\widehat\intkernel$ be an estimator of the true interaction kernel $\intkernel$. Suppose that the function $\smash{\widehat\intkernel(||\cdot||) \cdot}$ is Lipschitz continuous on $\R ^d$,  with Lipschitz constant $C_{\rm{Lip}}$. Denote by $\smash{\widehat{\bX}(t)}$ and $\bX(t)$ the solutions of the systems with interaction kernels $\smash{\widehat\intkernel}$ and $\intkerneltrue$ respectively, starting from the same initial condition. Then we have  
\begin{align*}
\sup_{t\in[0,T]}\!\! \|\widehat{\bX}(t)- \bX(t)\|^2
\leq 2Te^{8T^2C^2_{\rm{Lip}}}  \int_0^T\!\!\!\left\|\dot\bX(t)-\rhsfo_{\hat{\intkernelvar}}(\bX(t))\right\|^2\!\!\!dt 
\end{align*}
for each trajectory, and on average with respect to the initial distribution $\probIC$, 
\[
\E_{\probIC}[\sup_{t\in[0,T]} \|\widehat{\bX}(t)- \bX(t)\|^2] \leq  C(T, C_{\rm{Lip}}) \sqrt{N} \|\widehat{\intkernelvar}(\cdot)\cdot-\intkerneltrue(\cdot)\cdot\|_{L^2(\rhoT)}^2
\]
for a constant  $C(T,C_{\rm{Lip}} )$, where the measure $\rhoT$ is as in \ifarXiv Eq. \eqref{e:rhoT} in the main text.\fi
\end{proposition}

\section{Algorithm}\label{s:SI_algorithm}
We start from describing the algorithm in its simplest form, for learning first order system with homogeneous agents; we then move to first order systems with heterogeneous agents, and finish with the second order systems with heterogeneous agents.
\subsection{First Order Homogeneous Agent Systems}
Recall that we would like to estimate the interaction kernel $\intkernel$ of the $N$-agent system in \ifarXiv Eq. \eqref{e:firstordersystemsimple}\fi from $M$ independent trajectories $\{\bxm_i(t_l),\dotbxm_i(t_l)\}_{i = 1, l = 1, m = 1}^{N, L,M}$ with  $t_l = \frac{lT}{L}$. We obtain an estimator by minimizing the discrete empirical error functional
\begin{equation}
\mathcal{E}_{L,M}(\intkernelvar) =\frac{1}{LMN} \sum_{l, m , i = 1}^{L, M, N}\norm{\dotbxm_i(t_l) - \sum_{i' =1}^N\frac{1}{N}\intkernelvar(r_{i, i'}^{m}(t_l))\br_{i, i'}^{m}(t_l)}^2\,,
\label{e:deef}
\end{equation}
over all $\intkernelvar$ in a hypothesis space $\hypspace_n$.  

When only the positions can be observed, we assume that $T/L$ is sufficiently small so that we can accurately approximate the velocity $\dotbxm_i(t_l)$ by backward differences:
\[
\dotbxm_i(t_l) \approx \Delta{\bx}_i^{m}(t_l) = \frac{\bxm_i(t_l) - \bxm_i(t_{l -1})}{t_{l} - t_{l - 1}}, \quad \text{for $1 \le l \le L$},
\]
where we assumed $t_0$ is also observed. The error of the backward difference approximation is of order $O(T/L)$, leading to a $O(T/L)$ bias in the estimator.  Therefore, for simplicity, we assume in the theoretical discussion that follows that the velocity $\dotbxm_i(t_l)$ is observed.

First, we set the hypothesis space \revision{$\hypspace_n$ to be the span of $\{\psi_p\}_{p=1}^n$}, a a set of linearly independent functions on $[0,R]$. \revision{It is natural to} use an orthonormal basis of  $\hypspace_n$ in $L^2(\rho^T_L)$ for efficient computations.  If the true interaction kernel is known to be smooth, a global basis (e.g. Fourier) may be used to  achieve fast convergence. Since our admissible set is in $W^{1,\infty}$, we shall use a local basis consisting of piecewise polynomial functions \revision{on a partition of increasingly finer intervals}. \revision{The partitions} will be on the interval $[R_{min}, R_{max}]$, where $R_{min}$ and $R_{max}$ are minimal and maximal values of $r$ such that the empirical density $\rho_{L,M}^T(r)$ of the pairwise distances $\{r_{i, i'}^{m}(t_l)\}$ is greater than a threshold.    

Next, we minimize the empirical error functional over $\hypspace_n$ to obtain an estimator. To simplify notation, for each $m$, we denote
 \begin{equation}
\mathbf{d}^{m} := \begin{pmatrix} \dotbxm_{1}(t_2), \dots, \dotbxm_N(t_2); \dots ;\dotbxm_{1}(t_L) \dots \dotbxm_N(t_L) \end{pmatrix}\,
\label{e:vecd}
\end{equation}
a column vector in $\R^{LNd}$; and denote 
\[
\Psi_L^{m}(li,p) := \sum_{i' = 1}^N\frac{1}{N}\psi_{p}(r_{i, i'}^{m}(t_l))\br_{i, i'}^{m}(t_l) \in \R^{d}\,,
\]
for $2 \le l \le L$, $1 \le i \le N$ and $1 \le p \le n$, and refer it as the learning matrix $\Psi_L^{m}$. 
Then we can rewrite the empirical error functional as 
\[
\mathcal{E}_{L,M}(\intkernelvar) =\mathcal{E}_{L,M}(\mathbf{a}) = \frac{1}{M} \sum_{m=1}^{M} \norm{\mathbf{d}^m - \Psi_{L}^m\mathbf{a}}^2_{\R^{LNd}}\,.
\]
Our estimator is the minimizer of $\mathcal{E}_{L,M}(\mathbf{a}) $ over $ \R^{n}$. This is a Least Squares problem, and we solve for the minimizer from the normal equations
\begin{equation}
\underbrace{\frac{1}{M} \sum_{m=1}^{M}  A_L^m }_{A_{L,M}}\mathbf{a} =\frac{1}{M} \sum_{m=1}^{M}  b_L^m \,,
\label{e:ALM}
\end{equation}
where the trajectory-wise regression matrices are 
\[A_L^m := \frac{1}{LN}(\Psi^m_L)^T\Psi^m_L,\quad  b_L^m := \frac{1}{LN}(\Psi^m_L)^T\mathbf{d}^m.\]

We emphasize that the above regression is ready to be computed in parallel: we can compute simultaneously the matrices $A_L^m$ and $b_L^m$ for different trajectories.  The size of the matrices $A_L^m$ is $n\times n$, and there is no need to read and store all the data at once, thereby dramatically reducing memory usage. 

\subsection{Well-conditioning from coercivity} \label{sec_condN}
We show next that the coercivity condition implies that $A_{L,M}$ is well-conditioned and positive definite for large $M$. More specifically, the coercivity constant provides a lower bound on the smallest singular value of $A_{L,M}$, provided the basis for the hypothesis space is well-conditioned (e.g. orthonormal), therefore enabling control of the condition number of the regularized problem.

Recall the bilinear functional $\dbinnerp {\cdot, \cdot}$  defined in \eqref{eq:bilinearFn}. 

\begin{proposition} \revision{Assume that the coercivity condition holds on $\hypspace_n\subset L^\infty([0,R])$ with $c_L>0$. }
Let $\{\psi_1,\cdots, \psi_n\}$ be a basis of $\hypspace_n$ such that \begin{equation}\label{onb}\langle \psi_p(\cdot)\cdot,\psi_{p'}(\cdot)\cdot \rangle_{L^2{(\rhoL)}}=\delta_{p,p'}, \|\psi_p\|_{\infty} \leq S_0\end{equation} and $A_{L,\infty}=\big(\dbinnerp{\psi_{p},\psi_{p'}}\big)_{p, p'} \in \mathbb{R}^{n \times n}$. Then the smallest singular value of $A_{L,\infty}$ satisfies
\begin{align*}
\sigma_{\min}(A_{L,\infty}) \geq c_L\,. 
\end{align*}
Moreover, $A_{L,\infty}$ is the a.s. limit of $A_{L,M}$ in \eqref{e:ALM}. Therefore, for large $M$, the smallest singular value of $A_{L,M}$  
\begin{align*}
\sigma_{\min}(A_{L,M}) \geq  0.9 c_L
\end{align*} 
with  probability at least $1-2n\exp(-\frac{c_L^2M}{200n^2c_1^2+\frac{10c_Lc_1}{3}n} )$ with $c_1=R^2S_0^2+1$.
\label{p:sigmaminAL}
\end{proposition}

\begin{proof}
For each $\mathbf{a}\in \R^n$,
 \begin{align*}
\mathbf{a}^T A_{L,\infty} \mathbf{a}&=\dbinnerp {\sum_{p=1}^{n} a_{p}\psi_{p} , \sum_{p=1}^{n} a_{p} \psi_{p}}\geq c_L \big\| \sum_{p=1}^{n} a_{p} \psi_{p}(\cdot)\cdot\big\|_{L^2(\rhoL)}^2 = c_L\|\mathbf{a} \|^2\,.
\end{align*} 
This proves the desired bound on the smallest singular value.

Going back to the case of finite $M$: by the law of large numbers, the matrix $A_{L,M}= \sum_{m=1}^{M}  A_L^m$ converges to 
$A_{L,\infty} = \E[A_L^m]$ as $M \rightarrow \infty$. Hence if the sample size $M$ is large enough, then we apply the matrix Bernstein inequality to get the probability estimates for the event that $\sigma_{min}(A_{L,M})$ is  bounded below by $0.9 c_L$. 
\end{proof}

\begin{remark} Proposition \ref{p:sigmaminAL} highlights the importance of choosing basis functions to be linearly independent in $L^2(\rhoL)$ instead of in $L^\infty([0,R])$ for the hypothesis space $\hypspace_n$ (orthonormality can be easily obtained through Gram-Schmidt orthogonalization if the functions are linearly independent).
To see this, consider a set of basis functions consisting of piecewise polynomials that are supported on a partition of the interval $[0,R]$. These functions are linearly independent in $L^\infty([0,R])$, but can be linearly dependent in $L^2(\rhoL)$ if some of the partitioned intervals have zero probability under the measure $\rhoL$. This would lead to an ill-conditioned normal matrix $A_{L,\infty}$. 
 This issue can deteriorate in practice when the unknown $\smash{\rhoL}$ is replaced by the empirical measure $\smash{\rho_{T}^{L,M}}$. In this work we use piecewise polynomials on a partition of the support of $\smash{\rho_T^{L,M}}$, which are orthogonal in $\smash{L^2(\rho_T^{L,M})}$. 
\end{remark}

\subsection{First Order Heterogeneous Agent Systems}
For these systems the empirical error to be minimized is as in (9) in the main text:
\[
\sum_{l = 2, m = 1, i = 1}^{L, M, N}\frac{1}{LMN_{\clof_i}}\norm{\dotbxm_i(t_l) - \sum_{i' =1}^N\frac{1}{N_{\clof_{i'}}}\intkernelvar_{\clof_i\clof_{i'}}(r_{i, i'}^{m}(t_l))\br_{i, i'}^{m}(t_l)}^2,
\]
over all possible $\bintkernelvar = \{\intkernelvar_{\idxcl\idxcl'}\}_{\idxcl,\idxcl'=1}^{\numcl} \in \hypspace$. Here $\br_{i,i'}(t_l)$ and $r_{i,i'}(t_l)$ are as in \eqref{e:deef}. When given observation data, $\{\bxm_i(t_l)\}_{i = 1, m = 1, l = 1}^{N, M, L}$, but no derivative information, we approximate the derivatives using backward differencing scheme for $2 \le l \le L$; in either case we assemble the derivative vector $\mathbf{d}$ similarly to \eqref{e:vecd}, but with the normalization 
\[
\mathbf{d}^{m}(li) = (1/N_{\clof_i})^{1/2}\Delta{\bx}_i^{m}(t_l) \in \R^d. \]
  
Proceeding analogously to the homogeneous agent case, we search for $\intkernelvar_{\idxcl\idxcl'}$ in a $n_{\idxcl\idxcl'}$-dimensional hypothesis space $\hypspace_{n_{\idxcl\idxcl'}}$, with basis $\{\psi_{\idxcl\idxcl',p}\}_{p=1}^{n_{\idxcl\idxcl'}}$, and write $\intkernelvar_{\idxcl\idxcl'}(r) = \sum_{\idxcl,\idxcl' = 1}^{\numcl}\sum_{p = 1}^{n_{\idxcl\idxcl'}}a_{\idxcl\idxcl',p}\psi_{\idxcl\idxcl', p}(r)$ for some vector of coefficients $(a_{\idxcl\idxcl',p})_{p=1}^{n_{\idxcl\idxcl'}}$.  
For the learning matrix $\Psi_L^{m}$, we will divide the columns into $\numcl^2$ regions, each region indexed by the pair $(\idxcl,\idxcl')$, with $\idxcl,\idxcl' = 1, \cdots, \numcl$.  
We adopt the usual lexicographic partial ordering on these pairs.
The columns of $\Psi_L^{m}$ corresponding to $(\idxcl,\idxcl')$ are given by
\[
\Psi_L^{m}(li, \tilde{n}_{kk'} + p) =  \sqrt{\frac{1}{N_{\clof_i}}}\sum_{i' \in \cl_{\idxcl'}}\frac1{N_{\idxcl'}}\psi_{\idxcl\idxcl', p}(r_{i, i'}^{m}(t_l))\br_{i, i'}^{m}(t_l) \in \R^{d},
\]
for $i \in \cl_{\idxcl}$ and $2 \le l \le L$, and $\tilde{n}_{kk'} = \sum_{(\idxcl_1,\idxcl'_1) < (\idxcl,\idxcl')}n_{\idxcl_1\idxcl'_1}$.  We define
\[
\mathbf{a} = \begin{pmatrix} a_{11, 1}, \dots, a_{11, n_{11}}; \dots ; a_{\numcl\numcl, 1}, \dots, a_{\numcl\numcl, n_{\numcl\numcl}} \end{pmatrix}\in\R^{d_0}\,
\]
with $d_0={\scriptscriptstyle{\sum_{\idxcl, \idxcl' = 1}^\numcl n_{\idxcl, \idxcl'}}}$,  
to arrive at \eqref{e:ALM}

\subsection{Second Order Heterogeneous Agent Systems}
The learning problems of inferring the interactions of the $\dot\bx_i$'s and $\xi_i$'s can be de-coupled.  
We start with the inference of the interactions on $\dot\bx_i$'s. Let the observations of the second order heterogeneous agent system be $\{\bx_i^m(t_l), \dot\bx_i^m(t_l), \xi_i^{m}(t_l)\}_{l, i, m = 1}^{L, N, M}$. Let $\bv_i^{m} = \dot\bx_i^{m}$.
As usual, if velocities and/or accelerations are not observed, they are approximated by a finite-difference (in time) scheme, for example
\[
\Delta{\bv}_i^{m}(t_l) = \frac{\bv_i^{m}(t_l) - \bv_i^{m}(t_{l - 1})}{t_l - t_{l - 1}} \quad, \quad \Delta{\xi}_i^{m}(t_l) = \frac{\xi_i^{m}(t_l) - \xi_i^{m}(t_{l - 1})}{t_l - t_{l - 1}},
\]
for $2 \le l \le L$ and $1 \le i \le N$.  
For the data corresponding to the $m^{th}$ initial condition, we assemble the external influence (from interaction with the environment) vector $\vec{F}^{m, \bv}$ as:
\[
\vec{F}^{m, \bv}(li) = (1/N_{\clof_i})^{1/2}\forcev(\bv_i^{m}(t_l), \xi_i^{m}(t_l)) \in \R^{d},
\]
 and the approximated derivative of $\bv_i$'s as 
\[
\mathbf{d}^{m, \bv}(li) = (1/N_{\clof_i})^{1/2}m_i\Delta{\bv}_i^{m}(t_l) \in \R^{d}.
\]  
We use a finite dimensional subspace $\hypspace^E_{n^E}$, so that the candidate functions $\bintkernelvar^E = \{\intkernelvar_{\idxcl\idxcl'}^E\}_{\idxcl,\idxcl' = 1}^{\numcl}$ are expressed as $\bintkernelvar^E(r) = \sum_{\idxcl,\idxcl'=1}^{\numcl}\sum_{p = 1}^{n^E_{\idxcl,\idxcl'}} \alpha^E_{\idxcl\idxcl', p}\psi_{\idxcl\idxcl', p}^E(r)$.  
Using the same ordering from previous discussion on the first order heterogeneous agent system, we have, for a pair $(\idxcl,\idxcl')$ learning matrix $\Psi_{L,M}^{m, E}$ for the energy-based interaction kernel,
\[
\Psi_{L,M}^{m, E}(li, \tilde{n}^E + p) = N_{\clof_i}^{-1/2}\sum_{i' \in \cl_{\idxcl'}}\frac{1}{N_{\idxcl'}}\psi_{\idxcl\idxcl', p}^E(r_{i, i'}^{m}(t_l))\br_{i, i'}^{m}(t_l),
\]
for $2 \le l \le L$, $i \in \cl_{\idxcl}$ and $\tilde{n}^E = \sum_{(\idxcl_1,\idxcl'_1) < (\idxcl,\idxcl')} n^E_{\idxcl_1\idxcl'_1}$.  The construction of the alignment-based learning matrix $\Psi_{L,M}^{m, A}$ is analogous:
\[
\Psi_{L,M}^{m, A}(li, \tilde{n}^A + p) = N_{\clof_i}^{-1/2}\sum_{i' \in \cl_{\idxcl'}}\frac{1}{N_{\idxcl'}}\psi_{\idxcl\idxcl', p}^A(r_{i, i'}^{m}(t_l))\br_{i, i'}^{m}(t_l),
\]
for $2 \le l \le L$, $i \in \cl_{\idxcl}$ and $\tilde{n}^A = \sum_{(\idxcl_1,\idxcl'_1) < (\idxcl,\idxcl')} n^A_{\idxcl_1\idxcl'_1}$.  
We put all the $\alpha$'s together into $\mathbf{a}^E$ and $\mathbf{a}^A$, and further grouping them into one big vector, $\mathbf{a}^{\bv} = \begin{pmatrix} \mathbf{a}^E \\ \mathbf{a}^A \end{pmatrix}$ and $\Psi_{L,M}^{m,\bv} = \begin{pmatrix} \Psi_{L,M}^{m,E}, \Psi_{L,M}^{m,A} \end{pmatrix}$, we arrive at the final formulation,
\[
\frac{1}{M}\sum_{m=1}^M\norm{\mathbf{d}^{m,\bv} - \vec{F}^{m,\bv} - \Psi_{L,M}^{m,\bv}\mathbf{a}^{\bv}}_{\R^{LNd}}^2.
\]
As usual, we solve the associated normal equations of \eqref{e:ALM} with $A_L^m:= \ptrans{\Psi_{L,M}^{m,\bv}}\Psi_{L,M}^{m,\bv}$ and $b_L^m := \ptrans{\Psi_{L,M}^{m,\bv}}(\mathbf{d}^{m,\bv} - \vec{F}^{m,\bv})$, 
reducing the system size from $(MLNd)\times (n^E + n^A)$ to $(n^E + n^A)^2$. 

For the inference of the interactions on $\xi_i$'s, we let
\begin{equation*}
\vec{F}^{m, \xi}(li) = N_{\idxcl}^{-\frac{1}{2}} \forcexi(\xi_i^{m}(t_l)) \quad \text{and} \quad \mathbf{d}^{m, \xi}(li) =  N_{\idxcl}^{-\frac{1}{2}}\Delta{\xi}_i^{m}(t_l),
\end{equation*}
for $2 \le l \le L$ and $1 \le i \le N$;  then the learning matrix $\Psi_{L,M}^{m, \xi}$ is assembled similarly as
\begin{equation*}
\Psi_{L,M}^{m, \xi}(li, \tilde{n}^{\xi} + p) =  N_{\idxcl}^{-\frac{1}{2}} \sum_{i' \in \cl_{\idxcl'}}\frac{1}{N_{\idxcl'}}\psi_{\idxcl\idxcl', p}^{\xi}(r_{i, i'}^{m}(t_l))\br_{i, i'}^{m}(t_l),
\end{equation*}
for $2 \le l \le L$, $i \in \cl_{\idxcl}$, and $\tilde{n}^{\xi} = \sum_{(\idxcl_1\idxcl'_1) < (\idxcl,\idxcl')} n^{\xi}_{\idxcl_1,\idxcl'_1}$.  
We then arrive at the Least Squares problem
\begin{equation*}
\frac{1}{M}\sum_{m=1}^M \norm{\mathbf{d}^{m,\xi} - \vec{F}^{m,\xi} - \Psi_{L,M}^{m,\xi}\mathbf{a}^{\xi}}_{\R^{LNd}}^2
\end{equation*}
and solve it from the associated normal equations.

\subsection{The Final Algorithm}
Given observation data, $\{\bx_i^{m}(t_l)$ and $\dot\bx_i^{m}(t_l)$ and/or $\xi_i^{m}(t_l)\}_{l, i, m = 1}^{L, N, M}$, we use the Algorithm \ref{alg:main} to find the estimators for the interaction kernels.
\begin{algorithm}[H]
\mycaption{Learning Interaction Kernels from Observations}\label{alg:main}
\small{
\begin{algorithmic}[1]
\State Input: $\{\bx_i^{m}(t_l)$ and/or $\dot\bx_i^{m}(t_l)$ and/or $\xi_i^{m}(t_l)\}_{l, i, m = 1}^{L, N, M}$.
\State Output: estimators for the interaction kernels.
\If{First Order System}
  \State Find out the maximum interaction radii $R_{\idxcl\idxcl'}$'s.
  \State Construct the basis, $\psi_{\idxcl\idxcl', p}$'s.
  \State Assemble the normal equations (\ref{e:ALM}) (in parallel). 
  \State Solve for $\mathbf{a}$.
  \State Assemble $\blintkernel(r) = \sum_{\idxcl,\idxcl'=1}^{\numcl}\sum_{p = 1}^{n_{\idxcl\idxcl'}}a_{\idxcl\idxcl', p}\psi_{\idxcl\idxcl', p}(r)$.
\ElsIf{Second Order System}
  \State Find out the maximum interaction radii $R_{\idxcl\idxcl'}$'s.
  \State Construct the basis, $\psi_{\idxcl\idxcl', p}^E$'s and $\psi_{\idxcl\idxcl', p}^A$'s.
  \State Assemble 
       the normal equations (\ref{e:ALM}) (in parallel). 
  \State Solve for $\mathbf{a}^{\bv}$, and partition it to $\mathbf{a}^E$ and $\mathbf{a}^A$.
  \State Assemble $\blintkernel(r)^E = \sum_{\idxcl,\idxcl'=1}^{\numcl}\sum_{p = 1}^{n_{\idxcl\idxcl'}^E}a_{\idxcl\idxcl', p}^E\psi_{\idxcl\idxcl', p}^E(r)$.
  \State Assemble $\blintkernel(r)^A = \sum_{\idxcl,\idxcl'=1}^{\numcl}\sum_{p = 1}^{n_{\idxcl\idxcl'}^A}a_{\idxcl\idxcl', p}^A\psi_{\idxcl\idxcl', p}^A(r)$.
    \If{If there are $\xi_i$'s}
      \State Construct the basis, $\psi_{\idxcl\idxcl', p}^{\xi}$'s.
      \State Assemble the normal equations. 
      \State Solve for $\mathbf{a}^{\xi}$.
      \State Assemble $\blintkernel(r)^{\xi} = \sum_{\idxcl,\idxcl'=1}^{\numcl}\sum_{p = 1}^{n_{\idxcl\idxcl'}^{\xi}}a_{\idxcl\idxcl', p}^{\xi}\psi_{\idxcl\idxcl', p}^{\xi}(r)$.
    \EndIf
    \State \textbf{end if}.
\EndIf
\State \textbf{end if}.
\end{algorithmic}}
\end{algorithm}

\subsection{Computational Complexity}
The computational complexity is driven by the construction and solution of the least squares problem in Algorithm \ref{alg:main}. Though the observation data $\{\bx_i^{m}(t_l), \dot\bx_i^{m}(t_l), \xi_i^{m}(t_l)\}_{l, i, m = 1}^{L, N, M}$ requires an array of size $MLN(2d + 1)$, the linear system to be solved, i.e. the normal equations, is only of size $n^E + n^A$. When the normal equations are ill-conditioned or ill-posed, a truncated singular value decomposition will be used, which does a singular value decomposition of the matrix  $A_{L,M}$, and keeps those singular values which are above a (preset) threshold, then assemble an approximated matrix with the truncated singular value matrices.

 Furthermore, since the $M$ trajectories are independent, we can construct $\Psi^{m, E}$ and other related quantities for each trajectory at a time (which can be done in a parallel environment with two communication needed, one to send/receive the maximum interaction radii's, and the other to send/receive  $A_L^m$ and $b_L^m$  in the normal equations after they are built on the master core), each requires a total memory of $LNd(n^E + n^A) + LNd + LNd$, which is $\bigO(LNd)$, since $n^E + n^A \ll LNd$. 
 
 The computing time of the algorithm depends heavily on the time to assemble normal equations from $M$ trajectories; solving the final linear system requires basically little time compared to the assembly time, even using the truncated singular value decomposition solver. 
 
Therefore, the algorithm is effective at inferring the interactions from a wide variety of dynamics, and the results will be discussed in the next section.

\section{Examples}\label{s:SIExamples}
We consider here four important examples of self-organized dynamics: the opinion dynamics, the particle system with the Lennard-Jones potential, the predator-swarm system and the phototaxis dynamics.  We describe here in detail how the numerical simulations are set up for each of these examples.  In all but the Lennard-Jones system, we set up the experiments using the parameters as shown in Table \ref{t:learn_params}. We consider the regime with a rather small number of observations in terms of both $M$ and $L$ to emphasize that our technique can achieve good results even when a relatively small number of samples is given.
\begin{table}[H]\centering
\footnotesize{\begin{tabular}{| c | c | c | c |}
\hline 
 $N$   & $\#$ Trials & $M_{\rhoL}$   & $[0, T_f]$ \\ 
\hline 
 $10$ & $10$     & $2000$       &  $[0, cT]$\\
\hline
\end{tabular}}
\mycaption{\textmd{Parameters used in all the examples but the Lennard-Jones system. Here the observation time $T$ is system-specific. $c=2$ in all examples unless otherwise specified. }} 
\label{t:learn_params}
\end{table}

We use a large number $\smash{M_{\rhoL}}$ (in particular, $\smash{M_{\rhoL}\gg M}$) of independent trajectories (not to be used elsewhere) to obtain an accurate approximation of the unknown probability measure $\smash{\rhoL}$ in (4) in the main text. In what follows, to keep the notation from becoming cumbersome, we denote by $\rhoL$ this empirical approximation to $\rhoL$. We run the dynamics over the time $[0, T_f]$ with $M$ different initial conditions (drawn from the dynamics-specific probability measure $\probIC$), and the observations consist of the state vector, with no derivative information, at $L$ equidistant time samples in the time interval $[0, T]$.  
We report the relative (i.e. normalized by the norm of the true interaction kernel) error of our estimators in the $\smash{L^2(\rhoL)}$ norm. In the spirit of Proposition (3.4) in the main text, we also report on the error on trajectories $\bX(t)$ and $\widehat\bX(t)$ generated by the system with the true interaction kernel and with the learned interaction kernel, on both the “training” time interval $[0,T]$ and on a “prediction” time interval $[T,T_f]$ ($T_f=2T$ unless otherwise specified), with both the same initial conditions as those used for training, and on new initial conditions (sampled according to the specified measure $\probIC$). The trajectory error will be estimated using $M$ trajectories (we report the mean and standard deviation of the error).  We run a total of $10$ independent learning trials and compute the mean and standard deviation of the corresponding estimators, their errors, and the trajectory errors just discussed.  Since each learning trial generates different mean and standard deviation of the trajectory errors over different Initial Conditions (ICs), we also report the mean and standard deviation over the $10$ learning trials for $\text{mean}_{IC}$ and $\text{std}_{IC}$.

All ODE systems are evolved using \textbf{ode$15$s} in MATLAB\textsuperscript{\textregistered} with a relative tolerance at $10^{-5}$ and absolute tolerance at $10^{-6}$.  We choose the finite-dimensional hypothesis space $\hypspace_n$ (with $n$ chosen differently in each example, based on sample size) as the span of either piecewise constant or piecewise linear functions on $n$ intervals forming a uniform partition of $[0,R_{k, k'}]$, where $R_{k, k'}$ is the maximum observed pairwise distance between agents of type $k'$ and agents in type $k$ for $t \in [0, T]$.  

Learning results are showcased in \revision{\ifarXiv Fig.~\ref{fig:example_main}\fi}. The first one compares the learned interaction kernel(s) to the true interaction kernel(s) (with mean and standard deviation over the total number of learning trials) with the background showing the comparison of $\smash{\rhoL}$ (computed on $\smash{M_{\rhoL}}$ trajectories, as described above) and $\smash{\rho_T^{L, M}}$ (generated from the observed data consisting of $M$ trajectories). The second plot compares the true trajectories (evolved using the true interaction law(s)) and learned trajectories (evolved using the learned interaction law(s)) over two different set of initial conditions -- one taken from the training data, and one new, randomly generated from $\mu_0$. The third plot compares the true trajectories and the trajectories generated with the estimated interaction kernel, but for a different system with the number of agents $N_{\text{new}} = 4N$, again over two different sets of randomly chosen initial conditions.  Measurements of performance are also shown alongside the figures: ($L^2(\rhoL)$ errors, trajectory errors, etc.  
Let $\bX(t)$ and $\hat{\bX}(t)$ be two sets of continuous-time trajectories; the max-in-time error is defined as
\begin{equation}\label{e:tm_norm}
\norm{\bX - \hat{\bX}}_{\text{TM([0,T])}} = \sup_{t \in [0, T]}\tnorm{\bX(t) - \hat{\bX}(t)}\,.
\end{equation}
For second order systems with the auxiliary environment variable $\xi_i$'s, we are also interested in the trajectories of $\xi_i$, for which we may use $\norm{\Xi - \hat{\Xi}}_{\text{TM([0,T])}} = \sup_{t \in [0, T]}\norm{\Xi(t) - \hat{\Xi}(t)}_{\mathcal{S}}$.

Finally, for each example we consider adding noise to the observations: in the case of additive noise the observations are $\smash{\{(\bXm(t_l)+\eta_{1,l,m},\dot\bXm(t_l))+\eta_{2,l,m}\}_{l=1,m=1}^{L,M}}$, while in the  case of multiplicative noise they are $\smash{\{(\bXm(t_l)\cdot(1+\eta_{1,l,m}), }$ $\smash{\dot\bXm(t_l))\cdot(1+\eta_{2,l,m})\}_{l=1,m=1}^{L,M}}$, where in both cases $\eta_{1,l,m}$ and $\eta_{2,l,m}$ are i.i.d. samples from a distribution modeling noise, which we will pick to be Unif.$([-\sigma,\sigma])$. Note that in both these cases velocities are part of our observations, since with noise added in the position the inference of velocities becomes problematic due to the amplification of the noise that a simple finite difference scheme would incur. 

Finally, for several examples we also report the behavior of the relative error of the estimator as a function of the number of samples $L$ in time and of the number of trajectories $M$. We observe the decrease in error as $L$ increases, which is expected but is not captured by the estimate in Thm. (3.3) in the main text. These plots are qualitatively the same for all the experiments.

We devote the next sections to the various examples, discussing setups particular to each example and corresponding results.

\subsection{Opinion Dynamics}\label{s:SI_OD}
Modeling using self-organized dynamics has seen successful applications in studying and analyzing how the opinions of people influence each other and how consensus is formed based on different kinds of influence functions.  We refer to these systems as opinion dynamics. We consider the first order model in \eqref{e:1stordersystem}, and the interaction kernel defined as
\[
\intkernel(r) = \left\{
        \begin{array}{ll}
           1,    & \quad 0                          \le r < \frac{1}{\sqrt{2}}, \\
           0.1, & \quad \frac{1}{\sqrt{2}} \le r < 1, \\
           0,    & \quad 1                         \le r.
        \end{array}
    \right.
\]
In this context $\intkernel: \R_+ \rightarrow \R_+$ is sometimes referred to as the scaled influence function, modeling the change of each agents' opinion by relative differences in the opinions of the other agents. Here $\bx_i \in \R^{d}$ is the vector opinions of agent $i$.  Here $\norm{\cdot}$ can be taken as the normal Euclidean norm, but other metrics depending on the problem at hand may be used as well, with no changes in our definitions and constructions.  The time-discretization of this system is referred to as the classical Krause model for opinion dynamics.  With the specific $\intkernel$ above, \revision{there is only attraction present in the system, the opinions of the agents} merge into clusters, with the number of clusters significantly smaller than the number of agents. 
This clustering behavior severely reduces the number of effective samples of pairwise distance observable at large times. We consider the system and test parameters given in Table \ref{t:OD_params}.

\begin{table}[H]\centering
\footnotesize{\begin{tabular}{| c | c | c | c | c | c | c |}
\hline 
 $d$  & $M$ & $L$      & $T$   & $\probIC$          & $n$    & deg($\psi$) \\ 
\hline 
 $1$ & $50$ & $200$ & $10$ & $\mathcal{U}([0, 10]^2)$ & $200$ & $0$\\
\hline
\end{tabular}}
\mycaption{\textmd{(OD) Parameters for the system}}
\label{t:OD_params}
\end{table}

\begin{figure}[H]
\begin{subfigure}{\textwidth}
  \centering
   \includegraphics[width=\ifarXiv 0.48\textwidth \fi]{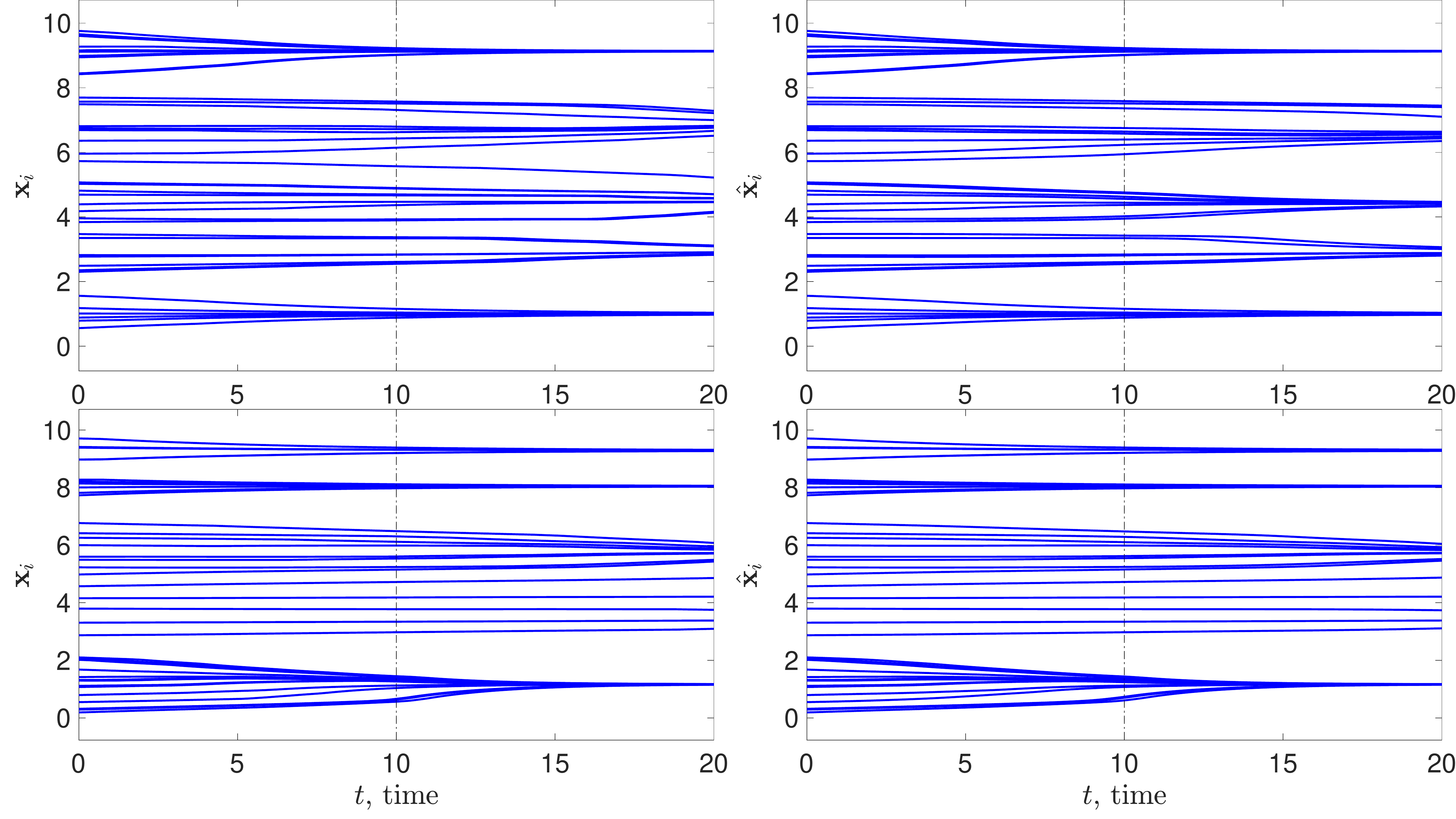}
\end{subfigure}
\mycaption{(OD) Trajectories $\bX(t)$ and $\smash{\widehat{\bX}(t)}$ obtained with $\intkernel$ and $\smash{\lintkernel}$ respectively, for dynamics with larger $N_{\text{new}} = 4N$, over two different set of initial conditions.  We are able to accurately predict the clusters (number and location). Errors are reported in Table \ref{t:ODH1_traj_err}.}
\label{fig:ODH1}
\vskip-4mm
\end{figure}

\begin{table}[H]\centering
\footnotesize{\begin{tabular}{| c || c | c |} 
\hline
                                                             & $[0, T]$                                                    & $[T, T_f]$\\
\hline
$\text{mean}_{\text{IC}}$: Training ICs & $3.5 \cdot10^{-2} \pm 8.1 \cdot10^{-3}$  & $4.8 \cdot10^{-2} \pm 1.4 \cdot10^{-2}$\\
\hline
$\text{std}_{\text{IC}}$: Training ICs    & $5.2 \cdot10^{-2} \pm 1.3 \cdot10^{-2}$  & $7.6 \cdot10^{-2} \pm 2.7 \cdot10^{-2}$\\
\hline            
$\text{mean}_{\text{IC}}$: Random ICs & $3.2 \cdot10^{-2} \pm 7.4 \cdot10^{-3}$  & $4.6 \cdot10^{-2} \pm 1.2 \cdot10^{-2}$\\
\hline
$\text{std}_{\text{IC}}$: Random ICs    & $5.0 \cdot10^{-2} \pm 1.7 \cdot10^{-2}$  & $7.2 \cdot10^{-2} \pm 2.7 \cdot10^{-2}$\\
\hline 
$\text{mean}_{\text{IC}}$: Larger $N$ & $3.1 \cdot10^{-2} \pm 2.0 \cdot10^{-3}$  & $7.3 \cdot10^{-2} \pm 4.1 \cdot10^{-3}$ \\
\hline
$\text{std}_{\text{IC}}$: Larger $N$    & $2.1 \cdot10^{-2} \pm 2.1 \cdot10^{-3}$  & $6.1 \cdot10^{-2} \pm 4.2 \cdot10^{-3}$ \\
\hline         
\end{tabular}}
\mycaption{ \textmd{(OD) Trajectory Errors: ICs used in the training set (first two rows), new IC"s randomly drawn from $\mu_0$ (second set of two rows), for ICs randomly drawn for a system with $4N$ agents (last two rows). Means and std's are over $10$ learning runs.}}
\label{t:ODH1_traj_err}
\end{table}

Fig.~\ref{fig:ODH1} shows the comparison between the estimated interaction kernel $\smash{\lintkernel}$ (as the mean over learning trials) and the true one, $\intkernel$. 
We obtain a faithful approximation of the true interaction kernel, including near the discontinuity and the compact support. Our estimator also performs well near $0$, notwithstanding that information of $\intkernel(0)$ is lost due to the structure of the equations, that have terms of the form $\phi(0)\vec{0} = \vec{0}$.
The same figure also compares the trajectories generated by the system governed by $\intkernel$ and that governed by $\smash{\lintkernel}$.
Table \ref{t:ODH1_traj_err} reports the max-in-time error for those trajectories.
We also test the robustness to noise, by adding noise to the observations of both positions and velocities, as described above: the estimated kernel is shown in Figure \ref{fig:ODHnoise}.
Figure \ref{fig:ODHMLtest} shows the behavior of the error of the estimator as both $L$ and $M$ are increased.
\begin{figure}[H]
\centering
\includegraphics[width=\ifarXiv 0.48\textwidth \fi]{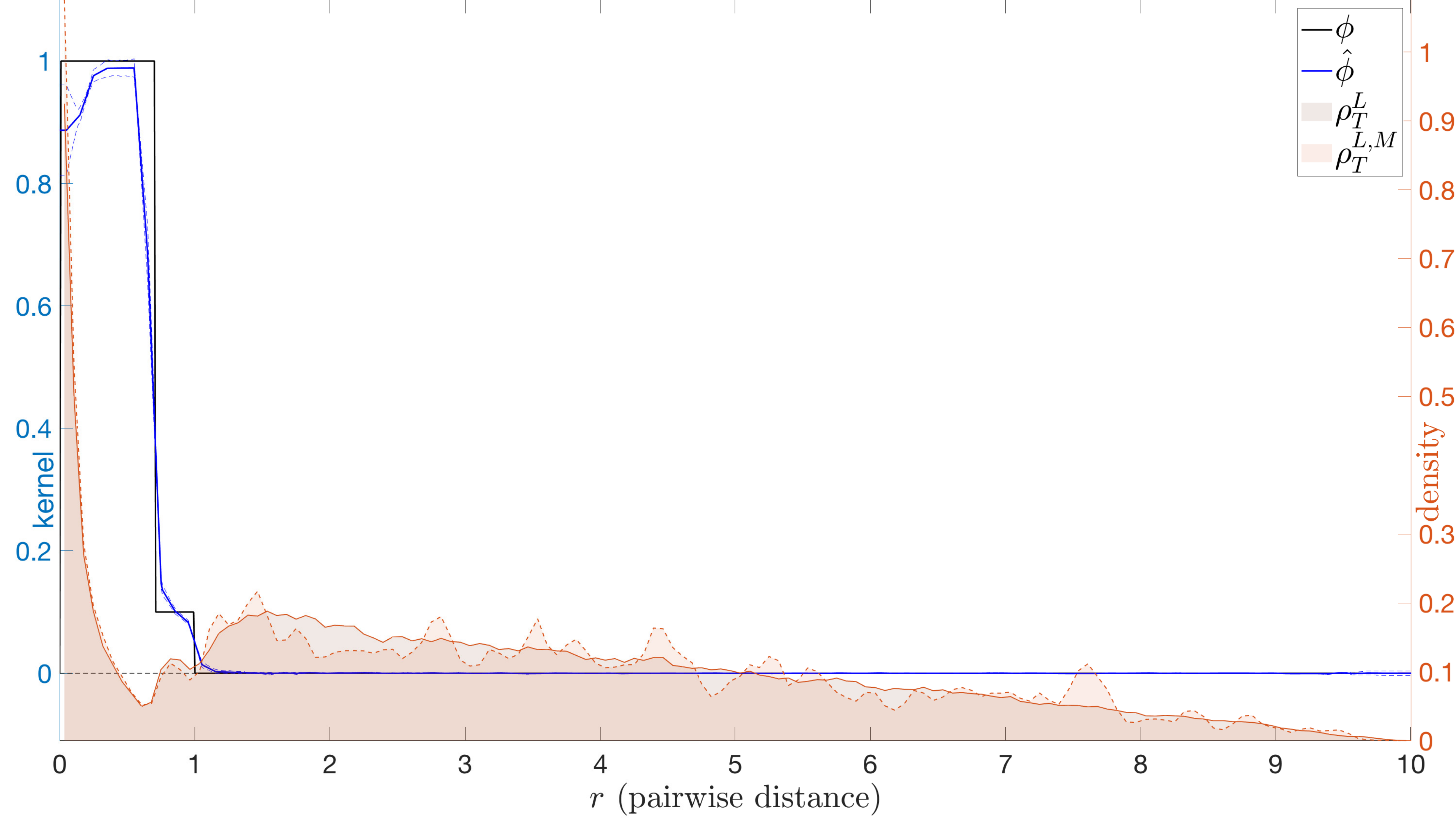}
\mycaption{(OD) Interaction kernel learned with Unif.$([-\sigma,\sigma])$ additive noise, for $\sigma=0.1$ in the observed positions and velocities. The estimated kernels are minimally affected, mostly in regions with small $\rhoL$ and near $0$.}
\label{fig:ODHnoise}
\end{figure}

\begin{figure}[H]
\centering
\includegraphics[width=\ifarXiv 0.48\textwidth \fi]{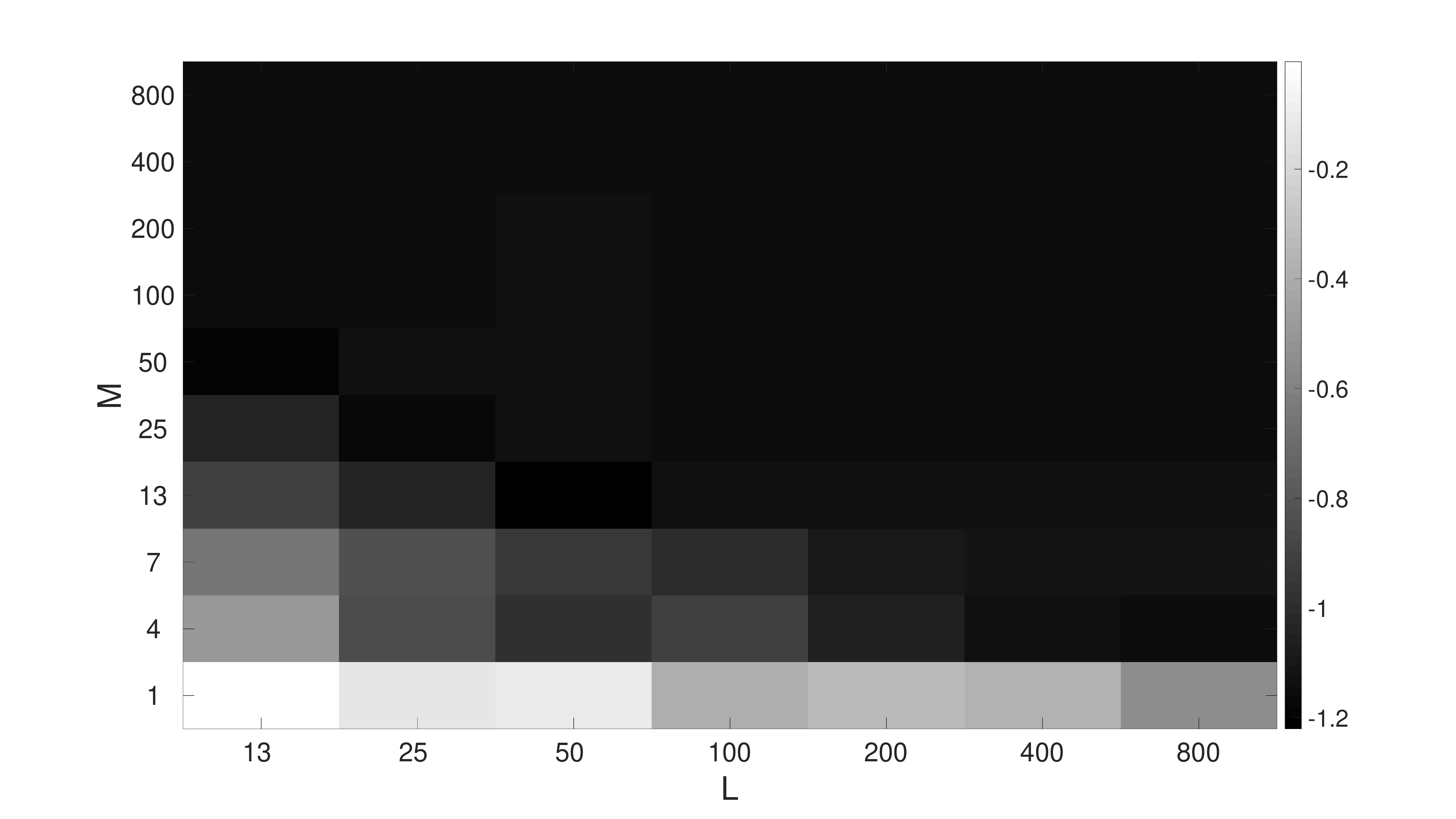}
\mycaption{(OD) Relative error, in $\log_{10}$ scale, of $\hat\intkernel$ as a function of $L$ and $M$. The error decreases both in $L$ and $M$, in fact roughly in the product $ML$, at least when $M$ and $L$ are not too small. $M=1$ does not seem to suffice, no matter how large $L$ is, due to the limited amount of ``information'' contained in a single trajectory.}
\label{fig:ODHMLtest}
\end{figure}

\subsection{Interacting Particles in Lennard-Jones Potential}
\label{LJdescriptions} 
\begin{figure}[H]
\centering
\begin{subfigure}{\ifarXiv 0.48\textwidth \fi}
   \includegraphics[width=1\linewidth]{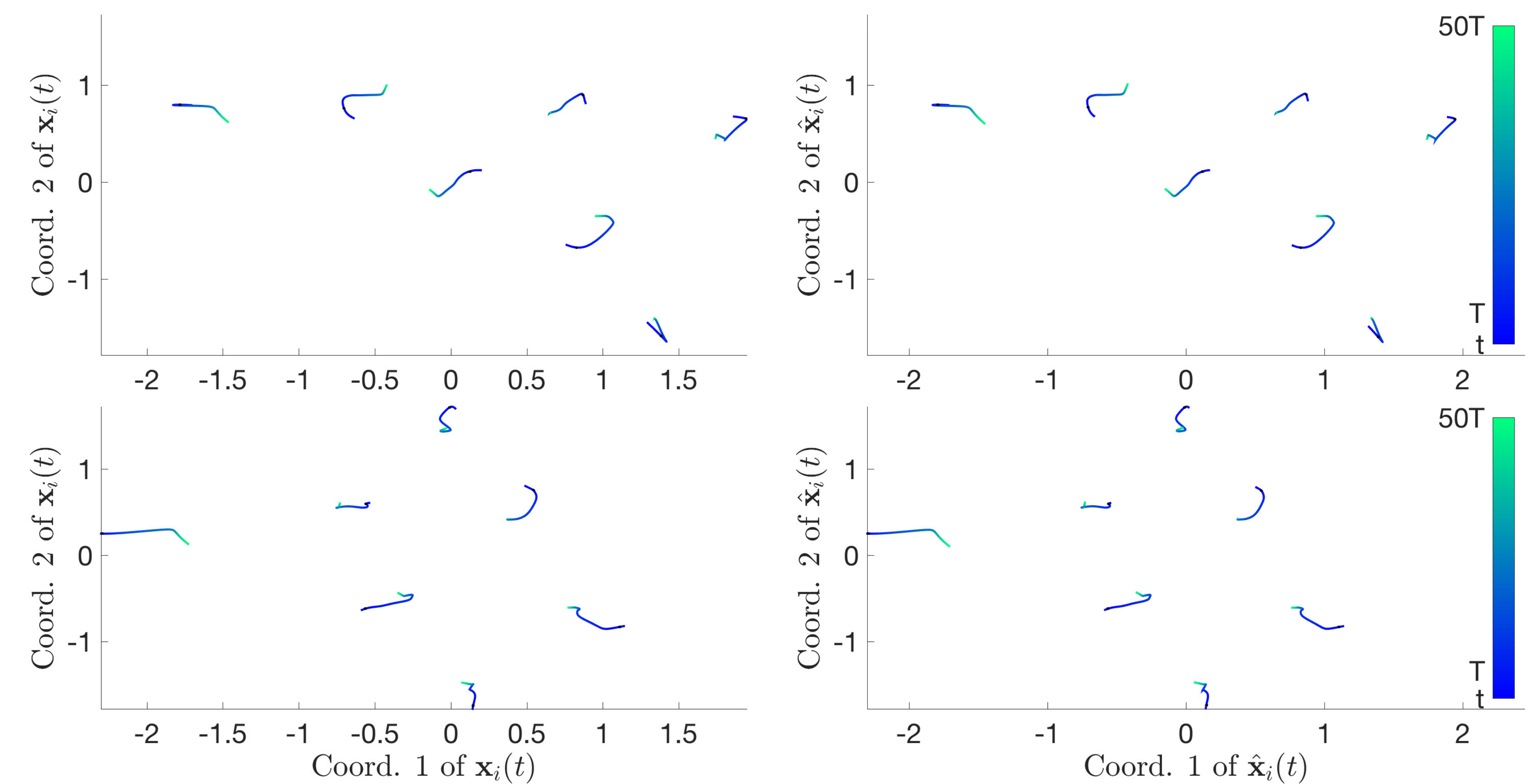}
\subcaption{$N$-particle system, with kernel learned from many short trajectories} \end{subfigure}
\begin{subfigure}{\ifarXiv 0.48\textwidth \fi}
   \includegraphics[width=1\linewidth]{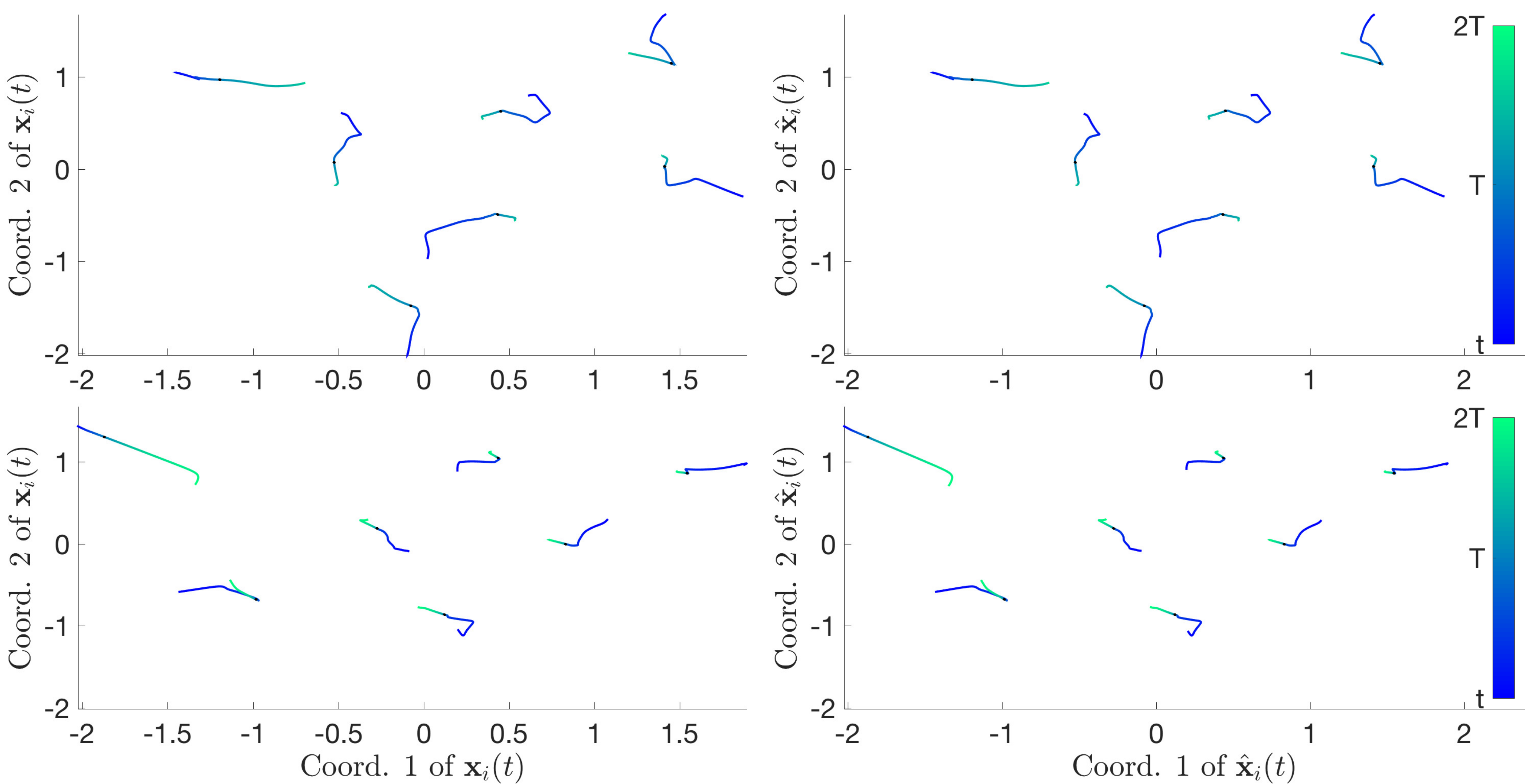}
\subcaption{$N$-particle system, with kernel learned from  a few long trajectories }
\end{subfigure}
\mycaption{(LJ) (a) and (b)presents trajectories $\bX(t)$ (left) and $\smash{\widehat{\bX}(t)}$ (right) obtained with $\intkernel$ and $\smash{\lintkernel}$ respectively, for initial conditions in the training dataset (top) and randomly sampled initial conditions (bottom).  The time $T$ is as in Table \ref{t:LJ_params}.  Trajectory errors for all cases are reported in Table \ref{t:LJ_traj_err_MC}.\label{fig:LJ1traj}}
\end{figure}
The expression of the Lennard-Jones potential is 
\[
\Phi(r) = 4\epsilon \left[\left(\frac{\sigma}{r}\right)^{12} - \left(\frac{\sigma}{r}\right)^6\right]=\epsilon \left[ \left(\frac{r_m}{r}\right)^{12}-2\left(\frac{r_m}{r}\right)^6 \right]
\]
where $\epsilon$ is the depth of the potential well, $\sigma$ is the finite distance at which the inter-particle potential is zero, $r$ is the distance between the particles, and $r_m$ is the distance at which the potential reaches its minimum. At $r_m$, the potential function has the value $-\epsilon$. The  $r^{-12}$ term describes Pauli repulsion at short ranges due to overlapping electron orbitals, and the $r^{-6}$ term describes attraction at long ranges (van der Waals force, or dispersion force).  We set $\epsilon=10$ and $\sigma=1$ in our simulations.

In the experiments, whose results are represented in \ifarXiv Fig.~\ref{f:LJ_main}\fi, the distribution $\probIC$ for the $M$ i.i.d. initial conditions is a standard Gaussian vector in $\mathbb{R}^{2N}$.  In this Lennard-Jones interacting system, one has to be careful in choosing the observation time interval.  Since the minimum distance between the particles at initial configurations is very close to 0 with high probability, the particles have very large velocities (e.g. $\sim10^{22}$) due to the singularity of the interaction kernel at 0. This obstruction made the learning algorithm infeasible since our algorithm is for learning bounded kernels. Therefore, we chose an observation time starting from a suitable time $t_0$, small but positive. On the other side of the training time interval, since the system evolves to equilibrium configurations very quickly, we observe the dynamics up to a time $T$ which is a fraction of the equilibrium time.  In each sampling regime, we observe the dynamics at discrete times $\{t_i\}_{i=2,\dots, L}$ and then use the standard finite difference method to obtain a faithful approximation of velocities of agents. 
\begin{table}[H]\centering
\footnotesize{\begin{tabular}{| c | c | c |c|c|c|c|}
\hline 
$N$   &$d$ &$\probIC$ & $\#$ Trials & $M_{\rhoL}$   & $[t_0, T_f]$ & deg($\psi_{kk'}$)  \\ 
\hline 
 $7$ & 2& $N(0,I_{2N})$&$10$     & $2000$       &  $[t_0, cT]$ & 1\\
\hline
\end{tabular}}
\mycaption{\textmd{(LJ) Parameters used in Lennard-Jones system}}
\label{t:LJ_learn_params}
\end{table}

\begin{table}[H]
\centering
\footnotesize{\begin{tabular}{ |c | c | c | c | c | c |c |c|}
\hline 
&   $M$      & $L$     & $n$   & $[t_0, T]$ & $c$ \\ 
\hline 
\footnotesize{\begin{tabular}{@{}c@{}}Many short traj.\end{tabular}}  &  $200$ & $91$ & $600$&$[0.001, 0.01]$&$50$\\
\hline 
\footnotesize{\begin{tabular}{@{}c@{}}Single  long  traj.\end{tabular}} & $20$ & $4991$ & $600$ &$[0.001, 0.5]$&$2$\\
\hline
\end{tabular}}
\mycaption{\textmd{(LJ) Observation parameters for the Lennard-Jones system}}
\label{t:LJ_params}
\end{table}

Table \ref{t:LJ_learn_params} and Table \ref{t:LJ_params} summarize the parameters used for the two regimes: many short-time trajectories, and a single large-time trajectory.
In the first regime, the randomness of initial conditions enables the agents to explore large regions of state space, and in the space of pairwise distance, in a short time.  In the second regime, the large-time dynamics plays a fundamental role in driving the pairwise distance between agents to cover areas of interest. 
\begin{table}[H]
\centering
\revision{
\footnotesize{\begin{tabular}{| c || c | c |} 
\hline
& Many short trajectories & a few long trajectories \\
\hline
Rel. Err. for $\hat\phi$ & $6.6\cdot 10^{-2}\pm 5 \cdot 10^{-3}$ & $7.2\cdot 10^{-2}\pm 1\cdot 10^{-2}$ \\
\hline
\end{tabular}}
\mycaption{\textmd{(LJ) Relative error of the estimator for the Lennard-Jones system}}\label{t:LJ_errs}
}
\end{table}

The estimator belongs to a piecewise linear function space $\hypspace_n$ of dimension $n=600$.  
As reported in Fig.~\ifarXiv\ref{f:LJ_main}\fi, the estimated interaction kernel $\lintkernel$ approximates the true interaction kernel $\intkernel$ well in the regions where $\rhoT^L$ (and $\rhoT$) is large, i.e. regions with an abundance of observed values of pairwise distances to reconstruct the interaction kernel. 
The dependency on $T$ of $\smash{\rhoL}$, and of the space $\smash{L^2(\rhoL)}$  (see (5) in the main text) used for learning, is rather pronounced, as may be seen from the histogram visualization also in \ifarXiv Fig.~\ref{f:LJ_main}\fi.
As usual we also compare trajectories $\smash{\widehat\bX(t)}$ generated by the system with the estimated interaction kernel learned with trajectories $\bX(t)$ generated by the original system, given the same initial conditions at $t_0$, both on the learning interval $[t_0,T]$ and on larger time intervals  $[t_0,cT]$. Figure \ref{fig:LJ1traj} provides a visualization of such trajectories. \revision{ Visualization of the corresponding systems with a larger number of agents $N_{\text{new}}$ can be found in Figure 1 of the main text}.
We report the estimation errors of the interaction kernel and the trajectory errors in Tables \ref{t:LJ_errs} and \ref{t:LJ_traj_err_MC}.

\revision{Table \ref{t:LJ_errs} shows the mean and standard deviations of the relative $L^2(\rhoT)$ errors of the kernel estimators in 10 different simulations. We report the relative errors of trajectory prediction in SI Sec.~\ifarXiv \ref{LJdescriptions}\fi.}
\begin{table}[H]
\centering
\footnotesize{
\footnotesize{\begin{tabular}{| c || c | c |} 
\hline
                                                             & $[t_0, T]$                    & $[T, T_f]$ \\
\hline
$\text{mean}_{\text{IC}}$: Training ICs & $1.6\cdot 10^{-3}\pm 2 \cdot 10^{-4}$ & $1.7\cdot10^{-2}\pm 2 \cdot10^{-3} $\\
\hline
$\text{std}_{\text{IC}}$: Training ICs    &   $4.6 \cdot10^{-4}\pm 5\cdot10^{-5}$ & $2.1\cdot10^{-2}\pm 4\cdot10^{-3} $ \\
\hline            
$\text{mean}_{\text{IC}}$: Random ICs & $1.6\cdot10^{-3}\pm 2\cdot10^{-4}$ & $1.7\cdot10^{-2}\pm 2\cdot10^{-3} $\\
\hline
$\text{std}_{\text{IC}}$: Random ICs    & $4.5\cdot10^{-4}\pm 5\cdot10^{-5}$  &  $1.9\cdot10^{-2}\pm 2 \cdot10^{-3}$ \\
\hline
$\text{mean}_{\text{IC}}$: Larger $N$ & $6.2 \cdot10^{-2}\pm 7 \cdot10^{-3}$   & $6.2\cdot10^{-2}\pm 2 \cdot10^{-2}$ \\
\hline
$\text{std}_{\text{IC}}$: Larger $N$    & $8.2 \cdot10^{-3}\pm 7\cdot10^{-4}$  & $3.0\cdot10^{-2}\pm 1\cdot10^{-2}$ \\
\hline
\hline
$\text{mean}_{\text{IC}}$: Training ICs & $3.4 \cdot10^{-3}\pm 1 \cdot10^{-3}$  & $5.1\cdot10^{-3}\pm 2 \cdot10^{-3} $ \\
\hline
$\text{std}_{\text{IC}}$: Training ICs    &   $ 2.7 \cdot10^{-3}\pm 2 \cdot10^{-3}$ &  $6.6 \cdot10^{-3}\pm 3 \cdot10^{-3}$ \\
\hline            
$\text{mean}_{\text{IC}}$: Random ICs & $4.1 \cdot10^{-3}\pm 2 \cdot10^{-3}$ & $8.7\cdot10^{-3}\pm 8 \cdot10^{-3} $ \\
\hline
$\text{std}_{\text{IC}}$: Random ICs    & $3.6 \cdot10^{-3}\pm2 \cdot10^{-3}$ &  $1.5 \cdot10^{-2}\pm 2 \cdot10^{-2}$ \\
\hline
$\text{mean}_{\text{IC}}$: Larger $N$ & $7.7 \cdot10^{-2}\pm 1 \cdot10^{-2}$ & $6.6 \cdot10^{-2}\pm 3 \cdot10^{-2}$\\
\hline
$\text{std}_{\text{IC}}$: Larger $N$    & $1.5\cdot10^{-2}\pm 1\cdot10^{-2}$ & $5.7\cdot10^{-2}\pm 3 \cdot10^{-2}$ \\
\hline  
\end{tabular}}}
\mycaption{\textmd{(LJ) Trajectory Errors for Many Short Trajectories Learning (top) and Single Large Time Trajectories Learning (bottom)}}
\label{t:LJ_traj_err_MC}
\end{table}
We also test the convergence of our estimator as $M\rightarrow\infty$: we choose the parameters for observations and learning as in Table \ref{t:LJ_learn_params_convergence}. It is important that we choose the dimension $n$ of hypothesis space to be dependent on $M$, as dictated by Thm. (3.3) in the main text. Also, in this experiment (and this experiment only!) we observe the true  derivatives (instead of approximating them by finite differences of positions), as those would introduce a bias term that does not vanishes unless $L$ also increased with $n$. 
\begin{table}[H]
\centering
\footnotesize{\begin{tabular}{| c | c | c |c|}
\hline 
 $[t_0, T]$ & $L$ & $\log_2(M)$  &$n$\\ 
\hline 
  $[0.001,0.01]$ & 10&$12:21$& $64(M/\log M)^{0.2}$\\
\hline
\end{tabular}}
\mycaption{\textmd{(LJ) Observation parameters in the plot of  convergence rate }}
\label{t:LJ_learn_params_convergence}
\end{table}
We obtain a decay rate for for $ \smash{\|\lintkernel(\cdot)\cdot-{\intkernel}(\cdot)\cdot\|_{L^2(\rhoL)}}$ around $M^{-0.36}$, which is close to the theoretical optimal learning rate $M^{-0.4}$ -- see \ifarXiv Fig.~\ref{f:LJ_main}\fi.
We impute this (small) difference to the singularity of the Lennard-Jones interaction kernel at $0$, which makes this interaction kernel not admissible in the our learning theory.  
 \begin{figure}[H]
\centering
\includegraphics[width=\ifarXiv 0.48\textwidth \fi]{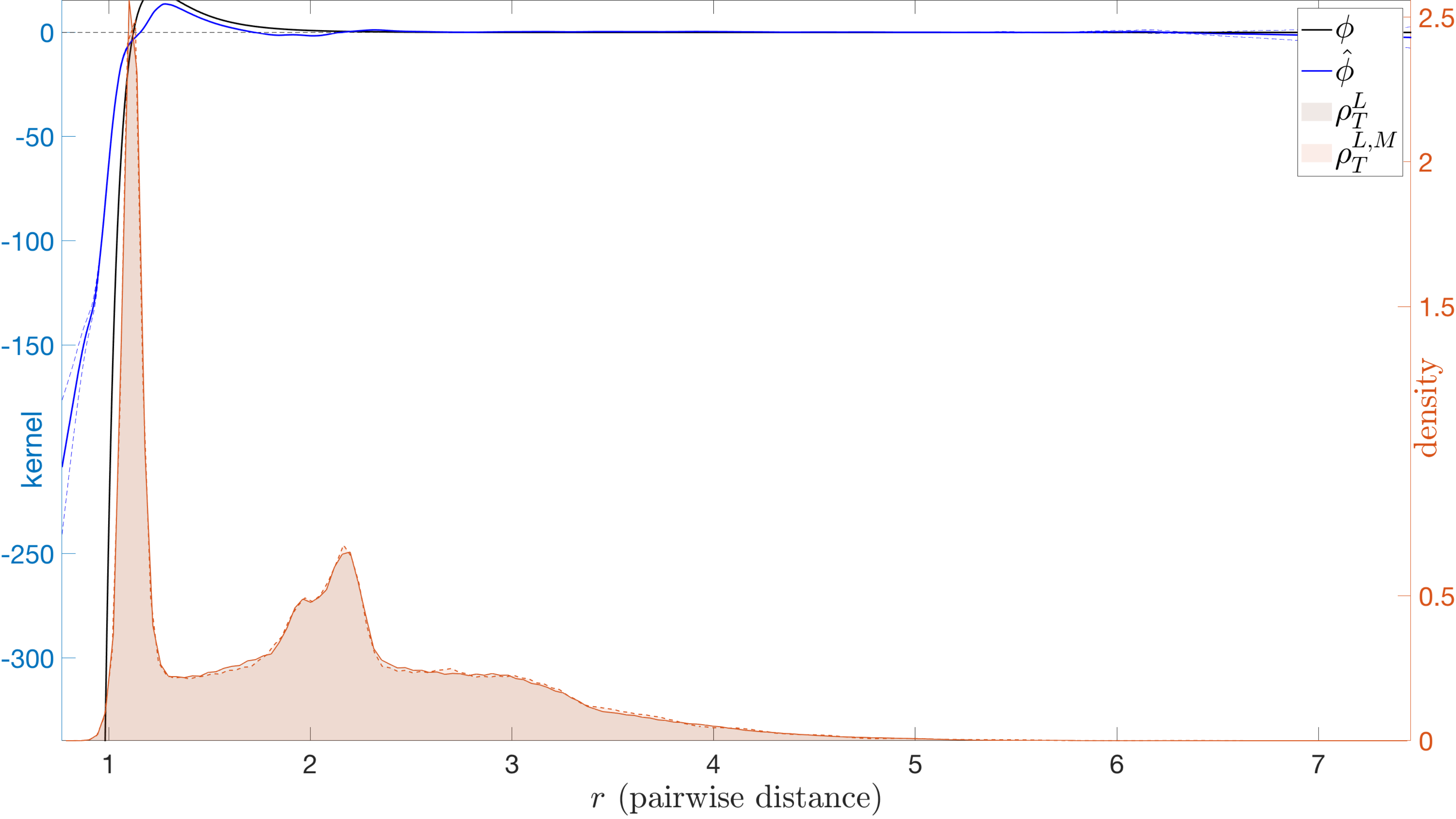}
\mycaption{(LJ) Interaction kernel learned with Unif.$([-\sigma,\sigma])$ additive noise, for $\sigma=0.1$ in the observed positions and velocities; here $M=500$, $L=2000$, with all the other parameters as in Table \ref{t:LJ_params}.}
\label{fig:LJnoise}
\end{figure}

However, the singularity of the Lennard-Jones interaction kernel at $0$ forces the particles close to each other to repel each other. Also, the system evolves rapidly to a steady-state, and the particles only explore a bounded region due to the large range attraction. Therefore, to obtain a well-supported non-degenerate measure $\smash{\rhoL}$, we should make observations on a time interval that avoids reaching either the singularity of the interaction kernel or the steady-state.  
The restriction of the Lennard-Jones interaction kernel to the support of $\smash{\rhoL}$ is bounded and smooth, and hence our learning theory applies and we achieve an almost optimal rate of learning in the numerical experiments. 
The estimated interaction kernel with noisy observation is visualized in Figure \ref{fig:LJnoise}.

Finally, Fig.~\ref{fig:LJ_coercivity_1} reports numerical validations of the coercivity condition in Definition \ref{def_coercivity_SI} for this system. We consider the number of agents $N$ ranging from $5$ to $30$, three different initial distributions $\probIC$, and observations on different time intervals. The coercivity constants computed by Monte Carlo sampling are close to the theoretical lower bound in all these cases. 
\begin{figure}[H]
\centering
\begin{minipage}{\ifarXiv 0.85\textwidth \fi}
\includegraphics[width=\textwidth]{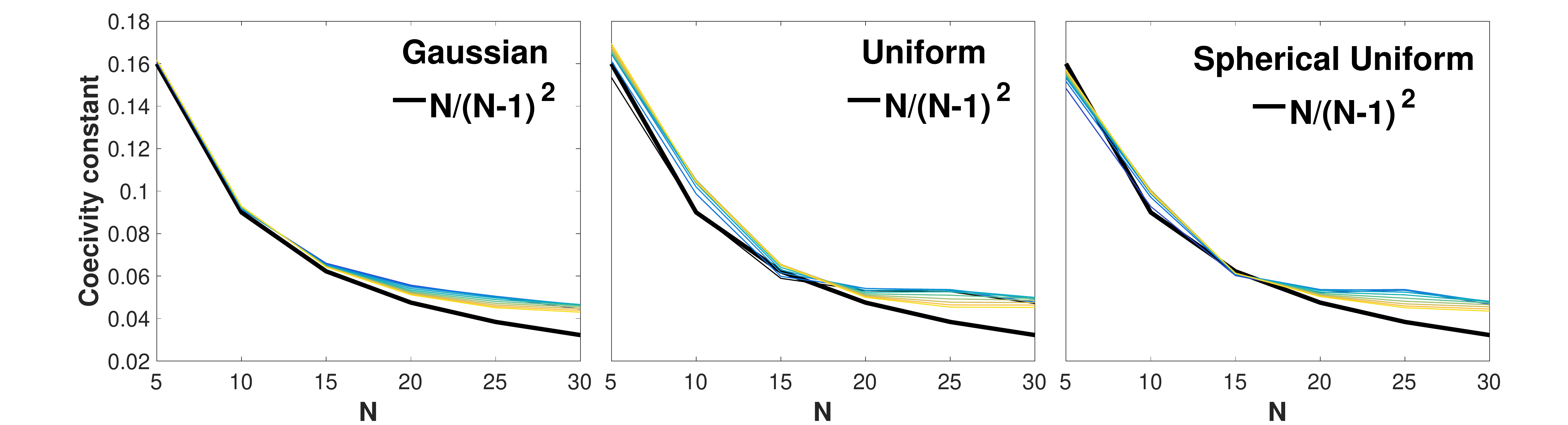}
\end{minipage}
\mycaption{{(LJ) \em{Coercivity condition validation in 2D Lennard-Jones system with different $N$}}. 
We compute the empirical coercivity constant $c_L$  defined in \eqref{gencoer} using $M=131,072$ trajectories with initial conditions drawn from $\probIC$. Three initial distributions for $\probIC$ are tested: the standard Gaussian vector in $\mathbb{R}^{2N}$ (left),  the uniform distribution on $[-0.5, 0.5]^{2N}$ (middle), and the uniform distributions on the unit spheres in $\mathbb{R}^{2N}$ (right). Ten different lengths of trajectories are considered (represented in each figure by the colored curves above the black curve, the theoretical lower bound of $c_L$): each with the same initial time $t_1=0.001$, but the end time $t_L$ ranges from 0.0059 to 0.0509 with a uniform time gap $10^{-4}$. In all these ten sampling regimes (all are short time periods), the coercivity constant is around $\frac{N-1}{N^2}$, matching the theoretical lower bound in \ifarXiv Thm.~\ref{t:coercivity}\fi for one time step. \revision{We also note that $c_L$ appears to not go to $0$ as $N$ increases, consistent with the conjecture that in a rather great generality $c_L$ stays bounded away from $0$ independently of $N$}.}
\label{fig:LJ_coercivity_1} 
\end{figure}

\subsection{Predator-Swarm system}\label{s:SI_PSD}
There is an increasing amount of literature in discussing models of self-organized animal motion \cite{CDOP2009, CDOMBC2007, CPT2010, CF2002, CS2007, Niwa1994, PEK1999, PVG2002, Romey1996, TT1995, YEECBKMS2009}.   Even more challenging is modeling interactions between agents of multiple types, in complex and emergent physical and social phenomena \cite{EMSC2014, PEK1999, CK2009, Nowak2006, FMSP2007}.  We consider here a representative heterogeneous agent dynamics: a Predator-Swarm system with a group of preys and a single predator, governed by either a first order or a second order system of ODE's.  The intensity of interaction(s) between the single predator and group of preys can be tuned with parameters, determining dynamics with various interesting patterns (from confusing the predator with fast preys, to chase, to catch up to one prey).  Since there is one single predator in the system, there is no predator-predator interaction to be learned. The interaction kernels (prey-prey, predator-prey) have both short-range repulsion to prevent the agents to collide, and long-range attraction to keep the agents in the flock.  Because of the strong short-range repulsion, the pairwise distances stay bounded away from $r = 0$.  We will see that these difficulties, similar to those confronted with the Lennard-Jones interaction kernel, do not prevent us from learning the interactions kernels.  

In our notation for the heterogeneous system, the set $C_1$ corresponds to the set of preys, and $C_2$ to the set consisting of the single predator.  

\noindent{\bf{Predator-Swarm, $1^{st}$ order}} (PS$1^{st}$).  We start from the first order system.
It is a special case of the first order heterogeneous agent systems we considered, with the following interaction kernels:
\[
\intkernel_{1, 1}(r) = 1 - r^{-2}, \quad \intkernel_{1, 2}(r) = -2r^{-2}, \quad \intkernel_{2, 1}(r) = 3r^{1.5}, \quad \intkernel_{2, 2}(r) \equiv 0.
\]
The simulation parameters are given in Table \ref{t:PSparams_1}.
\begin{table}[H]
\centering
\footnotesize{\begin{tabular}{| c | c | c | c | c | c |}
\hline 
 $d$  & $N_1$ & $N_2$   & $M$ & $L$ & $T$\\ 
\hline 
 $2$ & $9$      &$1$       & $50$ & $200$ & $5$ \\
\hline
\hline
$n_{1, 1}$ & $n_{1, 2} = n_{2, 1}$ & $n_{2, 2}$ & deg($\psi_{kk'}$) & Preys $\probIC^{\bX}$  & Pred. $\probIC^{\bX}$  \\ 
\hline 
$360$ & $120$ & $64$         &$[1, 1; 1, 0]$       &  Unif. on ring $[0.5, 1.5]$ & Unif. on disk at $0.1$ \\
\hline
\end{tabular}}
\mycaption{\textmd{(PS$1^{st}$) System parameters for first order Predator-Swarm system}}
\label{t:PSparams_1}
\end{table}

In the first column of \ifarXiv Fig.~\ref{fig:example_main}\fi, we show the comparison of the learned interaction kernels versus the true interaction kernels (with $\smash{\rho_T^{L, kk'}}$ and $\smash{\rho_T^{L, M, kk'}}$ shown in the background), and the comparison of true and learned trajectories over two different set of initial conditions.

As is shown in the top left a portion ($4$ sub-figures) of \ifarXiv Fig.~\ref{fig:example_main}\fi, we are able to match faithfully all four learned interactions to their corresponding true interactions over the range of $\rhoT$ when the pairwise distance data is abundant.  We are not able to learn the interaction kernels for $r$ close to $0$, demonstrated by the larger area of uncertainty (surrounded by the dashed lines) towards $0$: first, the prey-to-prey interaction is preventing preys colliding into each other; second, in the case of chasing predators, the preys are able to push away the predator.  The predator-to-prey and prey-to-predator interactions are learned over the same set of pairwise distance data, however, we are able to learn the details of the two interaction kernels, and judging from the learned interaction kernels, they are not simply negative of each other.  The predator-to-predator interaction simply is learned as a zero function, even though there is no pairwise distance data of a predator to a different predator.  Errors in their corresponding $L^2(\rho_T^{L, kk'})$ norms are reported in Table \ref{t:PS1_l2rho_err}.
\begin{figure}[H]
\centering
\includegraphics[width=\ifarXiv 0.75\textwidth \fi]{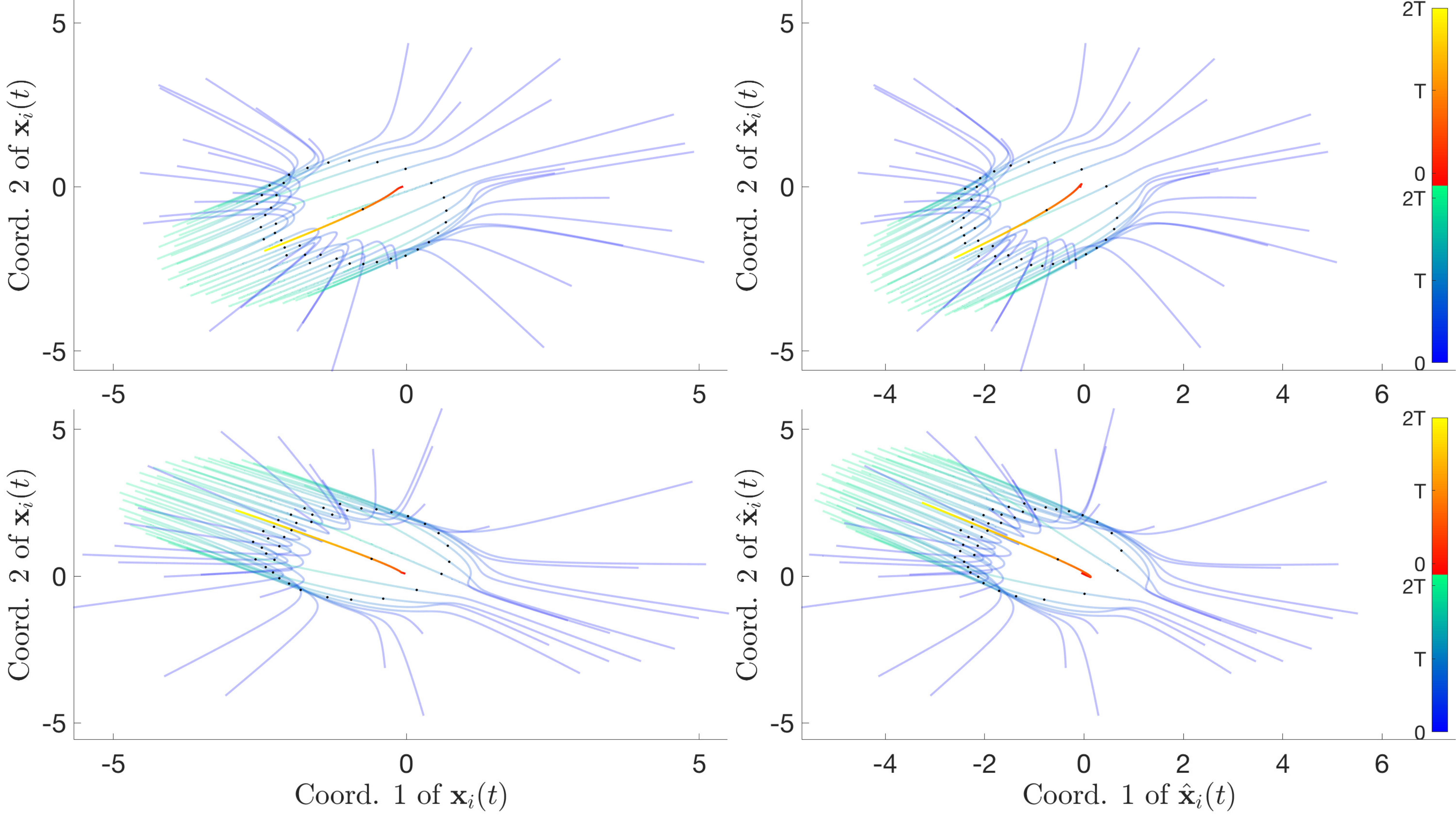}
\mycaption{(PS$1^{st}$) Trajectories $\bX(t)$ and $\smash{\widehat{\bX}(t)}$ obtained with $\intkernel$ and $\smash{\lintkernel}$ respectively, for two randomly chosen initial conditions and evolved for $N_{\text{new}}$ agents (with the same setup as in the case of  $N$ agents).  Trajectory errors are shown in Table \ref{t:PS1_traj_err}.}
\label{fig:PS1_traj_LN}
\end{figure}

The trajectory comparisons are shown in the bottom left portion ($4$ sub-figures) of \ifarXiv Fig.~\ref{fig:example_main}\fi.  We use color changing lines to indicate the movement of agents in time: with the blue-to-green lines attached to preys and the red-to-yellow line for the predator).  The black dot on the trajectories indicate the position of the agents at time $t = T$, and it shows the time divide: the first half of the time, $[0, T]$, is used for learning; and the second half of the time, $[T, T_f]$, is used for prediction.  

And the first row of $2$ sub-figures show the comparison of the trajectories over the initial condition taken from training data, it shows (visually) no major difference between the two, except one of the prey-trajectory, is having a bigger loop in the learned trajectories.  The second row of $2$ sub-figures compares the trajectories from a randomly chosen initial condition (outside of the training set).  
We are able to predict the movement of the predator in the learned trajectories, and movement of most preys.
In Fig.~\ref{fig:PS1_traj_LN} we compare the true and predicted trajectories over a corresponding system a dynamics but with a larger number $N_{\text{new}}$ of agents.
Table \ref{t:PS1_traj_err} reports the max-in-time error \eqref{e:tm_norm} in the trajectories in all cases considered.
We consider the effect of adding noise to observations, with results visualized in \ifarXiv Fig.~\ref{fig:PS1noisetraj}\fi
\begin{table}[H]\centering
\footnotesize{\begin{tabular}{| c || c |} 
\hline
Rel. Err. for $\hat\phi_{1, 1}$ & $5.6 \cdot10^{-2} \pm 1.1 \cdot10^{-3}$ \\
\hline
Rel. Err. for $\hat\phi_{1, 2}$ & $6.6 \cdot10^{-3} \pm 2.4 \cdot10^{-3}$ \\
\hline
Rel. Err. for $\hat\phi_{2, 1}$ & $2.7 \cdot10^{-2} \pm 8.9 \cdot10^{-3}$ \\
\hline
Abs. Err. for $\hat\phi_{2, 2}$ & $0$ \\
\hline  
\end{tabular}}
\mycaption{\textmd{(PS$1^{st}$) Estimator Errors}}
\label{t:PS1_l2rho_err}
\end{table}

\begin{table}[H]\centering
\footnotesize{\begin{tabular}{| c || c | c |} 
\hline
                                                             & $[0, T]$                                                    & $[T, T_f]$\\
\hline
$\text{mean}_{\text{IC}}$: Training ICs & $4.2 \cdot10^{-2} \pm 1.0 \cdot10^{-2}$  & $1.1 \cdot10^{-1} \pm 3.0 \cdot10^{-2}$\\
\hline
$\text{std}_{\text{IC}}$: Training ICs    & $7.2 \cdot10^{-2} \pm 5.6 \cdot10^{-2}$   & $1.9 \cdot 10^{-1} \pm 1.4 \cdot 10^{-1}$ \\
\hline            
$\text{mean}_{\text{IC}}$: Random ICs & $3.8 \cdot10^{-2} \pm 1.4 \cdot10^{-2}$  & $9.5 \cdot10^{-2} \pm 3.2\cdot10^{-2}$\\
\hline
$\text{std}_{\text{IC}}$: Random ICs    & $5.5 \cdot10^{-2} \pm 6.2 \cdot10^{-2}$  & $1.4 \cdot 10^{-1} \pm 1.4 \cdot 10^{-1}$\\
\hline 
$\text{mean}_{\text{IC}}$: Larger $N$ & $4.2 \cdot 10^{-1} \pm 1.7 \cdot 10^{-1}$ & $3.1 \pm 4.6$\\
\hline
$\text{std}_{\text{IC}}$: Larger $N$    & $1.7 \cdot10^{-1} \pm 9.6 \cdot10^{-2}$  & $15.8 \pm 27.4$\\
\hline         
\end{tabular}}
\mycaption{\textmd{(PS$1^{st}$) Trajectory Errors}}
\label{t:PS1_traj_err}
\end{table}
\revision{We show numerically that our learning approach is robust to the choice of hypothesis space, as predicted by the theory, by testing on the Predator-Swarm, $1^{st}$-order system with the B-splines basis.  Results are shown in Fig.~\ref{fig:PS1splines}. Note that the estimators perform similarly in comparison with \ifarXiv Fig.~\ref{fig:PS1noisetraj}\fi  are consistent with the error statistics in Table \ref{t:PS1_traj_err}, in both of which the hypothesis space uses piece-wise polynomial basis.
\begin{figure}[H]
\centering
\includegraphics[width=\ifarXiv 0.75\textwidth \fi]{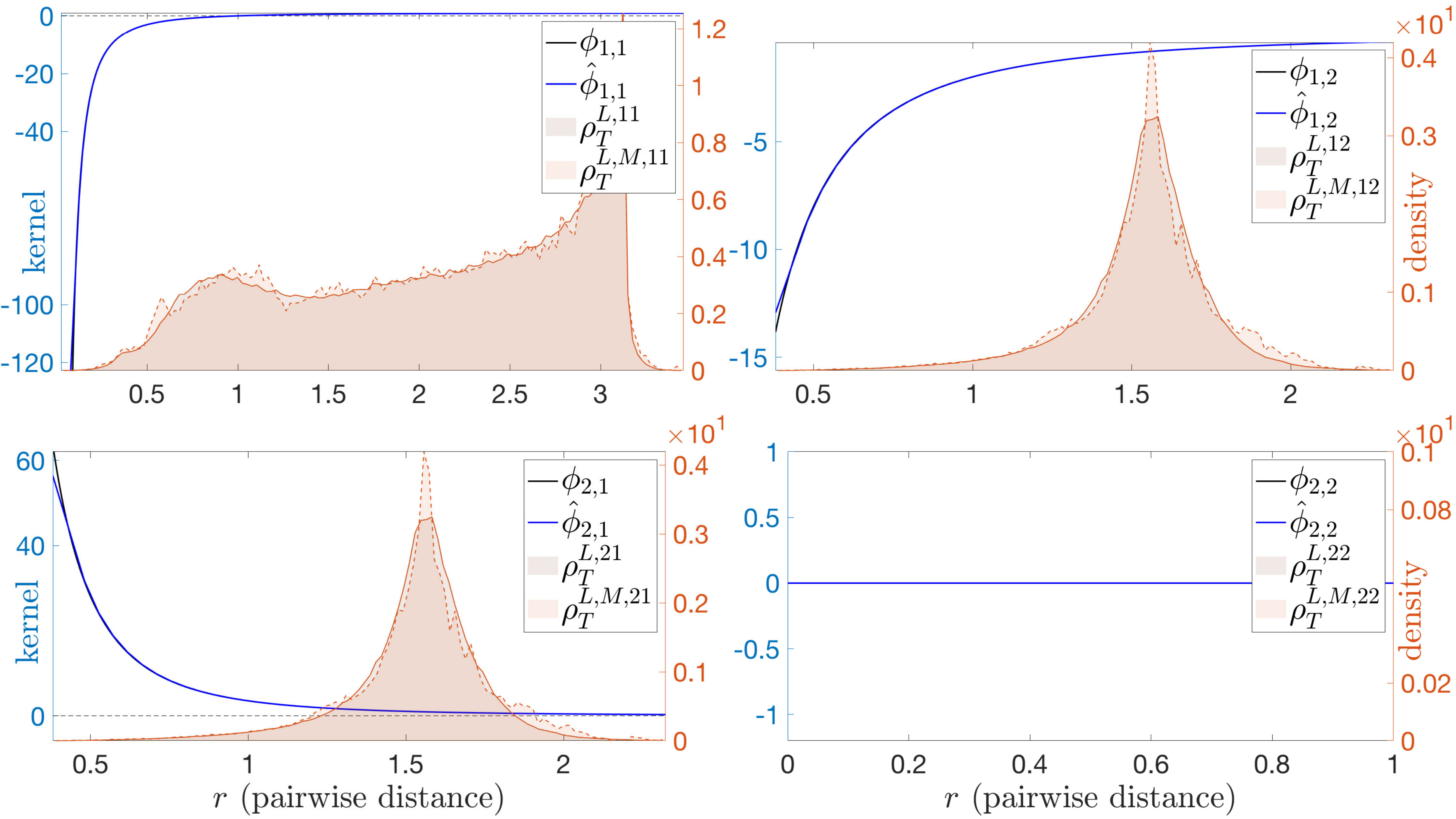}
\mycaption{\revision{(PS$1^{st}$) Comparison of interaction kernels (true versus learned) when the learned kernels are generated by linear B-splines ($n$ as in the other case considered for this system).  The relative error (in $L^2(\rho_T)$ norm) for prey on prey interaction is: $6.6 \cdot 10^{-2}$; for predatory on prey: $6.1 \cdot 10^{-3}$; for prey on predator: $3.6 \cdot 10^{-2}$; and finally for predator on predator: $0$.}}
\label{fig:PS1splines}
\end{figure}}

\begin{figure}[H]
\centering
\includegraphics[width=\ifarXiv 0.48\textwidth \fi]{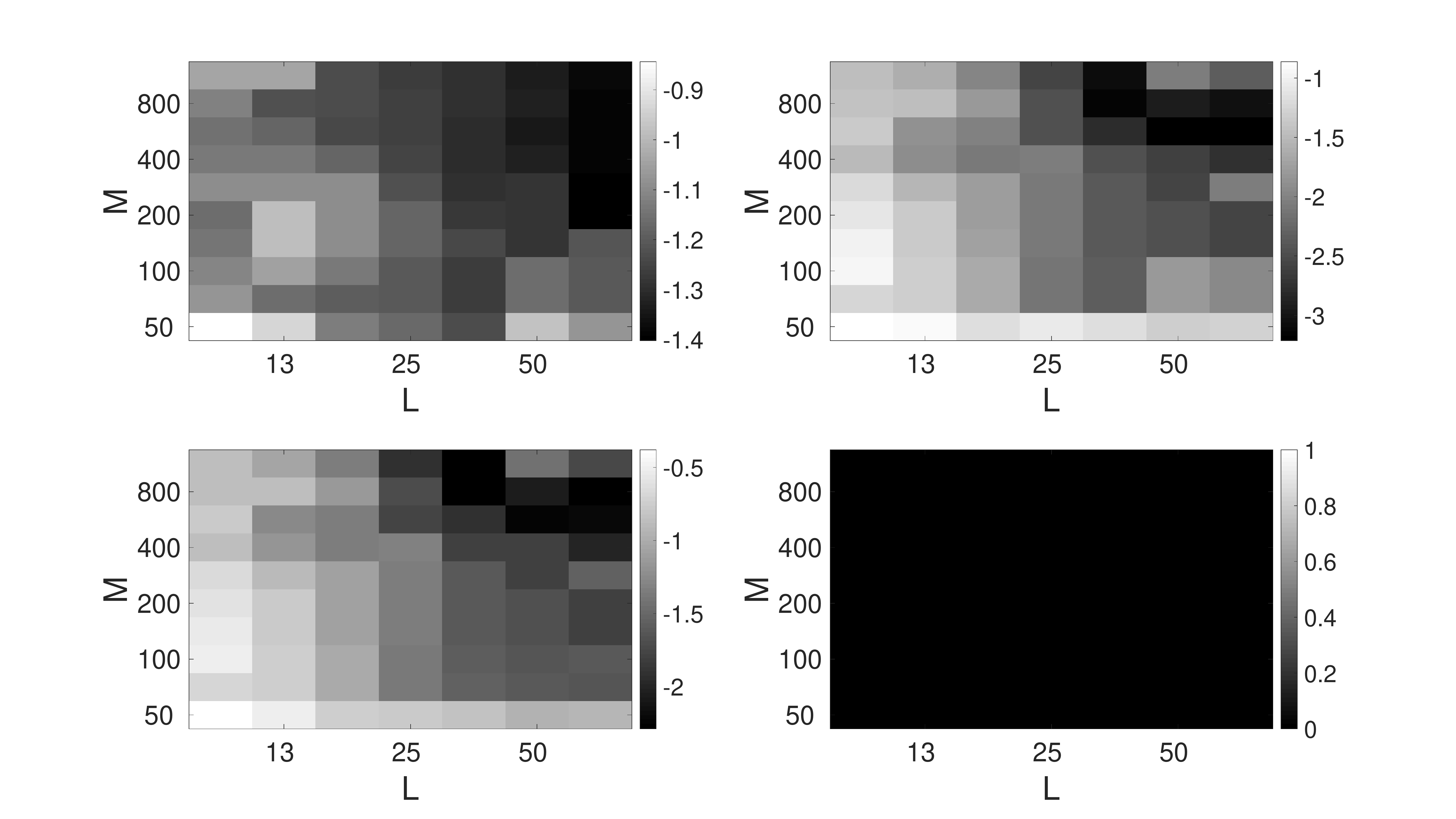}
\mycaption{(PS1) Relative error, in $\log_{10}$ scale, of $\hat\intkernel^E_{k,k'}$ (with $(k,k')$ increasing lexicographically from top-left to bottom-right) as a function of $L$ and $M$. The error decreases both in $L$ and $M$, in fact roughly in the product $ML$. The fourth plot is an identically $0$ absolute error, because both $\phi^E_{2,2}$ and its estimator are identically $0$, since there is only one predator. \revision{Note $M\gg1$ seems to be needed for accurate inference of the interaction kernels, regardless of how large $L$ is: the trajectories explored for small $M$ do not explore enough configuration to enable estimation, suggesting that the limit $M\rightarrow+\infty$ considered in this work is of fundamental importance, at least for non-ergodic systems.}}
\label{fig:PS1MLtest}
\end{figure}

\noindent{\bf{Predator-Swarm, $2^{nd}$-order}} (PS$2^{nd}$).  The second order Predator-Swarm system is a special case of the second order system which is considered in this paper, without alignment-based interactions and without environment variables $\xi_i$'s, similar to the Cucker-Dong model of repulsion-attraction \cite{CD2014} and D'Orsogna-Bertozzi model for modeling fish school formation \cite{CDOP2009, CDOMBC2007} without the non-collective forcing term.  The energy-based interactions are
\[
\intkernel_{1, 1}(r) = 1 - r^{-2}, \quad \intkernel_{1, 2}(r) = -r^{-2}, \quad \intkernel_{2, 1}(r) = 1.5r^{-2.5}, \quad \intkernel_{2, 2}(r) \equiv 0.
\]
The non-collective change on $\dot\bx_i$ is $\forcev_i(\dot\bx_i,\xi_i) = -\nu_{\clof_i}\dot\bx_i$, where the friction constants are type-based and $\nu_{\idxcl} = 1$ for all $\idxcl = 1, \cdots, \numcl$; and the mass of each agent is $m_i = 1$ for all $i = 1, \cdots, N$.  We consider the system and test parameters given in table \ref{t:PSparams_2} (the initial velocity of preys and predator are fixed at $0\in\R^2$).
\begin{table}[H]\centering
\footnotesize{\begin{tabular}{| c | c | c | c | c | c | c |}
\hline 
 $d$  & $N_1$ & $N_2$   & $M$ & $L$  & $T$\\ 
\hline 
 $2$ & $9$      &$1$       & $150$ & $300$ & $10$ \\
\hline
\hline
$n_{1, 1}$ & $n_{1, 2}=n_{2, 1}$ & $n_{2, 2}$ & deg($\psi_{kk'}^E$)                 & Preys $\probIC^{\bX}$      & Pred. $\probIC^{\bX}$  \\ 
\hline 
$1620$ & $540$ & $180$        &$[1, 1; 1, 0]$       &  Unif. on $[0.1, 1]^2$ & Unif. on $[0, 0.08]^2$ \\
\hline
\end{tabular}}
\mycaption{\textmd{(PS$2^{nd}$) System Parameters}}
\label{t:PSparams_2}
\end{table}

Note that the two dynamics, predator-prey $1^{st}$ order and predator-prey $2^{nd}$ order, use a similar set of interaction kernels, however, the resulting dynamics are significantly different from each other, as demonstrated in both the distribution of pairwise distance data and in the trajectories.

In the middle column of \ifarXiv Fig.~\ref{fig:example_main}\fi, we show the comparison of the learned interaction kernels versus the true interaction kernels (with $\smash{\rho_{T, r}^{L, kk'}}$ and $\smash{\rho_{T, r}^{L, M, kk'}}$ shown in the background), and the comparison of true and learned trajectories over two different set of initial conditions.
Similar observations to those for the $1^{st}$ order system apply here.
Errors of the estimators in the $\smash{L^2(\rho_T^{L, kk'})}$ norms are reported in Table \ref{t:PS2_l2rho_err}.
The test on trajectories (bottom middle portion ($4$ sub-figures) of \ifarXiv Fig.~\ref{fig:example_main}\fi) shows visually the accuracy of the predicted trajectories, quantified by the numerical report in Table \ref{t:PS2_traj_err}.
We also compare in Fig.~\ref{fig:PS2_traj_LN} the true and learned trajectories over a corresponding system with $N_{\text{new}}$ agents.
We consider the effect of adding noise to observations, with results visualized in Figure \ref{fig:PS2noise}.
Figures \ref{fig:PS1MLtest} and \ref{fig:PS2MLtest} show the behavior of the error of the estimator (for systems (PS1$^{st}$) and (PS2$^{nd}$) respectively) as both $L$ and $M$ are increased.
\begin{figure}[H]
\centering
\includegraphics[width=\ifarXiv 0.75\textwidth \fi]{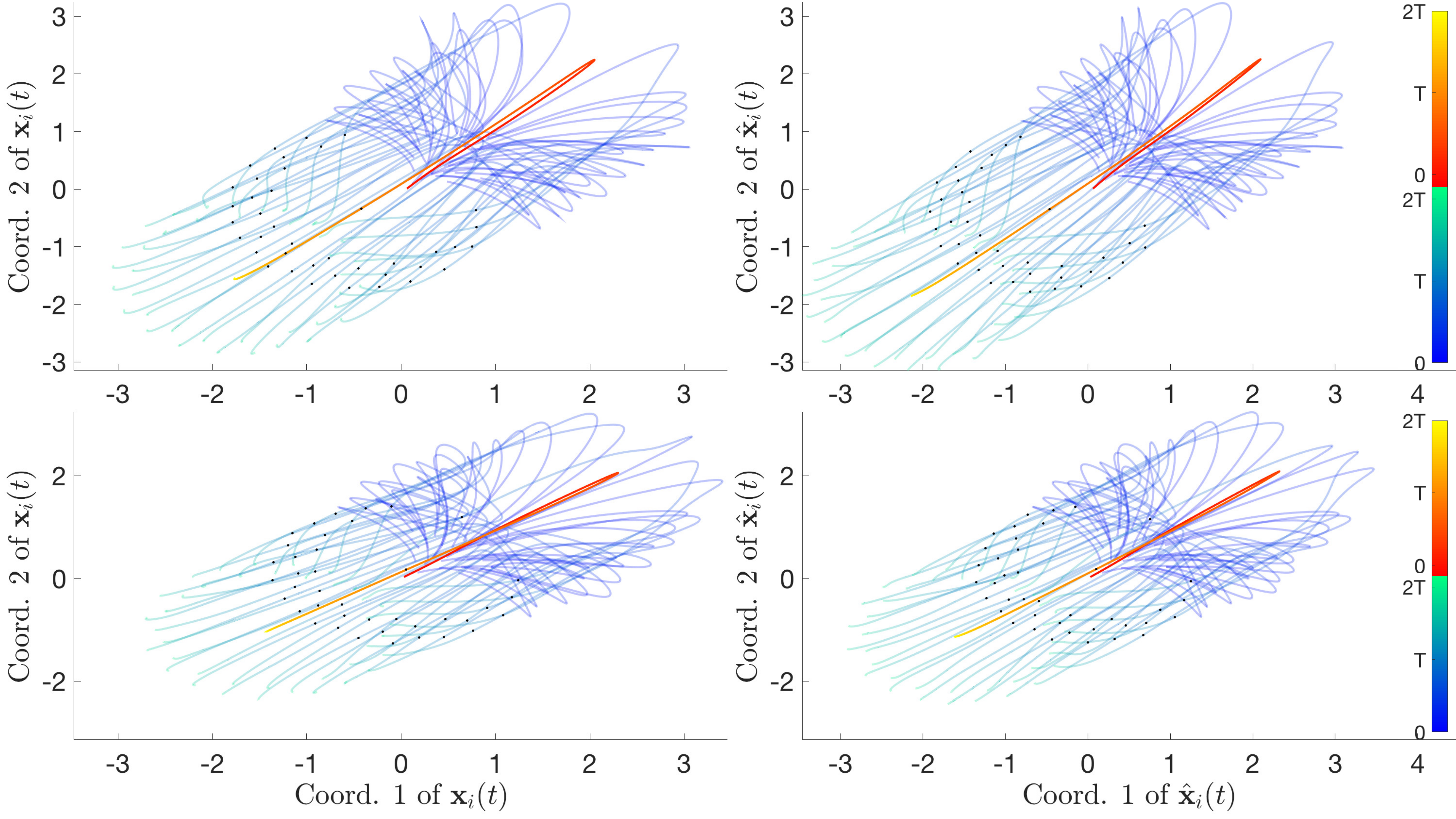}
\mycaption{(PS$2^{nd}$) Trajectories $\bX(t)$ and $\smash{\widehat{\bX}(t)}$ obtained with $\intkernel$ and $\smash{\lintkernel}$ respectively, for two randomly chosen initial conditions and evolved for $N_{\text{new}}$ agents (with the same setup as in the case of  $N$ agents).  Trajectory errors are shown in Table \ref{t:PS2_traj_err}.}
\label{fig:PS2_traj_LN}
\end{figure}

\begin{table}[H]\centering
\footnotesize{\begin{tabular}{| c || c |} 
\hline
Rel. Err. for $\hat\phi_{1, 1}^E$ & $1.5 \cdot10^{-1} \pm 5.0 \cdot10^{-2}$ \\
\hline
Rel. Err. for $\hat\phi_{1, 2}^E$ & $1.3 \cdot10^{-1} \pm 1.1 \cdot10^{-2}$ \\
\hline
Rel. Err. for $\hat\phi_{2, 1}^E$ & $7.1 \cdot10^{-1} \pm 3.8 \cdot10^{-1}$ \\
\hline
Abs. Err. for $\hat\phi_{2, 2}^E$ & $0$ \\
\hline  
\end{tabular}}
\mycaption{\textmd{(PS$2^{nd}$) Estimator Errors}}
\label{t:PS2_l2rho_err}
\end{table}

\begin{table}[H]\centering
\footnotesize{\begin{tabular}{| c || c | c |}
\hline
                                                             & $[0, T]$                                                     & $[T, T_f]$ \\
\hline
$\text{mean}_{\text{IC}}$: Training ICs & $3.5 \cdot10^{-1} \pm 1.2 \cdot10^{-1}$  & $7.9 \cdot10^{-1} \pm 2.1 \cdot10^{-1}$\\
\hline
$\text{std}_{\text{IC}}$: Training ICs    & $6.5 \cdot10^{-1} \pm 2.7 \cdot10^{-1}$  & $1.2 \pm 3.7\cdot10^{-1}$\\
\hline            
$\text{mean}_{\text{IC}}$: Random ICs & $3.5 \cdot10^{-1} \pm 1.2 \cdot10^{-1}$  & $8.0 \cdot10^{-1} \pm 2.3\cdot10^{-1}$\\
\hline
$\text{std}_{\text{IC}}$: Random ICs    & $5.8 \cdot10^{-1} \pm 1.6 \cdot10^{-1}$  &$1.2 \pm 3.1 \cdot10^{-1}$\\
\hline 
$\text{mean}_{\text{IC}}$: Larger $N$ & $2.0 \cdot10^{-1} \pm 3.0 \cdot10^{-2}$   & $4.6 \cdot10^{-1} \pm 1.2 \cdot10^{-1}$\\
\hline
$\text{std}_{\text{IC}}$: Larger $N$    & $1.1 \cdot10^{-1} \pm 1.4 \cdot10^{-2}$  & $2.5 \cdot10^{-1} \pm 5.6 \cdot10^{-2}$\\
\hline      
\end{tabular}}
\mycaption{\textmd{(PS$2^{nd}$) Trajectory Errors}}
\label{t:PS2_traj_err}
\end{table}

\begin{figure}[H]
\centering
\includegraphics[width=\ifarXiv 0.75\textwidth \fi]{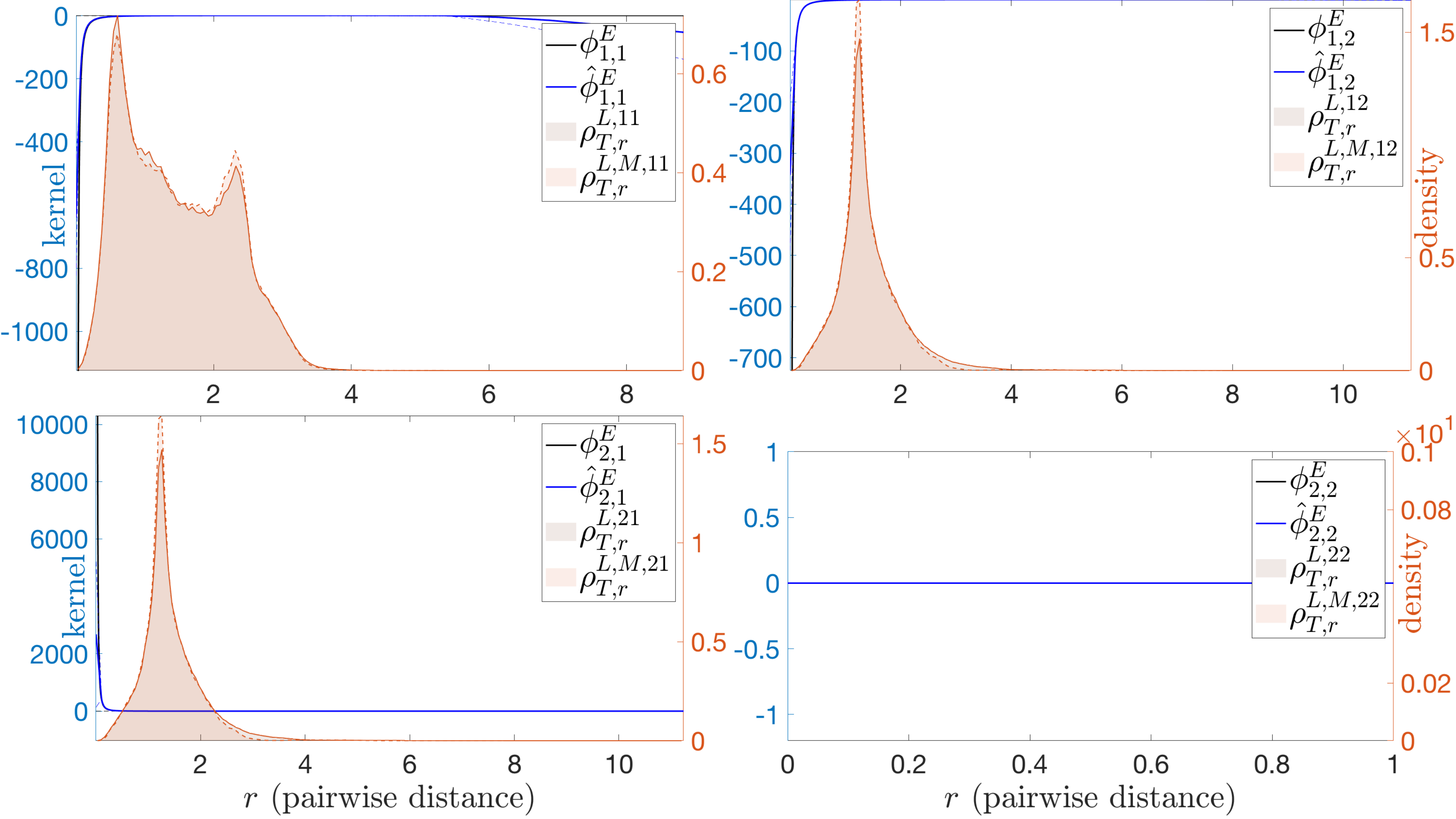}
\mycaption{(PS$2^{nd}$) Interaction kernels learned with Unif.$([-\sigma,\sigma])$ multiplicative noise, for $\sigma=0.1$ in the observed positions and velocities, with parameters as in Table \ref{t:PSparams_2}. The estimated kernels are minimally affected, mostly in regions with small $\rhoL$ near $0$.}
\label{fig:PS2noise}
\end{figure}

\begin{figure}[H]
\centering
\includegraphics[width=\ifarXiv 0.75\textwidth \fi]{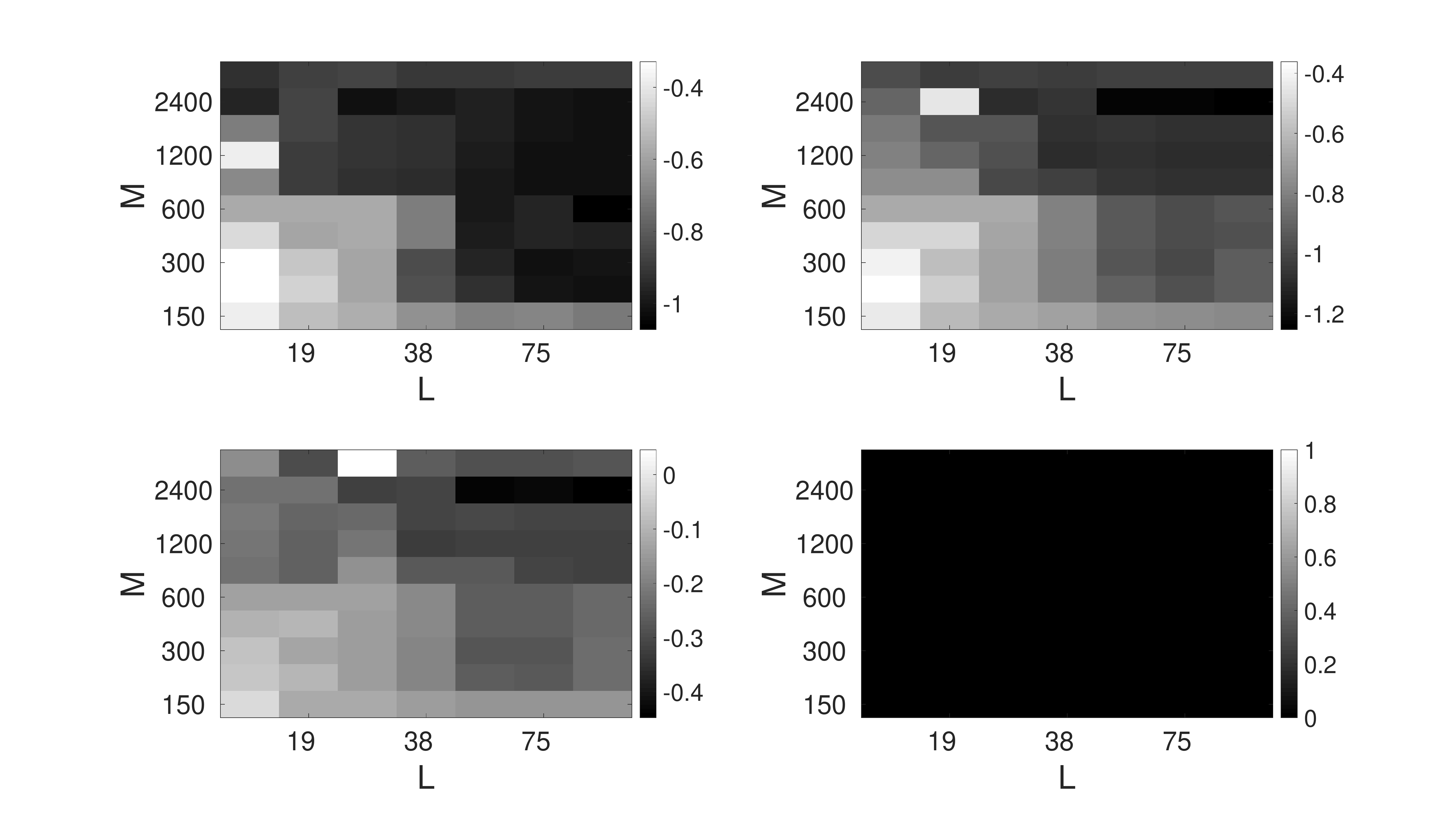}
\mycaption{(PS2) Relative error, in $\log_{10}$ scale, of $\hat\intkernel^E_{k,k'}$ (with $(k,k')$ increasing lexicographically from top-left to bottom-right) as a function of $L$ and $M$. The error decreases both in $L$ and $M$, in fact roughly in the product $ML$. The fourth plot is an identically $0$ absolute error, because both $\phi^E_{2,2}$ and its estimator are identically $0$, since there is only one predator. \revision{Note $M\gg1$ seems to be needed for accurate inference of the interaction kernels, regardless of how large $L$ is: the trajectories explored for small $M$ do not explore enough configuration to enable estimation, suggesting that the limit $M\rightarrow+\infty$ considered in this work is of fundamental importance, at least for non-ergodic systems.}}
\label{fig:PS2MLtest}
\end{figure} 
\subsection{Phototaxis Dynamics}\label{s:SI_PT}
Second order models have been widely used in describing self-organized human motion \cite{CPT2011, CSZ2004, SDOPTBBC2008}, synthetic agent (robots, drones, etc.) behavior \cite{CHDOB2007, LF2001, PGE2009, SS1997}, and bacteria/cell aggregation and motility \cite{CDFSTB2003, KS1970, KW1998, Perthame2007}.  A step further in accurately model reality is to consider models with responses of agents to their surrounding environment or the spread of emotion among agents within a system.  Such phenomena appear in a variety of applications, including modeling of emergency evacuation, crowded pedestrian dynamics, bacteria movement toward certain food sources \cite{MHT2011, DGAB2015, BDMTvdW2009, BHKTvdW2011, LL2015, CDFSTB2003, KS1970, KW1998, Perthame2007}.  We choose here a system modeling the dynamics of phototactic bacteria towards a fixed light source.  This system extends the Cucker-Smale system \cite{CS2007, CS2007a, HHK2010} with an extra auxiliary variable $\xi_i$ modeling the response (called excitation level) of individual bacteria to the light source.  The dynamics is known to lead to flocking (all bacteria moving in the same direction) within a rather short amount time, due to the interaction kernel having a long interaction range and the effect of light entering the dynamics uniformly.  This system is within our family of the second order systems, with homogeneous agents and no energy-induced interaction kernel.  The alignment-based interaction kernels acting on $\dot\bx_i$ and $\xi_i$ are the same:
\[
\intkernel^{\bv}(r) = \intkernel^{\xi}(r) = {(1 + r^2)^{-\frac14}}.
\]
The non-collective change on $\dot\bx_i$ is given by
\[
\forcev_i(\dot\bx_i,\xi_i) = I_0(\bv_{\text{term}} - \dot\bx_i)(1 - \gamma(\xi_i; \xi_{\text{cr}})),
\]
where $I_0 = 0.1$ is the light intensity, $\bv_{\text{term}} = (60,0)$ is the terminal velocity (light source at infinity), $\xi_{\text{cr}} = 0.3$ is the critical excitation level (when the light effect activates the bacteria), and $\gamma(\cdot)$ is the smooth cutoff function
\[
\gamma(\xi; \xi_c) = \left\{
        \begin{array}{ll}
          1,                                                             & \quad 0        \le \xi < \xi_c, \\
          \frac12(\cos(\frac{\pi}{\xi_c}(\xi - \xi_c) + 1), & \quad \xi_c   \le \xi < 2\xi_c, \\
          0,                                                             & \quad 2\xi_c \le \xi.
        \end{array}
    \right.
\]
Here $\xi_c$ is a a threshold constant.  The non-collective change on $\xi_i$ is given by
\[
\forcexi_i(\xi_i) = I_0\gamma(\xi_i; \xi_{\text{cp}}),
\]
where $\xi_{\text{cp}} = 0.6$ is the maximum excitation level of light effect on the bacteria.  The system parameters are summarized in Table \ref{t:PTparams}.
\begin{table}[H]\centering
\footnotesize{\begin{tabular}{| c | c | c | c |}
\hline 
 $d$  & $M$ & $L$     & $T$      \\
\hline
$2$ & $50$ &$200$ &$0.25$ \\
\hline
\hline
$\probIC^{\bX} = \probIC^{\dot\bX}$  & $\probIC^{\Xi}$                      & $n^{\bv} = n^{\xi}$ & deg($\psi_{kk'}^A$) $=$ deg($\psi_{kk'}^\xi$) \\ 
\hline 
Unif. on $[0, 100]^2$                  & Unif. on $[0, 0.001]^2$ & $400$                     & $1$\\
\hline
\end{tabular}}
\mycaption{\textmd{(PT) Parameters for Phototaxis Dynamics}}
\label{t:PTparams}
\end{table}

In the right column of \ifarXiv Fig.~\ref{fig:example_main}\fi, we show the comparison of the learned interaction kernels $\lintkernel^{A}$ and $\lintkernel^{\xi}$ versus the true interaction kernels, as well as the comparison of true and learned trajectories over two different set of initial conditions.
We are able to accurately learn the interaction kernels $\lintkernel^{A}$ and $\lintkernel^{\xi}$ over the support of $\rhoT$ when pairwise distance data is abundant.  When the pairwise distance data becomes scarce towards the two ends of the interaction interval $[0, R]$, we are able to faithfully capture the behavior of $\intkernel$ at $r = 0$; the errors are larger near the upper end $r = R$, where the data is extremely scarce.  Crucially, we recover faithfully the interactions between the agents and their environment. Estimation errors in the appropriate $L^2(\rho^L_{T, r, \dot{r}})$- and $L^2(\rho^L_{T, r, \xi})$-norms are reported in Table \ref{t:PT_l2rho_err}. A case with noisy observation is also investigated and shown in Fig.~\ref{fig:PTnoise}.  Trajectory errors are shown in Table \ref{t:PT_traj_err}.
We also compare in Fig.~\ref{fig:PT_traj_LN} the true and learned trajectories for a corresponding system a dynamics with larger $N$.
\begin{figure}[H]
\centering
\includegraphics[width=\ifarXiv 0.48\textwidth \fi]{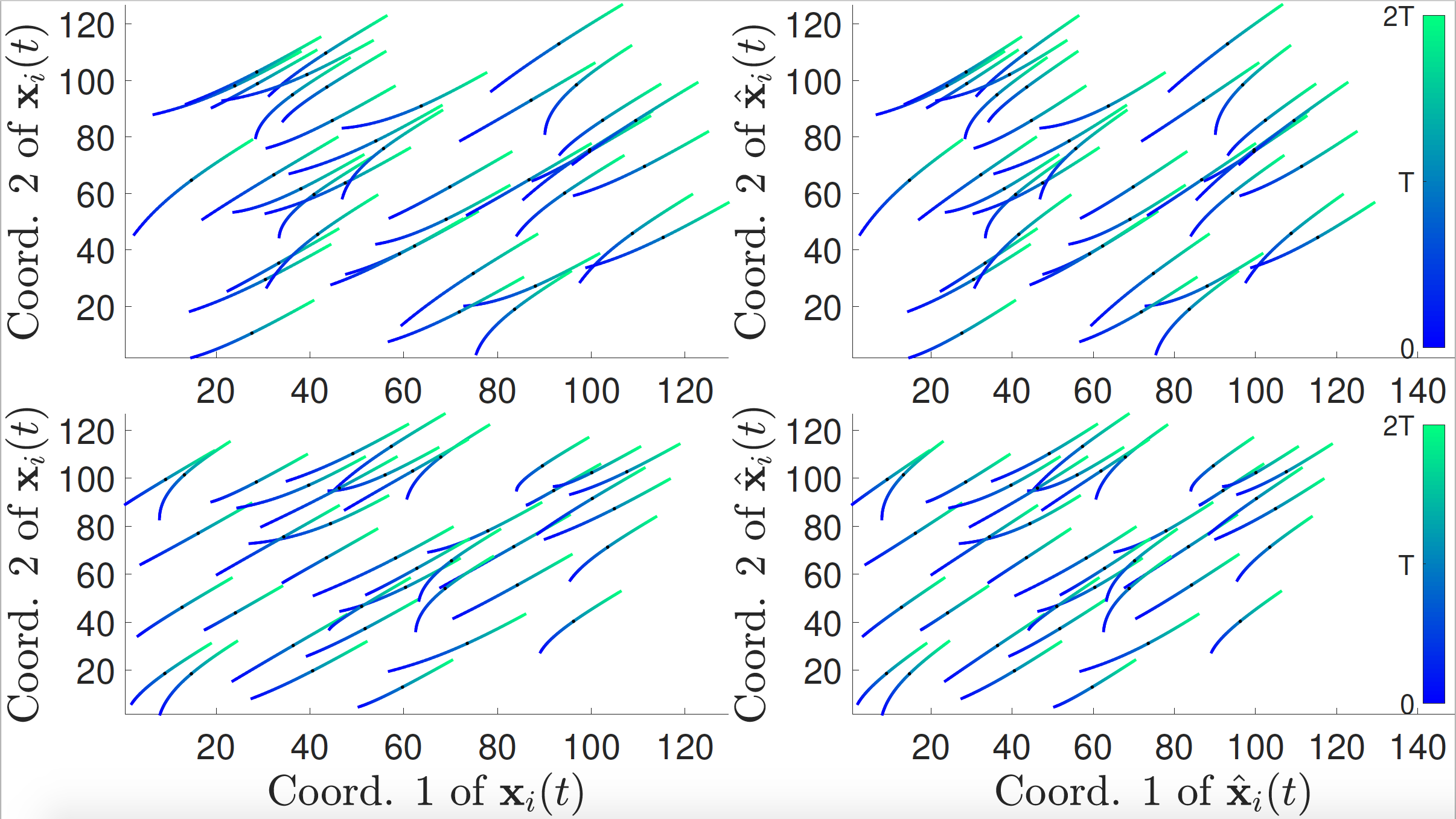}
\mycaption{(PT) Trajectories $\bX(t)$ and $\smash{\widehat{\bX}(t)}$ obtained with true and learned interaction kernels respectively, for two randomly chosen initial conditions and evolved using the larger number of agents $N_{\text{new}}$ (governed by the same equations as in the case of $N$ agents).  Trajectory errors are shown in Table \ref{t:PT_traj_err}.}
\label{fig:PT_traj_LN}
\end{figure}

\begin{table}[H]\centering
\footnotesize{\begin{tabular}{| c || c |} 
\hline
Rel. Err. for $\lintkernel^A$   & $9.4 \cdot10^{-3} \pm 5.2 \cdot10^{-3}$ \\
\hline
Rel. Err. for $\lintkernel^\xi$ & $8.2 \cdot10^{-3} \pm 5.0 \cdot10^{-3}$ \\
\hline  
\end{tabular}}
\mycaption{\textmd{(PT) Estimator Errors}}
\label{t:PT_l2rho_err}
\end{table}

\begin{table}[H]\centering
\footnotesize{\begin{tabular}{| c || c | c |} 
\hline
                                                             & $[0, T]$                                                   & $[T, T_f]$\\
\hline
$\text{mean}_{\text{IC}}$: Training ICs & $1.6 \cdot10^{-3} \pm 5.7 \cdot10^{-5}$ &$6.5 \cdot10^{-3} \pm 9.1 \cdot10^{-4}$ \\
\hline
$\text{std}_{\text{IC}}$: Training ICs    & $3.1 \cdot10^{-4} \pm 4.8 \cdot10^{-5}$ & $8.1 \cdot10^{-3} \pm 3.9 \cdot10^{-3}$\\
\hline            
$\text{mean}_{\text{IC}}$: Random ICs & $1.8 \cdot10^{-3} \pm 8.0 \cdot10^{-4}$ &  $7.3 \cdot10^{-3} \pm 3.2\cdot10^{-3}$\\
\hline
$\text{std}_{\text{IC}}$: Random ICs    & $1.5 \cdot10^{-3} \pm 3.4 \cdot10^{-3}$  &  $1.1 \cdot10^{-2} \pm 1.2 \cdot10^{-2}$\\
\hline 
$\text{mean}_{\text{IC}}$: Larger $N$ & $4.2 \cdot10^{-3} \pm 1.6 \cdot10^{-3}$   & $8.4 \cdot10^{-3} \pm 3.8 \cdot10^{-3}$\\
\hline
$\text{std}_{\text{IC}}$: Larger $N$    & $2.9 \cdot10^{-3} \pm 3.0 \cdot10^{-3}$  & $7.9 \cdot10^{-3} \pm 7.0 \cdot10^{-3}$\\
\hline           
\end{tabular}}
\mycaption{\textmd{(PT) Trajectory Errors}}
\label{t:PT_traj_err}
\end{table}

Finally we display, in Fig.~\ref{fig:PT_rhoLTA} and \ref{fig:PT_rhoLTXi}, the two joint distributions $\smash{\rho^L_{T, r, \dot{r}}}$ and $\smash{\rho^L_{T, r, \xi}}$, used to define the appropriate $L^2$-norms for measuring the performance of $\smash{\lintkernel^{A}}$ and $\smash{\lintkernel^{\xi}}$.  We also calculated the $\ell^1$ distance between the joint distribution $\rho^L_{T, r, \dot{r}}$ and the product of its marginals, and it is $1.3 \cdot 10^{-1}$.  For the $\ell^1$ distance between $\rho^L_{T, r, \xi}$ and the product of its marginals, it is $6.7 \cdot 10^{-2}$.  For the empirical distributions (over $10$ learning trials), the $\ell^1$ distance for $\rho^{L, M}_{T, r, \dot{r}}$ and the product of its marginal is $7.2 \cdot 10^{-1} \pm 1.0 \cdot 10^{-2}$; whereas the $\ell^1$ distance of $\rho^{L, M}_{T, r, \xi}$ to the product of its marginals is $3.7 \cdot 10^{-1} \pm 6.7 \cdot 10^{-3}$.
\begin{figure}[H]
\centering
\begin{subfigure}[b]{\ifarXiv 0.48\textwidth \fi}
   \includegraphics[width=1\linewidth]{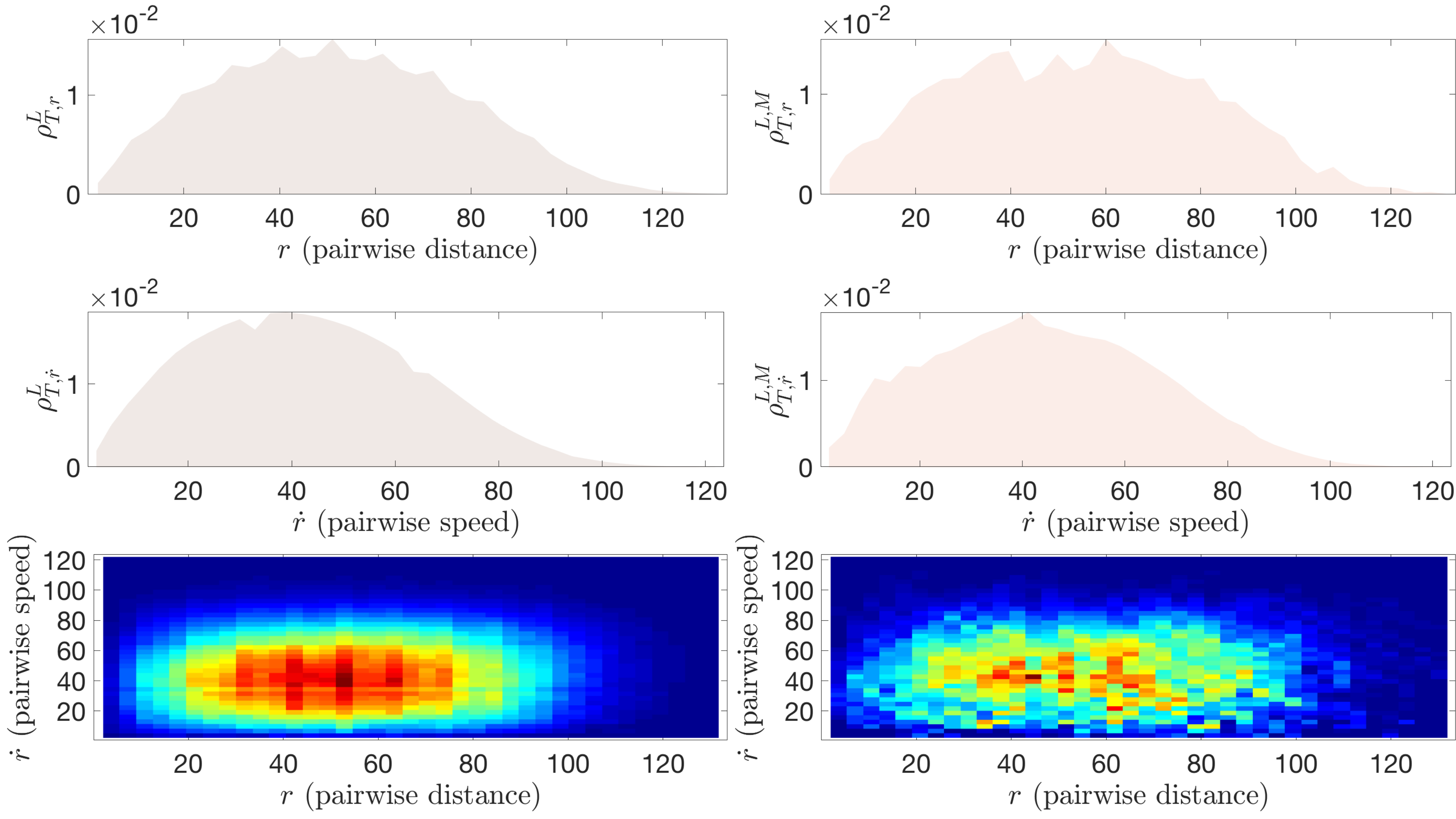}
   \mycaption{(PT) $\rho^L_{T, r, \dot{r}}$ vs. $\rho^{L, M}_{T, r, \dot{r}}$.}
   \label{fig:PT_rhoLTA} 
\end{subfigure}
\begin{subfigure}[b]{\ifarXiv 0.48\textwidth \fi}
   \includegraphics[width=1\linewidth]{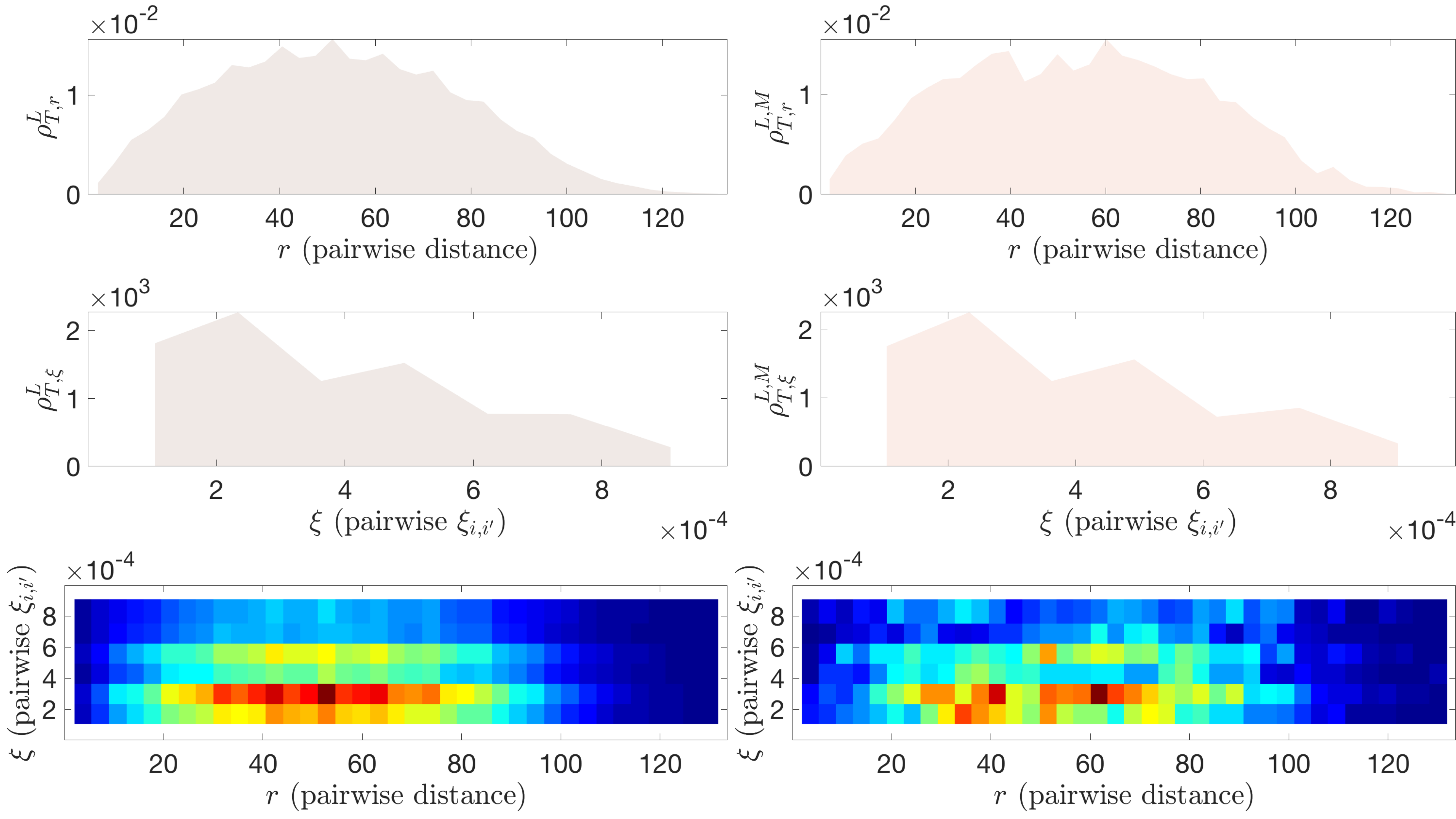}
   \mycaption{(PT) $\rho^L_{T, r, \xi}$ vs. $\rho^{L, M}_{T, r, \xi}$.}
   \label{fig:PT_rhoLTXi} 
\end{subfigure}
\mycaption{(PT) Density plots for the various $\rhoL$ measures.}
\end{figure}

\begin{figure}[H]
\centering
\includegraphics[width=\ifarXiv 0.75\textwidth \fi]{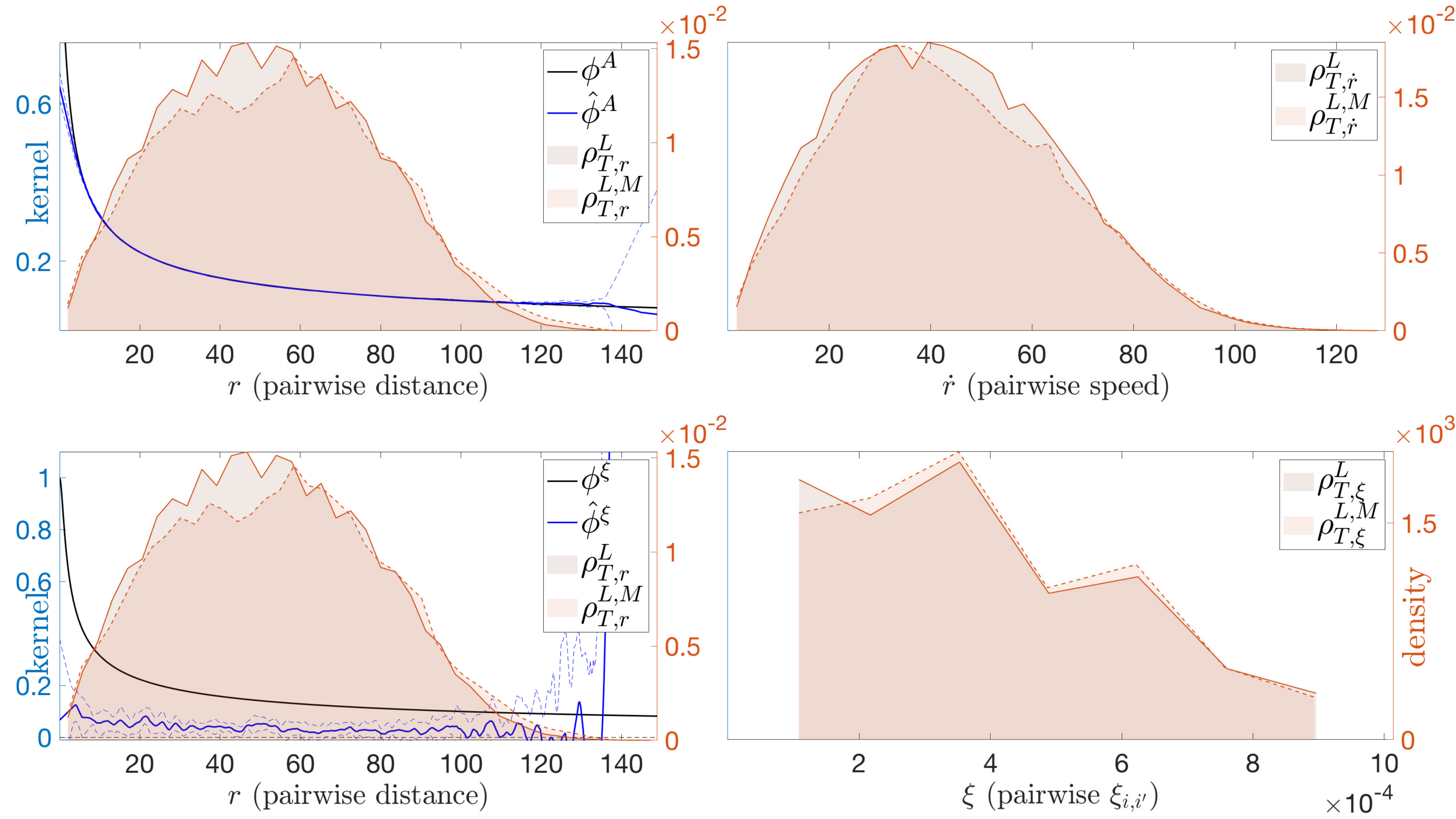}
\mycaption{(PT) Interaction kernels learned from noisy observations of positions and velocities. The noises are multiplicative, Unif.$([-\sigma,\sigma])$ with $\sigma=0.1$ and with other parameters as in Table \ref{t:PTparams}. The estimated kernel for associated with $\dot\bx_i$ is minimally affected, mostly in regions with small $\rhoL$; the additive noise is on a scale far great then that on $\xi_i$ hence severely affects the learning result on the interaction kernel on $\xi_i$.}
\label{fig:PTnoise}
\end{figure}

Figure \ref{fig:PTMLtest} shows the behavior of the error of the estimators as both $L$ and $M$ are increased.
\begin{figure}[H]
\centering
\includegraphics[width=\ifarXiv 0.48\textwidth \fi]{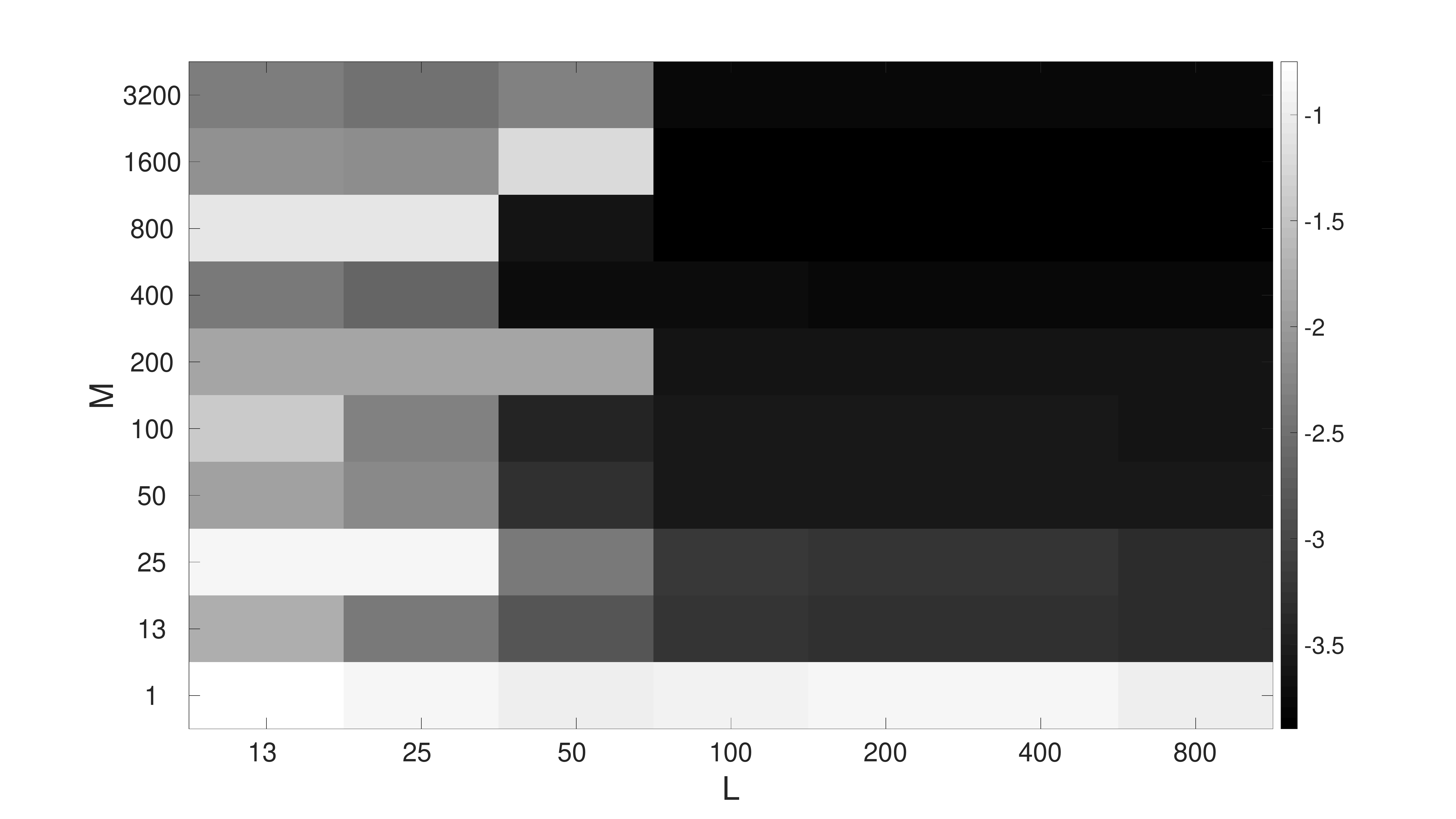}
\includegraphics[width=\ifarXiv 0.48\textwidth \fi]{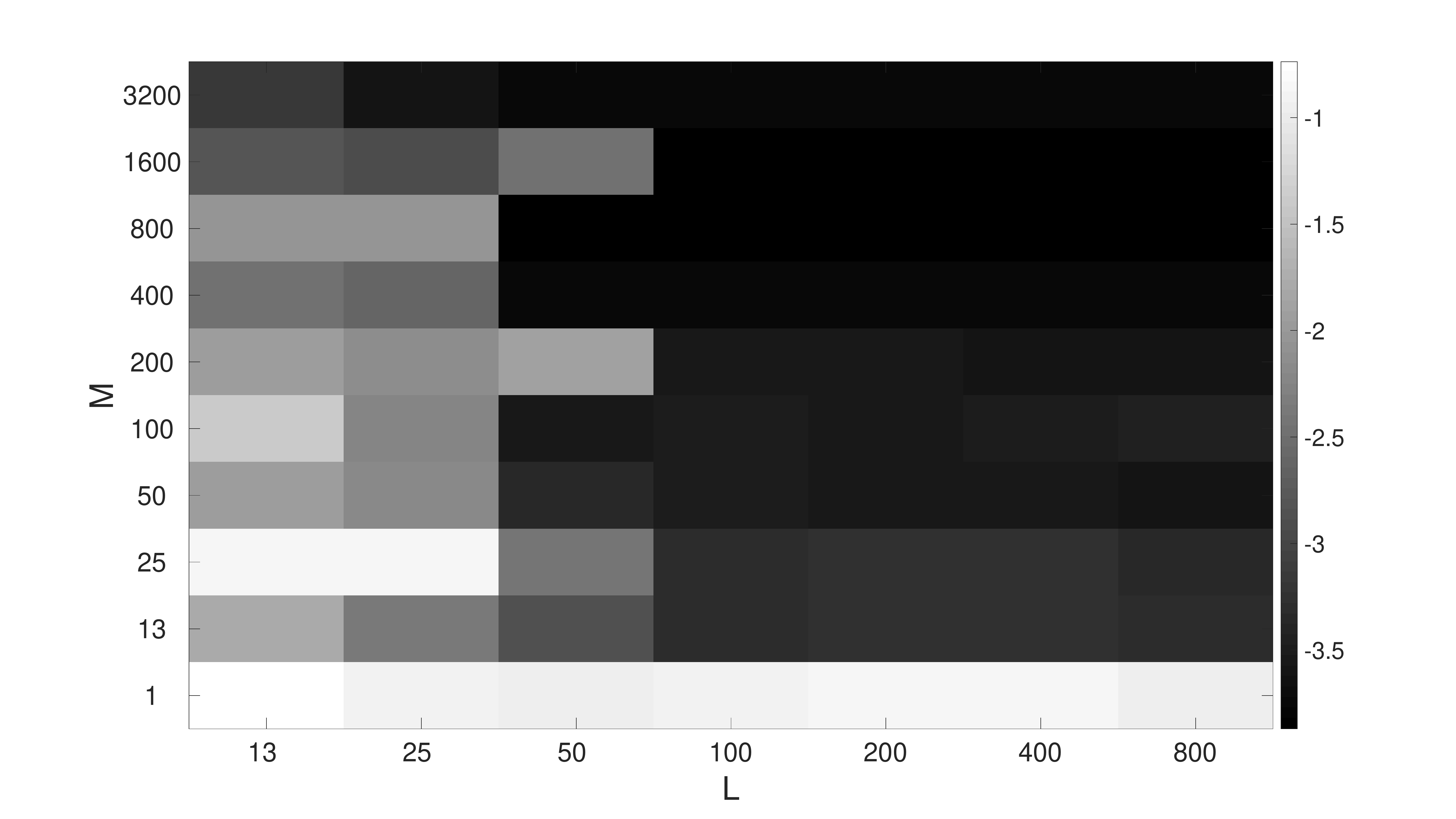}
\mycaption{(PT) Relative error, in $\log_{10}$ scale, of $\hat\intkernel^A$ (left) and $\hat\intkernel^\xi$ (right) as a function of $L$ and $M$. The error decreases both in $L$ and $M$, in fact roughly in the product $ML$. The fourth plot is an identically $0$ absolute error, because both $\phi^E_{2,2}$ and its estimator are identically $0$, since there is only one predator. \revision{Note $M\gg1$ seems to be needed for accurate inference of the interaction kernels, regardless of how large $L$ is: the trajectories explored for small $M$ do not explore enough configuration to enable estimation, suggesting that the limit $M\rightarrow+\infty$ considered in this work is of fundamental importance, at least for non-ergodic systems.}}
\label{fig:PTMLtest}
\end{figure} 
\subsection{Model Selection}\label{s:SI_MS}
Our learning approach can be used to identify the model of the system from the observation data.  We consider here two different scenarios of model selection: one is identifying the type -- energy-based vs. alignment-based -- of interaction kernels from a second order system driven by only one type of interaction kernel; the other is to identify the order of the system from a heterogeneous dynamics.

{\em{Model Selection: energy-based vs. alignment-based interactions}}. We consider a special case of the second order homogeneous agent dynamics, given as either
\[
\ddot{\bx}_i = \sum_{i' = 1}^N\frac{1}{N}\intkernele(r_{ii'})\br_{ii'} \quad \text{or} \quad \ddot{\bx}_i = \sum_{i' = 1}^N\frac{1}{N}\intkernela(r_{ii'})\dot\br_{ii'},
\]
with the (unknown) interaction kernels defined as 
\[
\intkernele(r) = 2 - \frac{1}{r^2} \quad \text{and} \quad \intkernela(r) = \frac{1}{(1 + r^2)^{0.25}}.
\]
The system parameters are given in Table \ref{t:MS_params_1}.
\begin{table}[H]\centering
\footnotesize{\begin{tabular}{| c | c | c | c | c | c | c | c |}
\hline 
 $d$ & $M$      & $L$    & $T$    & $\probIC^{\bX}$          & $\probIC^{\dot\bX}$         &$n^E = n^A$      & deg($\psi^A$)$=$deg($\psi^\xi$)\\ 
\hline 
 $2$ & $200$ & $200$ &  $10$ & Unif. on ring $[0.5, 1]$ & $\mathcal{U}([0, 10]^2)$ &$800$              & $1$\\
\hline
\end{tabular}}
\mycaption{\textmd{(MS$1$ and $2$) Test Parameters}}
\label{t:MS_params_1}
\end{table}

Given the observation data from either system ($\intkernele$- or $\intkernela$-driven), we proceed to learn the interaction kernels as usual, i.e. as if the dynamics were generated with both energy-based and alignment-based interaction kernels present. Results are shown in \ifarXiv Fig.~\ref{fig:ms_cases1}\fi.  The two sub-figures on the left show the learned interaction kernels $\lintkernele$ and $\lintkernela$ from a purely energy-based system: $\lintkernela$ is small in the appropriate norm, while $\lintkernele$ is large (and a good approximation to $\intkernele$): the estimators can therefore detect this is an energy-driven system. In the two sub-figures on the right, we display the analogous results corresponding to learning the interaction kernels for an alignment-based system. We obtain (almost) $0$ for the norm of $\lintkernele$.  The result why the $L^2(\rho^L_{T, r, \dot{r}})$ norm of $\lintkernela$ (from the first case) is not close to zero as the $L^2(\rho^L_{T, r})$ norm of the $\lintkernele$ (from the second case) lies in the difference in the joint distribution of the two cases, see Figures \ref{fig:MS1_rhoLTA} and \ref{fig:MS2_rhoLTA}.  To further investigate the properties of the joint distributions (and also to differentiate the two dynamics), we calculated the $\ell^1$ distance of the respective joint distributions to the product and their marginals.  For MS$1$, the $\ell^1$ distance (over $10$ learning trials) between the joint distribution $\rho^{L, M}_{T, r, \dot{r}}$ and the product of its marginals is $1.3 \cdot 10^{-1} \pm 3.8 \cdot 10^{-3}$.  For MS$2$, the $\ell^1$ distance (over $10$ learning trials) between the joint distribution $\rho^{L, M}_{T, r, \dot{r}}$ and the product of its marginals is $4.6 \cdot 10^{-1} \pm 3.4 \cdot 10^{-3}$.
\begin{figure}[H]
\centering
\begin{subfigure}[b]{\ifarXiv 0.48\textwidth \fi}
   \includegraphics[width=1\linewidth]{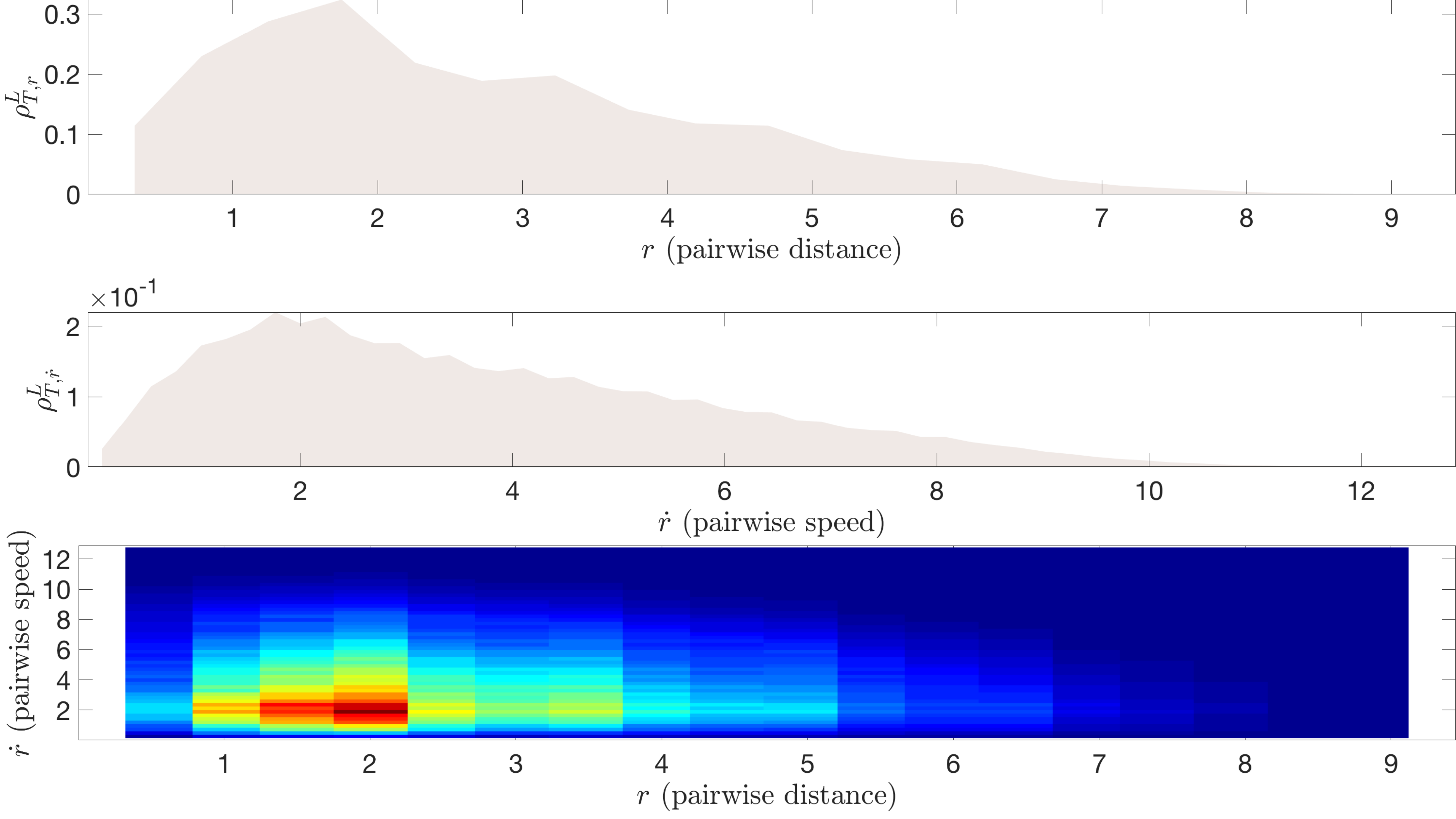}
   \mycaption{(MS$1$) Joint distribution of $\rho^L_{T, r, \dot{r}}$.}
   \label{fig:MS1_rhoLTA} 
\end{subfigure}
\begin{subfigure}[b]{\ifarXiv 0.48\textwidth \fi}
   \includegraphics[width=1\linewidth]{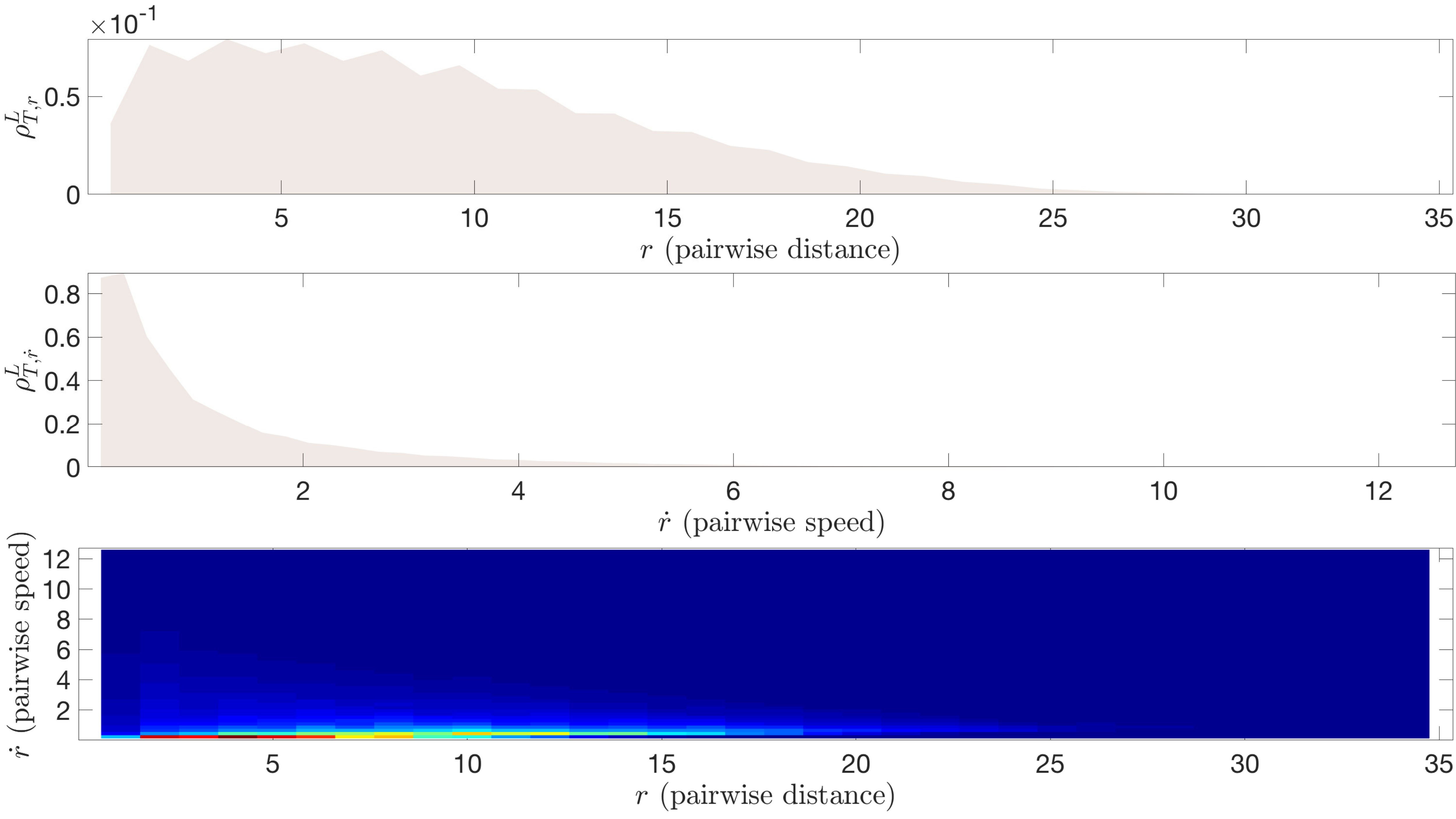}
   \mycaption{(MS$2$) Joint distributions of $\rho^L_{T, r, \dot{r}}$.}
   \label{fig:MS2_rhoLTA} 
\end{subfigure}
\mycaption{(MS$1$ and $2$) Density plots for the various $\rhoL$ measures.}
\end{figure}

{\em{Model Selection: first order vs. second order}}.  We consider two different heterogeneous agent systems, one first order and one second order, with the order of the system unknown to the estimator. The observations are in the time interval $[0,T]$, and in this case $T_f = T$.  We first consider the first order heterogeneous agent system
\[
\dot\bx_i = \sum_{i' = 1}^N \frac{1}{N_{\clof_{i'}}}\intkernel_{\clof_i\clof_{i'}}(r_{ii'})\br_{ii'},
\]
with 
\[
\intkernel_{1, 1}(r) = 1 - r^{-2}, \quad \intkernel_{1, 2}(r) = -2r^{-2}, \quad \intkernel_{2, 1}(r) = 3.5r^{-3}, \quad \intkernel_{2, 2}(r) \equiv 0,
\]
and the type information setup similar to that of the Predator-Swarm first order system (detailed in Sec.~\ifarXiv \ref{s:SI_PSD}\fi).  
For the second scenario, we consider the data generated by the following second order heterogeneous agent dynamics,
\[
\ddot\bx_i = -\dot\bx_i + \sum_{i' = 1}^N \frac{1}{N_{\clof_{i'}}}\intkernele_{\clof_i\clof_{i'}}(r_{ii'})\br_{ii'},
\]
with 
\[
\intkernel_{1, 1}(r) = 1 - r^{-2}, \quad \intkernel_{1, 2}(r) = -r^{-2}, \quad \intkernel_{2, 1}(r) = 1.5r^{-2.5}, \quad \intkernel_{2, 2}(r) \equiv 0,
\]
and the type information setup similar to that of the Predator-Swarm second order system (details shown in Sec.~\ifarXiv \ref{s:SI_PSD}\fi). 
The parameters for both systems are given in Tables \ref{t:MS_params_3} and \ref{t:MS_params_4}.
\begin{table}[H]\centering
\footnotesize{\begin{tabular}{| c | c | c | c |}
\hline 
 $d$  & $M$       & $L$  & $T$   \\
\hline 
 $2$ & $250$ & $250$ & $1$ \\ 
\hline
\hline
 $n$                          & Deg($\psi_{kk'}$) & Prey $\probIC^{\bX}$     & Pred. $\probIC^{\bX}$ \\ 
 \hline
$[298, 150; 150, 2]$ &$[1, 1; 1, 0]$        & Unif. on ring $[0.5, 1.5]$  & Unif. on disk at $0.1$\\
\hline
\end{tabular}}
\mycaption{\textmd{(MS$3$) Test Parameters}}
\label{t:MS_params_3}
\end{table}

\begin{table}[H]\centering
\footnotesize{\begin{tabular}{| c | c | c | c |}
\hline 
 $d$  & $M$       & $L$  & $T$   \\
\hline 
 $2$ & $250$ & $250$ & $1$ \\ 
\hline
\hline
 $n$                          & deg($\psi_{kk'}^E$) & Prey $\probIC^{\bX}$         & Pred. $\probIC^{\bX}$ \\ 
 \hline
$[298, 150; 150, 2]$ &$[1, 1; 1, 0]$           & $\mathcal{U}([0.1, 1]^2)$  & $\mathcal{U}([0, 0.07]^2)$\\
\hline
\end{tabular}}
\mycaption{\textmd{(MS$4$) Test Parameters}}
\label{t:MS_params_4}
\end{table}
With the order of the ODE system and the interaction kernels being the missing information, we construct estimators for the interaction kernels in two ways: first assuming a first order system, then assuming a second order system (without non-collective forcing).  We then generate predicted trajectories using the learned interaction kernels, and the same initial conditions as in the training data.  Next, we calculate the trajectory max-in-time error, obtaining the results in Table \ifarXiv\ref{tab:ms_cases2}\fi (shown as the mean of the trajectory error plus or minus standard deviation of the error over $10$ runs).  As indicated by the trajectory error statistics, the predicted trajectories with smaller error indicate the correct order of the true underlying system in both cases.  Details on the statistics of  the trajectory errors are reported in Tables \ref{t:MS3_traj_err} and \ref{t:MS4_traj_err}. In each, the column with smaller values (within both mean and standard deviation of the trajectory errors) corresponds the correct order of the system.
\begin{table}[H]\centering
\footnotesize{\begin{tabular}{| c || c | c |} 
\hline
                                          & Learned as $1^{st}$ order                                        & Learned as $2^{nd}$ order   \\
\hline
$\text{mean}_{\text{IC}}$ & $\mathbf{9.5 \cdot10^{-3} \pm 2 \cdot10^{-3}}$ & $3.9 \pm 8$\\
\hline
$\text{std}_{\text{IC}}$    & $\mathbf{1.8 \cdot10^{-2} \pm 1.1 \cdot10^{-2}}$ & $48 \pm 1 \cdot 10^2$\\
\hline                     
\end{tabular}}
\mycaption{\textmd{(MS$3$) Trajectory Errors}}
\label{t:MS3_traj_err}
\end{table}

\begin{table}[H]\centering
\footnotesize{\begin{tabular}{| c || c | c |} 
\hline
                                          & Learned as $1^{st}$ order                            & Learned as $2^{nd}$ order   \\
\hline
$\text{mean}_{\text{IC}}$ & $1.6 \pm 1 \cdot 10^{-1}$                        & $\mathbf{1.3 \cdot 10^{-1} \pm 3 \cdot 10^{-2}}$\\
\hline
$\text{std}_{\text{IC}}$    & $9.4 \cdot 10^{-1} \pm 2 \cdot 10^{-1}$ & $\mathbf{2.0 \cdot 10^{-1} \pm 5 \cdot 10^{-2}}$\\
\hline                     
\end{tabular}}
\mycaption{\textmd{(MS$4$) Trajectory Errors}}
\label{t:MS4_traj_err}
\end{table}
\ifarXiv
\bibliographystyle{plain}
\fi
\bibliography{learning_dynamics}
%


%
\end{document}